\newif\ifanon\anonfalse 

\documentclass[11pt]{article}
\usepackage[margin=1in]{geometry}
\usepackage[utf8]{inputenc}
\usepackage[T1]{fontenc} %

\usepackage{mathpazo}
\usepackage{framed}

\usepackage{xcolor}
\usepackage{color}
\definecolor{niceRed}{RGB}{190,38,38}
\definecolor{niceYellow}{HTML}{f5b400}
\definecolor{blueGrotto}{HTML}{059DC0}
\definecolor{royalBlue}{HTML}{057DCD}
\definecolor{navyBlue}{HTML}{0B579C}
\definecolor{limeGreen}{HTML}{81B622}
\definecolor{nicePurple}{HTML}{9c27b0}
\definecolor{lightRoyalBlue}{HTML}{def2ff}  
\definecolor{gold}{HTML}{ffa300}

\newcommand{\white}[1]{\textcolor{white}{#1}}

\usepackage{color-edits} 
\addauthor[Jane]{jl}{blue}
\addauthor[Anay]{am}{gold}
\addauthor[GPT]{gpt}{niceRed}
\addauthor[Manolis]{mz}{teal}

\usepackage{hyperref}
\hypersetup{
  colorlinks = true, 
  urlcolor = {blueGrotto},
  linkcolor = {royalBlue},
  citecolor = {navyBlue}
}
\usepackage[
  backref=true,
  backend=biber,
  natbib=true,
  style=alphabetic,
  sorting=alphabeticlabel,
  sortcites=true,
  minbibnames=3,
  maxbibnames=999,
  mincitenames=1,
  maxcitenames=4,
  minalphanames=1,
  maxalphanames=6,
  doi=false, 
  eprint=false, isbn=false, 
]{biblatex}

\AtEveryBibitem{%
  \clearname{editor}%
}
\AtEveryBibitem{%
  \clearlist{location}%
} 

\DeclareSortingTemplate{alphabeticlabel}{
  \sort[final]{%
    \field{labelalpha} %
  }
  \sort{%
    \field{year}   %
  }
  \sort{%
    \field{title}  %
  }
}

\AtBeginRefsection{\GenRefcontextData{sorting=ynt}}
\AtEveryCite{\localrefcontext[sorting=ynt]}
\usepackage{booktabs} 
\usepackage{multicol} 
\usepackage{multirow} 
\usepackage{longtable}

\usepackage{subcaption}
\usepackage{graphicx}
\usepackage{float} 
\usepackage{tikz}
\usepackage{pgfplots}
\pgfplotsset{compat=1.17}
\usepgfplotslibrary{fillbetween}
\usetikzlibrary{arrows.meta}
\usetikzlibrary{calc}

\tikzset{
  myNodeFlex/.style={
    draw,
    rectangle,
    rounded corners,
    text centered,
    minimum height=1.5em,
  }
}

\tikzset{
  myNode/.style={
    draw,
    rectangle,
    rounded corners,
    text centered,
    minimum height=1.5em,
    minimum width=3cm,
    text width=5cm,    
  }
}

\tikzset{
  myNodeNarrow/.style={
    draw,
    rectangle,
    rounded corners,
    text centered,
    minimum height=1.5em,
    minimum width=1cm,
  }
}

\tikzset{
  myNodeWide/.style={
    draw,
    rectangle,
    rounded corners,
    text centered,
    minimum height=1.5em,
    minimum width=6cm,
  }
}

\usepackage{physics} %
\usepackage{amsmath}
\usepackage{amssymb}
\usepackage{amsthm}
\usepackage{bbm}
\usepackage{blkarray}
\usepackage{mathtools}
\usepackage{nicefrac}
\usepackage{dsfont}
\usepackage{pifont} %
\usepackage{thm-restate} %

\usepackage{setspace}

\usepackage{typed-checklist}
\usepackage{enumerate} 
\usepackage{enumitem} 
\usepackage{tcolorbox}
\usepackage[framemethod=TikZ]{mdframed}
\mdfsetup{%
backgroundcolor=royalBlue!0, %
roundcorner=4pt,
skipabove=5pt,
skipbelow=5pt,
linewidth=1pt}
\usepackage{changepage} %

\usepackage{algorithm}
\usepackage{algpseudocode}[1]

\usepackage{mathrsfs} %
\usepackage{soul} %
\setuldepth{Berlin}

\usepackage{etoolbox}
\usepackage{array}
\usepackage[capitalise,noabbrev,nameinlink,sort]{cleveref} %

\theoremstyle{plain} 
\newtheorem{theorem}{Theorem}[section]
\newtheorem{corollary}[theorem]{Corollary}

\newtheorem{lemma}[theorem]{Lemma}
\newtheorem{fact}[theorem]{Fact}
\newtheorem{claim}[theorem]{Claim}

\newtheorem{assumption}{Assumption}

\newtheorem{inftheorem}[theorem]{Informal Theorem}
\newtheorem{infcorollary}[theorem]{Informal Corollary}

\newtheorem{definition}{Definition}

\newtheorem*{definition*}{Definition}
\newtheorem{problem}{Problem}
\newtheorem{openproblem}{Open Problem}

\theoremstyle{definition} 

\newtheorem{remark}[theorem]{Remark}

\theoremstyle{remark}

\AfterEndEnvironment{definition}{\noindent\ignorespaces}
\AfterEndEnvironment{infdefinition}{\noindent\ignorespaces}
\AfterEndEnvironment{assumption}{\noindent\ignorespaces}
\AfterEndEnvironment{problem}{\noindent\ignorespaces}
\AfterEndEnvironment{openproblem}{\noindent\ignorespaces}
\AfterEndEnvironment{lemma}{\noindent\ignorespaces}
\AfterEndEnvironment{theorem}{\noindent\ignorespaces}
\AfterEndEnvironment{proposition}{\noindent\ignorespaces}
\AfterEndEnvironment{fact}{\noindent\ignorespaces}
\AfterEndEnvironment{question}{\noindent\ignorespaces}
\AfterEndEnvironment{corollary}{\noindent\ignorespaces}
\AfterEndEnvironment{model}{\noindent\ignorespaces}
\AfterEndEnvironment{remark}{\noindent\ignorespaces}
\AfterEndEnvironment{proof}{\noindent\ignorespaces}
\AfterEndEnvironment{fact}{\noindent\ignorespaces}
\AfterEndEnvironment{minftheorem}{\noindent\ignorespaces}
\AfterEndEnvironment{inftheorem}{\noindent\ignorespaces}
\AfterEndEnvironment{maintheorem}{\noindent\ignorespaces}
\AfterEndEnvironment{restatable}{\noindent\ignorespaces}
\AfterEndEnvironment{infassumption}
{\noindent\ignorespaces}
\AfterEndEnvironment{infcorollary}{\noindent\ignorespaces}

\crefname{section}{Section}{Sections}
\crefname{theorem}{Theorem}{Theorems}
\crefname{lemma}{Lemma}{Lemmas}
\crefname{problem}{Problem}{Problems}
\crefname{program}{Program}{Programs}
\crefname{definition}{Definition}{Definitions}
\crefname{conjecture}{Conjecture}{Conjectures}
\crefname{corollary}{Corollary}{Corollaries}
\crefname{construction}{Construction}{Constructions}
\crefname{conjecture}{Conjecture}{Conjectures}
\crefname{claim}{Claim}{Claims}
\crefname{observation}{Observation}{Observations}
\crefname{proposition}{Proposition}{Propositions}
\crefname{fact}{Fact}{Facts}
\crefname{question}{Question}{Questions}
\crefname{problem}{Problem}{Problems}
\crefname{remark}{Remark}{Remarks}
\crefname{example}{Example}{Examples}
\crefname{equation}{Equation}{Equations}
\crefname{appendix}{Section}{Sections}
\crefname{algorithm}{Algorithm}{Algorithms}
\crefname{model}{Model}{Models}
\crefname{figure}{Figure}{Figures}
\crefname{infassumption}{Informal Assumption}{Informal Assumptions}
\crefname{inftheorem}{Informal Theorem}{Informal Theorems}
\crefname{infdefinition}{Informal Definition}{Informal Definitions}
\crefname{minftheorem}{Main Informal Theorem}{Main Informal Theorems}
\crefname{maintheorem}{Main Theorem}{Main Theorems}
\crefname{assumption}{Assumption}{Assumptions}
\crefname{step}{Step}{Steps}
\crefname{result}{Result}{Results}
\crefname{event}{Event}{Events}
\crefname{none}{}{}

\newlist{asmpenum}{enumerate}{1} %
\setlist[asmpenum]{label={\arabic*.},ref=\theassumption.{\arabic*}}
\crefname{asmpenumi}{Assumption}{Assumptions}

\BeforeBeginEnvironment{subenvironment}{\white{.}\\ \vspace{-10mm}}
\AfterEndEnvironment{subenvironment}{\vspace{4mm}}

\newcommand{\yesnum}{\refstepcounter{equation}\tag{\theequation}}
\makeatletter
\newcommand{\tagnum}[2]{%
    \refstepcounter{equation}%
    \tag{#1) \ (\theequation}%
    \protected@write \@auxout {}{%
        \string \newlabel {#2}{{\theequation}{\thepage}{}{equation.\theequation}{}}%
    }%
}
\makeatother

\makeatletter
\renewcommand{\eqref}[1]{\textup{\eqrefform@{\ref{#1}}}}
\let\eqrefform@\tagform@

\newcommand{\changetag}[1]{%
  \renewcommand\tagform@[1]{\maketag@@@{(\ignorespaces#1\unskip\@@italiccorr)}}%
}
\makeatother

\newcommand{\Stackrel}[2]{\stackrel{\mathmakebox[\widthof{\ensuremath{#2}}]{#1}}{#2}}

\newcommand{\quadtext}[1]{\quad\text{#1}\quad}
\newcommand{\qquadtext}[1]{\qquad\text{#1}\qquad}
 
\newcommand{\quadand}{\quadtext{and}}
\newcommand{\qquadand}{\qquadtext{and}}

\def\abs#1{\left| #1 \right|}
\def\sabs#1{| #1 |}
\newcommand{\given}{\;\middle|\;}
\newcommand{\sinparen}[1]{(#1)}
\newcommand{\binparen}[1]{\big(#1\big)}

\newcommand{\sinbrace}[1]{\{#1\}}

\newcommand{\inbrace}[1]{\left\{#1\right\}}

\newcommand{\inparen}[1]{\left(#1\right)}
\newcommand{\insquare}[1]{\left[#1\right]}
\newcommand{\inangle}[1]{\left\langle#1\right\rangle}

\let\norm\relax
\newcommand{\norm}[1]{\ensuremath{\left\lVert #1 \right\rVert}}

\newcommand{\midsepremove}{\aboverulesep = 0mm \belowrulesep = 0mm}

\newcommand{\midsepdefault}{\aboverulesep = 0.605mm \belowrulesep = 0.984mm}

\newcommand{\N}{\mathbb{N}}
\newcommand{\R}{\mathbb{R}}

\newcommand{\evE}{\ensuremath{\mathscr{E}}}
\newcommand{\evF}{\ensuremath{\mathscr{F}}}
\newcommand{\evG}{\ensuremath{\mathscr{G}}}

\newcommand{\E}{\operatornamewithlimits{\mathbb{E}}} 
\newcommand{\Ex}{\E}

\newcommand{\cov}{\ensuremath{\operatornamewithlimits{\rm Cov}}}
\newcommand\ind{\mathds{1}}

\newcommand{\argmin}{\operatornamewithlimits{arg\,min}}

\newcommand{\tv}[2]{\operatorname{d}_{\mathsf{TV}}\sinparen{#1,#2}}
\newcommand{\kl}[2]{\operatornamewithlimits{\mathsf{KL}}\sinparen{#1\|#2}} 
\newcommand{\chidiv}[2]{\chi^2\inparen{#1\|#2}}

\newcommand{\renyi}[3]{\cR_{#1}\inparen{#2\|#3}}

\newcommand{\zo}{\ensuremath{\inbrace{0, 1}}}

\newcommand{\sfrac}[2]{{#1/#2}} 
\newcommand{\nfrac}[2]{\nicefrac{#1}{#2}}
 
\newcommand{\opt}{\ensuremath{\mathrm{OPT}}}

\newcommand{\poly}{\mathrm{poly}}
\newcommand{\polylog}{\mathrm{polylog}}
\newcommand{\vc}{\textrm{\rm VC}}

\newcommand{\supp}{\operatorname{supp}}

\newcommand{\iid}{i.i.d.}
\newcommand{\unif}{{\rm Unif}}

\newcommand{\tick}{\ding{51}}
\let\cross\relax
\newcommand{\cross}{\ding{55}}

\newcommand{\eps}{\varepsilon}
\renewcommand{\epsilon}{\varepsilon}
\makeatletter
\newcommand*{\tran}{{\mathpalette\@tran{}}}
\newcommand*{\@tran}[2]{\raisebox{\depth}{$\m@th#1\intercal$}}
\makeatother

\mathchardef\NABLA"272
\newcommand*{\Nabla}{\boldsymbol\NABLA}
\let\nabla\Nabla

\newcommand{\wh}[1]{\widehat{#1}}

\renewcommand{\tilde}{\widetilde}
\newcommand{\wt}[1]{\widetilde{#1}}

\newcommand{\customcal}[1]{\euscr{#1}}

\newcommand{\cB}{\customcal{B}}

\newcommand{\cD}{\customcal{D}}
\newcommand{\cE}{\customcal{E}}

\newcommand{\cL}{\customcal{L}}
\newcommand{\cM}{\customcal{M}}
\newcommand{\cN}{\customcal{N}}

\newcommand{\cP}{\customcal{P}} 
\newcommand{\cQ}{\customcal{Q}} 
\newcommand{\cR}{\customcal{R}}

\newcommand{\cX}{\customcal{X}}

\DeclareMathAlphabet{\mathdutchcal}{U}{dutchcal}{m}{n}
\SetMathAlphabet{\mathdutchcal}{bold}{U}{dutchcal}{b}{n}
\DeclareMathAlphabet{\mathdutchbcal}{U}{dutchcal}{b}{n}

\DeclareMathAlphabet\urwscr{U}{urwchancal}{b}{n}%
\DeclareMathAlphabet\rsfscr{U}{rsfso}{m}{n}
\DeclareMathAlphabet\euscr{U}{eus}{m}{n}
\DeclareFontEncoding{LS2}{}{}
\DeclareFontSubstitution{LS2}{stix}{m}{n}
\DeclareMathAlphabet\stixcal{LS2}{stixcal}{m} {n}

\renewcommand{\paragraph}[1]{\medskip \noindent\textbf{#1}~}
\newcommand{\itparagraph}[1]{\medskip \noindent\textit{#1}~}

\newcommand{\ie}{\textit{i.e.}}
\newcommand{\eg}{\textit{e.g.}}

\newcommand{\hypo}[1]{\mathdutchcal{#1}}

\newcommand{\hyD}{\hypo{D}}

\newcommand{\hyF}{\hypo{F}}

\newcommand{\hyH}{\hypo{H}}

\newcommand{\hyP}{\hypo{P}}

\newcommand{\hyS}{\hypo{S}}

\newcommand{\pERM}{Pessimistic-ERM}

\newcommand{\lreg}{\textsf{L1-Regression}}
\newcommand{\constrainedreg}{$\mathsf{Constrained\text{-}Regression}$} %

\renewcommand{\d}{{\rm d}}

\newcommand{\optset}{{H^\star}}

\newcommand{\cDtrue}{\cD^\star}
\newcommand{\cDgiven}{\cD}
\newcommand{\cPtrue}{\cP^\star}

\newcommand{\unlabeled}{{{\cD}}}

\renewcommand{\tilde}{\widetilde}

\newcommand{\hstar}{h^\star}
\newcommand{\muStar}{\mu^\star}
\newcommand{\SigmaStar}{\Sigma^\star}

\newcolumntype{L}[1]{>{\raggedright\let\newline\\\arraybackslash\hspace{0pt}}m{#1}}
\newcolumntype{C}[1]{>{\centering\let\newline\\\arraybackslash\hspace{0pt}}m{#1}}
\newcolumntype{R}[1]{>{\raggedleft\let\newline\\\arraybackslash\hspace{0pt}}m{#1}}

\newcommand{\np}{\textsf{NP}} 
 
\newcommand\blfootnote[1]{%
  \begingroup
  \renewcommand\thefootnote{}\footnote{#1}%
  \addtocounter{footnote}{-1}%
  \endgroup
}

\title{Smoothed Analysis of Learning from Positive Samples}
\author{ 
        \centering
        \hspace{4mm}
        \begin{tabular}{C{4.75cm}C{4.75cm}C{4.75cm}}
        {\bf Jane H. Lee} & {\bf Anay Mehrotra} & {\bf Manolis Zampetakis}\\[2mm]
        Yale University & Stanford University & Yale University\\[-4mm]
        {\small\phantom{....................}} \mbox{\small\href{mailto:jane.h.lee@yale.edu}{\texttt{jane.h.lee@yale.edu}}} & {\small \phantom{............}} \mbox{\small\href{mailto:anaymehrotra1@gmail.com}{\texttt{anaymehrotra1@gmail.com}}} & \mbox{\small\href{mailto:manolis.zampetakis@yale.edu}{\texttt{manolis.zampetakis@yale.edu}}}
        \\
        \end{tabular} 
} 

\date{}

\begin{document}

\maketitle
\thispagestyle{empty}

\begin{abstract}
    Binary classification from positive-only samples is a variant of PAC learning where the learner receives \iid{} positively labeled samples and aims to learn a classifier that, with high probability, achieves low classification error.
    Previous work by \citet[STOC]{natarajan1987learning}, \citet[Thesis]{graus1989lower}, and \citet[Machine~Learning]{shvaytser1990positiveonly} characterized learnability in this setting and revealed a largely negative picture: almost no interesting classes, including two-dimensional halfspaces, are learnable from positive-only examples.
    This poses significant challenges for the plethora of applications of positive-only learning from bioinformatics to ecology, where practitioners rely on heuristics for learning. 
    
    In this work, we initiate a smoothed analysis of positive-only learning.
    We assume we have access to samples from a reference distribution $\cD$ such that the true data distribution $\cDtrue$ is \emph{smooth} with respect to it.
    Our first result demonstrates that, in stark contrast to the worst-case setting, all VC classes become learnable in the smoothed model, requiring $O(\nfrac{\vc}{\eps^2})$ positive samples to guarantee $\eps$-classification error.
    We then present a computationally efficient algorithm for any concept class that admits $\poly(\eps)$-approximation by degree-$k$ polynomials (whose range is lower-bounded by a constant) with respect to $\cD$ in the $L_1$-norm.
    The algorithm runs in time $\poly(d^{k}/\eps)$, which qualitatively matches the running time of the \lreg{} algorithm.
    This smoothed analysis contributes to the growing body of work designing better learning guarantees under smoothness  \citep[J.~ACM]{haghtalab2024smoothJACM} \cite[COLT]{pmlr-v247-chandrasekaran24a}.

    \blfootnote{{Accepted for presentation at the 58th ACM Symposium on Theory of Computing (STOC), 2026}}
    Our results also imply faster or more general algorithms for the following problems:
    \begin{enumerate}[leftmargin=5mm]
        \item \textbf{Estimation under unknown truncation,} where we give the first polynomial sample and time algorithm for estimating the parameters of an exponential family distribution from samples truncated to an unknown set $S^\star$ that is approximable by non-negativepolynomials in $L_1$-norm.
        This improves upon \cite[FOCS]{Kontonis2019EfficientTS} \cite[FOCS]{lee2024unknown}, which required strong approximation with respect to $L_2$.
        \item \textbf{Truncation detection,} where we present the first algorithm for detecting whether given samples have been truncated (or not) for a broad class of distributions, including non-product distributions. 
        This improves upon \cite[STOC]{de2024detecting} who were limited to product distributions.
        \item \textbf{Learning with a list of reference distributions,} as a corollary of our main result on smoothed analysis. We obtain analogous sample and computational complexity results in the more general setting where we do not have access to (samples from) a reference distribution $\cD$ but rather only have access to samples from a list of $O(1)$ distributions one of which witnesses the smoothness of $\cDtrue$. 
        This naturally arises if list-decoding algorithms are used to learn samplers for $\cDtrue$ from corrupted data.
    \end{enumerate}

\end{abstract}
 
\clearpage

\thispagestyle{empty}
{   
    \tableofcontents
}
\thispagestyle{empty}

\clearpage
\pagenumbering{arabic}

    \vspace{-4mm}
\section{Introduction} \label{sec:intro}
    \vspace{-2mm}
    
    In the celebrated PAC Learning framework \cite{valiant1984theory}, the goal is to learn a binary concept from perfectly labeled samples by minimizing misclassification error of future predictions. 
    In many real-world settings, however, while positively labeled samples are readily available, negatively labeled samples are scarce or costly to obtain. {This was motivated as a problem even in the very first papers of \citet{valiant1984deductive} where he discusses the learning task of classifying elephants from non-elephants and he mentions: \emph{``While it may be reasonable to discuss the distribution of the attributes of elephants, we may prefer not discussing the distribution of the attributes of non-elephants.''}}
    This issue arises, also, in medical diagnosis, where confirmed cases of a condition are recorded but healthy controls are unlabeled \cite{elkan2008pu}, or in web search and recommendation systems, where user engagement indicates positive feedback but non-engagement is ambiguous \cite{liu2003building}. 
    These practical challenges can be modeled through a natural extension of the PAC framework by \citet{natarajan1987learning}, where the learner only observes samples drawn from the underlying distribution of positive data.
    \vspace{-7mm}
    \begin{definition}[PAC Learning with Positive-Only Samples; \cite{natarajan1987learning}]
        \label{def:positiveOnly}
         Let $\cDtrue$ be an unknown distribution over the feature space $\cX$ and let $\hstar \in \hyH$ be the target concept.
         The learner receives $n$ i.i.d.\ samples drawn according to the conditional distribution $\cPtrue = \cDtrue {\mid} (\hstar(X) {=} 1)$, that is, the distribution $\cDtrue$ conditioned on the event that $\hstar(X) = 1$.
         As in the standard PAC model, the goal is to output a hypothesis $\wh{h}\colon \cX\to \zo$ that achieves small misclassification error with respect to the original distribution $\cDtrue$, \ie{}, 
         \vspace{-2mm}
         \[
            \Pr\nolimits_{X\sim \cDtrue}\sinparen{\wh{h}(X)\neq \hstar(X)} \to 0 \quad \text{as} \quad n \to \infty\,. \vspace{-2mm}
         \]
        Observe that although the samples are drawn from $\cPtrue$ the performance of $\wh{h}$ is evaluated in both the positive and the negative regions.
        \vspace{-5pt}
    \end{definition}
     \citet{natarajan1987learning} was the first to study \textit{learning from positive samples} and he characterized the classes that can be \textit{properly} PAC learned from only positive samples.\footnote{\citet{natarajan1987learning}'s characterization also requires the learners' errors to be one-sided and, in particular, have zero false positives. Later works relaxed this requirement (for improper learning) \citep{shvaytser1990positiveonly,kivinen1995one,graus1989lower}.}
     Subsequently, \citet{shvaytser1990positiveonly,kivinen1995one,graus1989lower} completed the characterization for improper learners and bounded the sample complexity for the problem. These results are perceived mostly as negative results since they demonstrated that even fundamental classes such as two-dimensional halfspaces cannot be properly learned from positive-only samples.
        
    To bypass this worst-case bottleneck, and due to the practical relevance of learning from positive-only data, several works explored \emph{average-case} versions of the problem.
    For instance, a line of work (\eg{}, \cite{frieze1996linearTransformations,Nguyen2009learningParallelepiped,anderson2013simplices,de2015satisfying,Kontonis2019EfficientTS,canonne2020satisfying,he2023robustMoments,lee2024unknown}) designs computationally efficient algorithms for learning specific hypothesis classes (\eg{}, simplices, parallelepipeds, halfspaces, polynomial threshold functions, or concepts with small \textit{surface area}) from positive samples, when crucially, the underlying distribution $\cDtrue$ is typically assumed to be either the uniform distribution on the Boolean Cube or Gaussian.\footnote{We note that the notion of error in some of these works is slightly different from the one in \cref{def:positiveOnly}; \eg{}, \cite{anderson2013simplices} define the error of $\wh{h}$ as $\tv{\unif{}(\wh{h})}{\unif{}(\hstar)}$ where $\unif{}(h)$ is the uniform distribution over $\inbrace{x\in \cX\colon h(x)=1}.$}
    \enlargethispage{\baselineskip}

    In the majority of the many practical applications of learning from positive-only data, parametric assumptions on $\cDtrue$ naturally can be unrealistic.
    Indeed, the negative results to positive-only learning are a significant barrier for the use of machine learning in critical applications where negative examples are unavailable \citep{jaskie2022positive}.
    In response, a substantial body of work in applied sciences, from bioinformatics \cite{ming2016pupylation,pengyi2019adaSampling,zhao2008geneFunction,xiao2008niologicalSequence,yang2012positive,luigi2010geneRegulatory} to ecology \cite{steven2004ecologyMaxEntropy}, has developed heuristics for learning from positive-only samples.
    Sometimes these heuristics come with rigorous statistical guarantees but, then as the works above, require strong additional assumptions such as exact access to samples from the marginal distribution $\cDtrue$, which, naturally, can be unavailable in practice (see \cref{sec:realWorldApplications} for a discussion).

    Hence, all existing works on positive-only learning with provable guarantees either require strong assumptions on $\cDtrue$ (\eg{}, Gaussianity or exact sample access) or are limited to restrictive hypotheses classes, raising the following fundamental question:
    \begin{mdframed}
        \emph{Can we learn general concept classes from positive-only samples without strong assumptions on $\cDtrue$?}
    \end{mdframed}
    \vspace{-7pt}
    
    \paragraph{Our Contribution.} To answer this question we initiate the study of \textit{smoothed analysis of learning from positive-only samples} which lies in between the worst-case analysis of positive-only learning where no assumptions are made on $\cDtrue$ (but even halfspaces are not learnable), and the average-case analysis where many classes are learnable (but strong assumptions are made on $\cDtrue$, \eg{}, Gaussianity). Our work conceptually follows earlier smoothed analysis of learning problems (\eg{}, \cite{haghtalab2024smoothJACM}) where the goal is to bypass worst-case negative results for online learning. In particular, we only make the assumption that the underlying distribution $\cDtrue$ is \textit{smooth} with respect to a reference distribution $\cD$\footnote{If $\cD{=}\cDtrue$, then smoothness holds; hence, the smoothed model strictly generalizes the assumption that $\cDtrue$ is known. Furthermore, if $\cDtrue$ is Gaussian, we can find an appropriate reference measure $\cD$ under minimal assumptions (see, \eg{}, \cite{Kontonis2019EfficientTS}) so the smoothed model also generalizes Gaussianity of $\cD$ up to these minimal assumptions.} (see \cref{asmp:smoothness}).
    
    We show that, surprisingly, the smooth model is just as tractable as the average-case model. In particular, our main results are:
    \vspace{-4pt}
    \begin{enumerate}[leftmargin=15pt,itemsep=-1pt]
      \item \textbf{Learnability of VC-classes from Positive-only Data (\cref{thm:sampleComplexity}).} We show that any class $\hyH$ with $\vc{}(\hyH)<\infty$ can be learned with $O(\vc{}(\hyH)/\eps^2)$ positive samples in the smoothed model.
      \item \textbf{Efficient Learnability from Positive-only Data (\cref{thm:main}).} We show that any class $\hyH$ approximable by non-negativepolynomials in $L_1$-norm can be computationally efficiently learned in the smoothed model.
      \vspace{-4pt}
     \end{enumerate}
     One important aspect of our results is that we place only mild assumptions about the knowledge of the reference measure $\cD$. Parallels can be made to the line of work on smoothed analysis of online learning where the first papers assumed perfect knowledge of the reference measure $\cD$ \cite{haghtalab2020smooth,haghtalab2024smoothJACM} but follow-up works eliminated the knowledge of $\cD$ completely \cite{block2024performanceempiricalriskminimization,blanchard2024agnostic}. Our results only rely on the mild assumption that we have sample access to $\cD$ and otherwise $\cD$ can be arbitrary and we do not need to know, \eg{}, its density. Because of this very mild assumption on $\cD$ our results have also implication for other important problems in learning theory:
     \vspace{-3pt}
     \enlargethispage{\baselineskip}
     \begin{enumerate}[leftmargin=15pt,itemsep=-3pt]
         \item[$\blacktriangleright$] \textbf{Truncated Statistics (\cref{sec:intro:application:truncation}).} Our positive-only learning results improve the estimation of high-dimensional distributions under unknown truncation or censoring \cite{daskalakis2018efficient, Kontonis2019EfficientTS}.
         \item[$\blacktriangleright$] \textbf{Detecting Truncation (\cref{sec:intro:application:detection}).} Our positive-only learning can also be used in detection algorithms of whether a dataset has been subject to censoring \cite{de2023detectingConvex, de2024detecting}.
        \item[$\blacktriangleright$] \textbf{Learning with a list of reference distributions (\cref{sec:intro:application:listDecoding}).}
            Finally, we generalize our algorithms to work when we do not have sample access to the reference distribution $\cD$, but only have a list of $O(1)$ distributions one of which witnesses the smoothness of $\cDtrue$ (\cref{sec:intro:application:listDecoding}).
            This naturally arises when a sampler for $\cDtrue$ must be learned from a data-source with high-corruption rate necessitating the use of list decoding algorithms.
     \end{enumerate}

     \paragraph{Roadmap.} 
        In \cref{sec:intro:PIU}, we give a more precise formulation of the smoothed model and then provide an informal statement of our main results in \cref{sec:informal}.
        This is followed in \cref{sec:intro:application} by informal statements of applications and generalizations that we highlighted above.
        Next, we overview the key ideas in our proof in \cref{sec:technicalOverview}. 
        Finally, we discuss some real-world applications where positive-only learning arises but $\cD$ is not exactly known (\cref{sec:intro:realWorldApplications}), open problems (\cref{sec:intro:takeawaysOpenProblems}), and additional related work (\cref{sec:intro:relatedWorks}).

\subsection{Smoothed Analysis of Positive-Only Learning} 
\label{sec:intro:PIU}
    Toward developing a more realistic theoretical framework that yields better practical algorithms, we adopt the smoothed analysis framework introduced by \citet{spielman2004smooth}. 
    Smoothed analysis was originally developed to explain why algorithms with poor worst-case complexity often succeed in practice: it studies algorithm performance under small perturbations of worst-case inputs, effectively smoothing the input distribution.
    In learning theory, this philosophy has been adapted by requiring that the unknown data distribution $\cDtrue$ is smoothed with respect to a reference measure $\cD$. The smoothness condition that we use in this paper is more relaxed than usual the notion of smoothness that has been used in learning theory by \cite{haghtalab2024smoothJACM,block2024performanceempiricalriskminimization,blanchard2024agnostic}, and places less assumptions on the distribution $\cDtrue$.

\begin{restatable}[Generalized Smoothness]{assumption}{smoothness}\label{asmp:smoothness}
    There are known constants $\sigma \in (0,1]$ and $q\geq 1$ and a distribution $\cD$, such that, $\cDtrue$ is smooth with respect to $\cD$ in the following sense: 
            for any measurable set $S$,
            \[
                \cDtrue(S)\leq \frac{1}{\sigma}\cdot \cD(S)^{\sfrac{1}{q}}\,.
            \]
    {We often refer to $\cDtrue$ as the unlabeled distribution and $\cD$ as the imperfect unlabeled distribution.}
\end{restatable}
This assumption upper bounds the unlabeled distribution's mass on any set $S$ in terms of the reference or imperfect unlabeled distribution's mass.
To gain some intuition, suppose that $\cD$ is a uniform mixture of the unlabeled distribution $\cDtrue$ and some other ``bad'' distribution $\cB$, \ie{}, $\cD=0.5\,\cDtrue+0.5\,\cB$.
Then, \mbox{for any set $S$, $\cDtrue(S) \leq 2\, \cD(S)$ and, hence, smoothness holds with $(\sigma,q)=(\nfrac{1}{2},1)$.}
 
    For $q=1$, this assumption reduces to the notion of smoothness considered by recent works on smoothed online learning \cite{manthey2020smoothed,haghtalab2022efficientSmooth,haghtalab2020smooth,haghtalab2024smoothJACM,block2024performanceempiricalriskminimization,blanchard2024agnostic}. 
When $q=2$, generalized smoothness is implied by $\chidiv{\cDtrue}{\cD} = O(1)$ and, for larger $q$, it is implied by bounds on the suitable Rényi divergence (see \cref{sec:distCloseness}).\footnote{The $\chi^2$-divergence of $\cP$ with respect to $\cQ$ is $\chidiv{\cP}{\cQ}=\int_x \inparen{\nfrac{\d \cP(x)}{\d \cQ(x)}}^2 \d \cQ(x) - 1$. We define Rényi divergences in \cref{sec:preliminary}, although understanding their definition is not crucial for our results.} 

As we will present in the next section, generalized smoothness suffices to design sample-efficient algorithms. For our computationally efficient algorithms though, we need to additionally assume that the positive sample set $\optset$ has non-trivial mass.
\begin{restatable}[Non-Trivial Fraction of Positives]{assumption}{posFraction}\label{asmp:posFraction}
    For some known $\alpha>0$, $\Pr_{x\sim \cDtrue}\inparen{h^\star(x)=1}\geq \alpha$.
\end{restatable} 

\noindent 
    This assumption is also required by prior works on positive-only learning that design computationally efficient algorithms (\eg{}, \cite{Kontonis2019EfficientTS,lee2024unknown}).
    We note that the worst-case impossibility results for learning with positive-only samples continue to hold even under \cref{asmp:posFraction}; for completeness, we prove this in \cref{sec:impossibility}.

\subsection{Our Main Results}\label{sec:informal}
In this section, we present our main results on smoothed analysis for positive-only learning.
Then, in \cref{sec:intro:application}, we present the applications of these results to other learning theory problems, as we highlighted before. %
\subsubsection{Sample Efficiency of Smoothed Positive-Only Learning}
Our first result provides the sample complexity of positive-only learning in the smoothed model.
\begin{restatable}[Sample Complexity]{theorem}{sampleComplexity}
    \label{infthm:sampleComplexity}\label{thm:sampleComplexity}
    Suppose \cref{asmp:smoothness} holds. %
    Fix any $\eps,\delta\in (0,\nfrac{1}{2})$.
    There is an algorithm that, given $\eps,\delta$ and $n=\wt{O}\binparen{\!\inparen{\vc{}(\hyH)+\log{\nfrac{1}{\delta}}} / {{(\eps \sigma)}}^{2q}}$ independent samples from $\cPtrue$ and $\cD$, outputs a hypothesis $\wh{h}$ such that, with probability $1-\delta$,
    \[ 
        \Pr\nolimits_{\cDtrue}\binparen{\wh{h}(x)\neq h^\star(x)}\leq \eps\,.
    \]
\end{restatable} 
An implication of this theorem is that smooth positive-only learning is also characterized by the VC-dimension. 
Interestingly, however, the algorithm in \cref{thm:sampleComplexity} is not empirical risk minimization (ERM). In fact, it is not too hard to show that ERM provably fails in this setting (see \cref{sec:intro:technicalOverview:challenges}).
Instead, \cref{thm:sampleComplexity} proposes an iterative procedure that solves a constrained optimization problem, which we refer to as \textit{\pERM{}}. 
This yields an improper learning algorithm, and we further show that proper learning is in fact impossible for (smooth) positive-only learning (see \cref{rem:necessityOfImproper}). 
Both the failure of ERM and the impossibility of proper learning underscore the fundamental gap between classical analysis of positive-only learning and the smoothed analysis framework studied in this paper. 
\smallskip

\noindent \textbf{Improved Sample Complexity for $q = 1$.} Under an additional assumption  (\cref{asmp:densityLB}) that complements generalized smoothness (\cref{asmp:smoothness}) by providing a lower bound on $\cDtrue(S)$ for measurable sets, we further improve the sample complexity in the case $q = 1$ to $\wt{O}\sinparen{\nfrac{\vc{}(\hyH)}{\eps}}$ (\cref{thm:sampleComplexity:bothSided}). This matches the optimal rate for realizable learning and is therefore optimal for smooth positive-only learning with $q = 1$, despite the latter being a harder problem. 
To see this, note that in realizable PAC Learning, one has sample access to $\cDtrue$---one can simply ignore the labels of the samples---and this is sufficient to ensure smoothness with $q=\sigma=1$.
{Under this assumption, prior works \cite{denis1990positive,lee2024unknown} can also be used for positive-only learning, but \cref{thm:sampleComplexity:bothSided} improves their sample complexity: \cite{lee2024unknown} requires $O(\eps^{-4})$ samples versus our $O(\eps^{-1})$; \cite{denis1990positive} only handles $q=1$ where it matches our sample complexity of $O(\eps^{-1})$ while we can also handle $q>1$ (\cref{thm:sampleComplexity}). Finally, we stress that without this assumption (\cref{asmp:densityLB}), these works are inapplicable while our main results (\cref{thm:sampleComplexity,thm:main}) remain valid.}

\subsubsection{Smooth Computational Efficiency of Positive-Only Learning} \label{sec:intro:computationallyEfficientResult}

We next investigate the smooth computational complexity of positive-only learning under the additional assumption of a non-trivial mass of positive samples (\cref{asmp:posFraction}). 

\begin{inftheorem}[Computational Complexity; see \cref{thm:main}]\label{mainthm:algo}\label{infthm:main}
    Suppose \cref{asmp:smoothness,asmp:posFraction} hold and fix any constant $C > 0$. 
    Fix $\eps,\delta\in (0,\nfrac{1}{2})$, and let $k$ be such that degree-$k$ polynomials with range $[-C,\infty)$ $\zeta$-approximate $\hyH$ in \mbox{$L_1$-norm} with respect to $\cD$ for $\zeta\leq \inparen{\nfrac{\eps\alpha\sigma}{(C+1)}}^{\poly(q)}$.
    Then, there is an algorithm that, given $n=\poly(d^k,\nfrac{1}{\zeta},\log{\nfrac{1}{\delta}})$ independent samples from $\cPtrue$ and $\cD$, outputs a hypothesis $\wh{h}$ such that, with probability $1-\delta$, 
    \[
        \Pr\nolimits_{\cDtrue}\binparen{\wh{h}(x)\neq h^\star(x)}\leq \eps\,.
    \]
    The algorithm runs in time $\poly(n)$.
\end{inftheorem}
    Therefore, efficient smooth positive-only learning is possible when the optimal hypothesis $h^{\star}$ has a low complexity in the sense that low-degree polynomials whose range is bounded below can well approximate it under distributions $\cD$ (\cref{def:poly_approx}). 
        The degree of polynomials $k$ required typically scales as $k = \poly\!\inparen{\nfrac{1}{\eps}}$. 
        This exponential dependence on $\nfrac{1}{\eps}$ is necessary for any SQ algorithm (such as the one in \cref{mainthm:algo}) due to a recent lower bound \cite{diakonikolas2024statistical} (see \cref{rem:SQlowerbound} for details).

        Notably, \cref{mainthm:algo} only requires an approximation with respect to the accessible unlabeled distribution $\cD$ and not the inaccessible distribution $\cDtrue$.
        This enables the use of recent testable learning techniques when certain sandwiching polynomials exist \cite{rubinfeld2023testing,gollakota2023momentMatching,gollakota2023universal} to ensure that the required polynomial approximation guarantee holds by, roughly, verifying that the moments of samples from $\cD$ match those of a known reference distribution (see \cref{rem:testing} for details). 
        Finally, we note that the requirement of $L_1$-approximability is natural and is required by nearly all existing computationally efficient agnostic learning algorithms.
        We additionally require the range of the $L_1$-approximating polynomials to be lower bounded by some constant. 
        This requirement is always satisfied whenever $L_1$-sandwiching polynomials exist but is much weaker.
        For instance, it holds for all classes with finite Gaussian surface area when $\cD$ is Gaussian \cite{2026nonNegativeApproximation}.

\subsection{Generalizations and Application to Truncated Statistics}\label{sec:intro:application} 

In this section, we {first present applications of our main results to get the fastest and/or most general known algorithms for problems in truncated statistics and, then, generalize our results.}

\subsubsection{Truncated Estimation with Unknown Survival Set}\label{sec:intro:application:truncation}
        Our {main results have} implications for efficient statistical inference from truncated samples, a problem with a long history of study in Statistics and Econometrics \cite{Galton1897,Pearson1902,PearsonLee1908,Lee1915,fisher31,maddala1983limited,Cohen91,hannon1999estimation,raschke2012inference} that has attracted considerable recent interest in Theoretical Computer Science \cite{daskalakis2018efficient, Kontonis2019EfficientTS, daskalakis2019computationally, trunc_regression_unknown_var, ons_switch_grad, lee2023learning, diakonikolas2024statistical,galanis2024discreteTruncated,lee2024unknown,nagarajan2020truncatedMixtureEM,fotakis2020booleanProductTruncated,nagarajan2023EMMixtureTruncation,tai2023mixtureCensored,kouridakis2026truncation}.

        One important goal in truncated statistics is to estimate the parameters of a distribution when samples are observed only if they fall in some \textit{unknown} survival set $S^\star\subseteq\R^d$.
        A bit more formally, 
        there is a parametric family $\hyD=\inbrace{\cD(\theta)\colon \theta}$ with target parameter $\theta^\star$. 
        The goal is to find an estimate $\wh{\theta}$ of $\theta^\star$ given samples from the truncation\footnote{The truncation of a distribution $\cD$ to a set $S$ is the conditional distribution of $\cD$, conditioned on the event that a point lies in $S$. When $\cD$ has a density, the truncated density is given by $\frac{1}{\cD(S)} \cdot \ind\inparen{x \in S} \cdot \cD(x)$} of the target distribution $\cDtrue = \cD(\theta^\star)$ to the set $S^\star$ -- without \textit{any} access to $S^\star$.
        For this problem to be tractable, it is necessary to assume a lower bound on the mass of $S^\star$, specifically that $\cDtrue(S^\star)\geq \alpha$.
        The state-of-the-art algorithm for this problem \cite{lee2024unknown} works in three phases:
        \begin{enumerate}[itemsep=0pt]
            \item \textit{Phase 1:}~ First, it finds an unlabeled density $\cD$ close to $\cDtrue$ (in the sense of \cref{asmp:smoothness});
            \item \textit{Phase 2:}~ Then, it learns $S$ such that the symmetric difference $\cDtrue(S\triangle S^\star)\approx 0$; and 
            \item \textit{Phase 3:}~ Finally, it recovers $\theta^\star$ using  learned set $S$. 
        \end{enumerate}
        The bottleneck in their algorithm is Phase 2 which falls under the setting of \cref{infthm:main} and, substituting the algorithm in Phase 2 by the algorithm in \cref{infthm:main} gives us the following improvement upon \citet{lee2024unknown}'s result.

        \begin{infcorollary}[Truncated Estimation with Unknown Survival Set; see \cref{cor:truncation}]\label{infcor:truncation} 
            Let the family $\hyD=\inbrace{\cD(\theta)\colon \theta}$ satisfy the assumptions of \cite{lee2024unknown} and let $\alpha=\Omega(1)$. %
            Fix any $\eps, \delta \in (0, \nfrac{1}{2})$, and let $S^\star$ be $\poly(\eps)$-approximable by degree-$k$ polynomials with range within $[-O(1),\infty)$ in $L_1$-norm with respect to $\hyD$.
            There is an algorithm that given $n = \poly(d^k, \nfrac{1}{\eps}, \log{\nfrac{1}{\delta}})$ independent samples from $\cD(\theta^\star)$ truncated to $S^\star$ 
            outputs parameter $\wh{\theta}$ such that with probability $1-\delta$, 
            \[ \tv{\cD(\wh{\theta})}{\cD(\theta^\star)} \leq \eps\,.\]
            The algorithm runs in time $\poly(n)$. 
        \end{infcorollary} 
        We stress that before our work the only known algorithms in this setting are by \cite{Kontonis2019EfficientTS,lee2024unknown}, both of which require $S^\star$ to be approximable in $L_2$-norm.
        In contrast, the above result works with just an approximation in the $L_1$-norm with, \eg{}, non-negative polynomials.
        This change enables us to obtain the following new results from \cref{infcor:truncation}:
        \begin{center}
            \textit{The first efficient algorithm for truncated parameter estimation, when $S^\star$ is an arbitrary\\ function of $O(1)$ halfspaces}\footnote{
                Functions of $k$ halfspaces are a significant generalization of intersections of $k$ halfspaces and are of the form:
                $b(h_1,h_2,\dots,h_k)$ where $b$ is an arbitrary boolean function $b\colon \zo^k\to \zo$ and $h_1,h_2,\dots,h_k$ are $k$ halfspaces.
                In particular, they can capture complex \textit{non-convex} sets, while intersections of halfspaces are necessarily convex.
                } \textit{and $\hyD$ sub-exponential and satisfies the assumptions of \cite{lee2024unknown}.}
        \end{center}
        Notably, this holds for Gaussians and product exponential distributions. 
        For the specific case of Gaussians, this is known already due to \cite{lee2024unknown} and the $L_2$-sandwiching approximators for decision trees of halfspaces \cite{klivans2023testable,gopalan_2010_fooling}.

    \subsubsection{Detecting Truncation}\label{sec:intro:application:detection}
        Our next application is also from Truncated Statistics \cite{Cohen91,maddala1983limited}.
        In particular, to the problem of testing whether a given set of samples is \iid{} from a distribution $\cQ$ or \iid{} from the truncation $\cQ(S^\star)$ of $\cQ$ to some unknown set $S^\star$.
        
        A growing line of works in Theoretical Computer Science study this problem \cite{de2023detectingConvex,he2023testingjuntatruncation,beretta2026feature,de2024detecting}.
        It is not too hard to show (see the work of \citet*{de2024detecting}) that if the mass of $S^\star$ under $\cQ$ is very close to $1$ then this problem is hopeless.
        This is roughly because as the mass $\cQ(S^\star)$ approaches 1, the distributions $\cQ$ and $\cQ(S^\star)$ become indistinguishable. 
        Hence, one focuses on the case where $\cQ(S^\star)\leq 1-\beta$ for some known constant $\beta=\Omega(1)$.
        \citet*{de2024detecting} study this problem in the special case where $S^\star$ is a degree-$k$ polynomial threshold function (PTF) and give a $O(d^{k/2})$ sample and time algorithm that works when $\cQ$ satisfies the following:
        \begin{enumerate}[itemsep=0pt,label=(C\arabic*)]
            \item $\cQ$ is a hypercontractive;
            \label[none]{asmp:detectingTruncation:hypercontractive}
            \item $\cQ$ is an \iid{} product distribution;\label[none]{asmp:detectingTruncation:product}
            \item $\cQ$ is \textit{known} in the following sense: we can \textit{exactly} compute an orthonormal basis of polynomials up to degree $k$ in time $O(d^{k/2})$.
            \label[none]{asmp:detectingTruncation:access}
        \end{enumerate}
        Our \cref{mainthm:algo} implies efficient algorithms for this problem for a broader set of truncation sets (beyond PTFs) and it also allows $\cQ$ to be a non-product distribution and does not require it to be known as long as one has sample access to $\cQ$ (relaxing \cref{asmp:detectingTruncation:product,asmp:detectingTruncation:access}).
        \begin{infcorollary}[see \cref{cor:detectingTruncation,cor:detectingTruncation:gen}]\label{infcor:detectingTruncation}
            Consider the setting described above with $\beta=\Omega(1)$.
            Suppose one has an efficient sampling oracle for $\cQ$.
            Fix any $\delta\in (0,\nfrac{1}{2})$ and $k\geq 1$ such that $S^\star$ is $O(1)$-approximable by degree-$k$ polynomials $q\colon \R^d\to (-O(1),\infty)$ in the $L_1$-norm with respect to $\cQ$.
            There is an algorithm that given $n=d^{k}\cdot \polylog({\nfrac{1}{\delta}})$ samples from $\cQ$ or $\cQ(S^\star)$, with probability $1-\delta$, correctly detects whether the samples were from $\cQ$ or $\cQ(S^\star)$ in $\poly(n)$ time.
            
            Moreover, the algorithm also works (with the same guarantees) with an \emph{approximate} sampling oracle--that outputs samples from a distribution $\cR$ such that $\cQ$ is $(\sigma,q)$-smooth (for $q=O(1)$ and $\sigma=\Omega(1)$; see \cref{asmp:smoothness}) with respect to $\cR$. 
            (In particular, it is sufficient to have $\chidiv{\cR}{\cQ}=O(1)$; \cref{lem:closenessFromChiSquare}.)
        \end{infcorollary}
        This result, in particular, shows that detection of truncation is possible with only an \textit{approximate} sampling oracle to $\cQ$, which is significantly less restrictive than \cref{asmp:detectingTruncation:access}, and even when $\cQ$ is not a product distribution, which is much weaker than \cref{asmp:detectingTruncation:product}.
        The latter, in particular, makes progress on an open question of \cite{de2024detecting} asking for efficient truncation-detection algorithms that work for non-product distributions; see the associated STOC talk \cite{de2024detectingFOCStalk}.
        Further, since the above algorithm can handle truncation sets beyond PTFs, we get novel algorithms in the following case:
        \begin{center}
            \textit{$S^\star$ is any function of $O(1)$-halfspaces and $\cQ$ is \textit{any} sub-exponential distribution \cite{ksv2026sandwiching}.}
            
        \end{center}
        Recall that functions of halfspaces can be non-convex.
        We stress that prior to our work no efficient algorithms for detecting truncation were known for non-convex hypothesis classes beyond PTFs, even when $\cQ$ was \textit{known} to be $\cN(0, I)$.
        Further, even for a single halfspace, no efficient algorithms were known to detect truncation when only sample access to $\cQ$ was available.

    \begin{remark}
        While our algorithm is applicable in much more general settings than prior works \cite{de2023detectingConvex,de2024detecting}, in the specific cases where their algorithms can be applied, the number of samples needed are $\poly(d)$ {(for convex bodies under Gaussian distributions)} and $O(d^{k/2})$ {(for PTFs under hypercontractive product distributions satisfying \ref{asmp:detectingTruncation:access})}. 
        These results are particularly surprising because the sample complexity is smaller than the sample complexity of learning $S^\star$ in the same settings, which are $d^{O(\sqrt{d})}$ and $O(d^k)$ respectively. 
        Since our approach is based on a learning algorithm, we do not obtain this saving and require $d^{O(\sqrt{d})}$ and $O(d^k)$ samples respectively but we are able to design a more widely applicable algorithm. 
    \end{remark} 

\subsubsection{Learning with Positive Samples and Lists of Reference Samples}\label{sec:intro:application:listDecoding}

Following our main results, one may ask if positive-only learning is possible with less information than access to an underlying ``reference'' distribution $\cD$. In this section, {we show that this is indeed true:} in addition to the positive samples, it is sufficient to have a list of distributions that contains one \textit{unknown} distribution that ``witnesses'' the smoothness of $\cDtrue$.
\begin{definition}[List Decoding Model]
    \label{def:list}
    Fix a distribution $\cDtrue$ over domain $\cX$ and numbers $\ell,q\geq 1$ and $\sigma\in (0,1]$.
    In an instance of the list decoding model, specified by parameters $(\cDtrue,\ell,\sigma,q)$, one is provided a list of $\ell$ distributions $\cD_1, \cD_2, \dots, \cD_\ell$ over $\cX$ such that for some \textit{unknown} index $1\leq i^\star\leq \ell$,
    {$\cD_{i^\star}$ is such that $\cDtrue$} is $(\sigma,q)$-smooth with respect to $\cD_{i^\star}$ in the sense of \cref{asmp:smoothness}. 
    The distributions $\unlabeled_j$ (for $j\neq i^\star$) can be arbitrary.
\end{definition}
This model generalizes smooth positive-only learning, which assumes sample access to $\cD_{i^\star}$. 
It is inspired by real-world scenarios where the unlabeled samples are scraped from unreliable sources that may be adversarially corrupted, which, in some cases, can be used to obtain a list of unlabeled samples satisfying \cref{def:list} (see \cref{rem:list}).

Our main result for this list-decoding model shows that it is also characterized by VC dimension and its computational complexity can also be bounded by polynomial approximability.
\begin{inftheorem}[see \cref{thm:listDecoding}]\label{infthm:list}
    Consider a PAC learning instance with      
        class $\hyH$,
        feature distribution $\cDtrue$, and 
        target $h^\star\in \hyH$.  
    Consider the corresponding instance of the List Decoding Model in \cref{def:list} with $\sigma=\Omega(1)$ and $\ell=O(1)$ and parameter $q\geq 1$.
    \begin{itemize}
        \item[$\triangleright$] \textit{(Sample Complexity)}~~ 
            $\wt{O}\!\inparen{\eps^{-2q}\cdot \inparen{\vc{}(\hyH)+\log{\nfrac{1}{\delta}}}}$ independent positive samples are sufficient to $(\eps,\delta)$-PAC learn $\hyH$ with respect to $\cDtrue$.
        \item[$\triangleright$] \textit{(Computational Complexity)}~~
            Suppose $\cDtrue(\optset)\geq \Omega(1)$.
            Let $k\geq 1$ be such that $\hyH$ is $\zeta$-approximable (with range $[-O(1),\infty)$) by degree-$k$ polynomials in the $L_1$-norm with respect to $\cD_{i^\star}$ for $\zeta\leq \eps^{\,O(q^2)}$.
            Then, $n=\poly(d^k,\nfrac{1}{\zeta},\log{\nfrac{1}{\delta}})$ positive samples and $\poly(n)$ time is sufficient to $(\eps,\delta)$-PAC learn $\hyH$ with respect to $\cDtrue$.
    \end{itemize}
\end{inftheorem}
This theorem significantly extends our main result (\cref{infthm:main}), which is equivalent to the special case of $\ell=1$. Despite this, interestingly, minimal changes to the algorithm in our main result (\cref{infthm:main}) are sufficient to give the above result. The proof of \cref{infthm:list} appears in \cref{sec:intro:application:listDecoding} along with a formal statement.

\begin{remark}[Obtaining a List of Distributions From Adversarially Corrupted Data]\label{rem:list}
  In this remark, we show that if one is given adversarially corrupted unlabeled samples, then one can satisfy the guarantees of the above list-decoding model. Suppose that $\cDtrue$ is a Gaussian distribution $\cN(\muStar,\SigmaStar)$. Consider a data source that is $(1-\gamma)$-corrupted for some $\gamma\in (0,1]$: with probability $\gamma$, the source outputs a sample from $\cDtrue$, and, otherwise with probability $1-\gamma$, it outputs an arbitrary sample that may depend on the past samples and the knowledge of the learning algorithm. In \cref{sec:advCorruption}, we show that one can efficiently convert an $(1-\gamma)$-corrupted source for any $\gamma>0$, into one satisfying the requirements of the List Decoding Model with $\ell= \poly(\nfrac{1}{\gamma})$. Moreover, under mild assumptions, we also extend this to the case where $\cDtrue$ is an exponential family distribution.
\end{remark}

\subsection{Technical Overview}\label{sec:technicalOverview}
    {We initiate the smooth analysis of positive-only learning.}
Below we briefly explain the key ideas and challenges in the proofs of our main results. %

    \subsubsection{Challenges in Using Existing Approaches for PU Learning}
        \label{sec:intro:technicalOverview:challenges}
        First, we discuss challenges in using existing techniques from positive and unlabeled (PU) learning, which is a usual method to bypass the impossibility results in positive-only learning.

        \paragraph{{Issue I: ERM is insufficient}.}
            A standard approach to learning with positive and unlabeled samples is to reduce the problem to agnostic learning.
            In this reduction, the agnostic learning instance is created by imagining all unlabeled samples to be negative samples.
            The resulting feature distribution is a mixture of the form (for some weight $\gamma\in (0,1)$):
            \[
                \cM = \gamma \cPtrue + (1-\gamma) \cD\,.
                \yesnum\label{eq:overview:mixture}
            \]
            If the unlabeled samples are drawn from the underlying distribution $\cDtrue$ (\ie{}, if $\cD=\cDtrue$), then $\optset$ (the set of positive examples) is also the optimal solution for the agnostic learning problem and ERM is sufficient for learning.
            However, in general, when $\cD \neq \cDtrue$, the optimal hypothesis of the resulting problem can be very different from $\optset$ and, hence, ERM is not sufficient and a different approach is needed to even bound the sample complexity of {smooth positive-only learning} (see \cref{rem:necessityOfImproper}).

        \paragraph{Additional Challenges to Get Computational Efficiency.}
        To illustrate additional challenges in obtaining computationally efficient algorithms, let's suppose that $\cD$ satisfies the following strengthening of \cref{asmp:smoothness}: for all measurable sets $S$,
            \[
                \cDtrue(S)
                    \leq O(\cD(S)^{1/q})
                \qquadand
                \cDtrue(S)
                    \geq \Omega(\cD(S)^{q})
                \,,
                \yesnum\label{eq:bounded_chisquare}
            \]
            These requirements turn out to be sufficient to ensure that $\optset$ is (close to) the optimal solution of the agnostic learning problem created above and, hence, to ensure that ERM learns $\optset$.\footnote{See \cref{rem:densityLB} for a discussion of why the assumption in \eqref{eq:bounded_chisquare} is significantly stronger than \cref{asmp:smoothness}.}
            Thus, under the simplification in \cref{eq:bounded_chisquare}, the main challenge is developing an efficient algorithm to solve the agnostic learning problem created.
            However, agnostic learning is computationally challenging  \cite{guruswami2009hardness, feldman2006new, daniely2016complexity}.
            Almost all efficient algorithms for agnostic learning are based on the now-standard \lreg{} algorithm of \citet*{kalai2008agnostically}.
            As a result, all of these algorithms only work if the hypothesis class $\hyH$ is approximable by polynomials with respect to the feature distribution, say, $\cQ$.
            (The specific definition of polynomial approximation is not necessary to follow the remainder of the technical overview; the definition appears in \cref{sec:preliminary}.)
            Hence, for efficiently learning $\hyH$ one needs a bound on the degree of polynomials which approximate $\hyH$ with respect to $\cQ$.
            There are several challenges in obtaining these bounds.

            \paragraph{Issue II~~(Standard Polynomial-Approximation Guarantees Do Not Apply to $\cM$).} First, existing results on polynomial approximation are only known for special distributions $\cQ$ -- such as log-concave distributions or the uniform distribution over the Boolean hypercube.
            However, because distribution $\cPtrue$ and $\cD$ can be very different from each other, their mixture $\cM$ (\cref{eq:overview:mixture}) may not be log-concave or uniform (see \cref{fig:mixture}).\footnote{One might suspect that it is sufficient to have separate approximating polynomials $p_1(\cdot)$ and $p_2(\cdot)$ with respect to $\cPtrue$ and $\cD$ respectively.
            Unfortunately, this is not the case: for using \lreg{}, we require both of these polynomials to be the same, \ie{}, $p_1=p_2$.} 
            \begin{figure}[htb!]
                \centering
                \vspace{4mm}
                \includegraphics[width=0.75\linewidth]{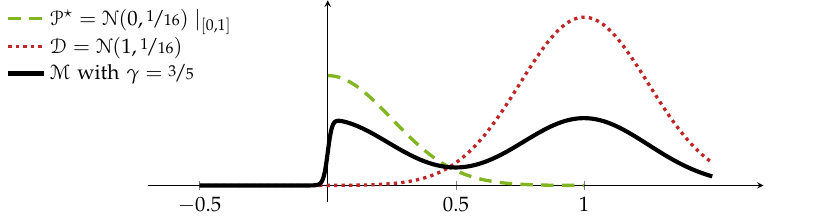}
                \caption{Example of mixture distribution $\cM$ constructed by the distribution of positive examples $\cPtrue$ and the ``imperfect'' distribution of unlabeled examples $\cD$ (also see \cref{eq:overview:mixture}). 
                Here, $\cPtrue$ is a truncation of a Gaussian distribution to a convex set (an interval, namely, $\optset=[0,1]$) and $\cD$ is a Gaussian distribution.
                Nevertheless, $\cM$ is non-log-concave.
                }
                \label{fig:mixture}
            \end{figure}
            
            \paragraph{Issue III~~(Approximating Polynomial w.r.t.\ $\cD$ Can Be Poor for $\cPtrue$ {{and vice versa}}):}
                One idea to overcome this is to use the facts that (1) $\cPtrue$ is a truncation of $\cDtrue$ to a set of constant mass (namely, $\optset$; \cref{asmp:posFraction}) and (2) $\cDtrue$ and $\cD$ are close to each other (they satisfy \eqref{eq:bounded_chisquare}).
                Using this, one can show that if $p\colon\R^d\to \R$ $\eps$-approximates a hypothesis $h$ with respect to $\cD$ in $L_q$-norm for the constant $q$ from \eqref{eq:bounded_chisquare}, 
                then it is an $O(\eps)$-approximating polynomial with respect to $\cPtrue$ in $L_1$-norm.
                (The same also holds if we convert an approximating polynomial for $\cPtrue$ to one for $\cD$.)
                Hence, to ensure $L_1$-norm approximation with respect to both distributions, $p(\cdot)$ needs to approximate $h$ in $L_q$-norm for one of the two distributions.
                This is a very strong requirement, and in fact, to the best of our knowledge, no useful approximation results are known for $q \geq 3$.%

        \enlargethispage{\baselineskip}
            
            \paragraph{Issue IV~~(Using Bridge Distributions Needs $L_2$-Approximation and May Not Exist):} 
                Another idea is to use the concept of a bridge distribution, which has appeared in some recent works (\eg{}, \cite{Kontonis2019EfficientTS,kalavasis2024transfer,lee2024unknown}).
                In some special cases, these works show that even though $\cD$ and $\cDtrue$ are far from each other -- say, they have a very large $\chi^2$-divergence -- one can construct a bridge distribution $\mu$ such that $\mu$ is close to both $\cD$ and $\cDtrue$ in the $\chi^2$-divergence (\ie{}, $\chidiv{\cD}{\mu},\chidiv{\cDtrue}{\mu}$ are small).
                (This construction crucially uses the fact that the $\chi^2$-divergence does not satisfy the triangle inequality.)
                If such a bridge distribution exists, then one can convert an $L_2$-approximating polynomial with respect to $\mu$  to an $L_1$-approximating polynomial with respect to \textit{both} $\cD$ and $\cPtrue$.
                This approach has a few downsides:
                First, it only works in restricted cases where a bridge distribution exists.
                Second, a bridge, even if it exists, might not admit approximating polynomials.
                Finally, even if the first two issues can be handled, this approach inherently require{s} an $L_2$-approximating polynomial to transfer an approximating polynomial with respect to $\mu$ to an approximating polynomial with respect to $\cD$ and $\cPtrue$ (using, \eg{}, the Cauchy--Schwarz inequality).
            In contrast, our result works with $L_1$-approximating polynomials with respect to just $\cD$ with a lower bound on their range.

        \subsubsection{Our Approach: Reduction to Constrained Learning} 
            Instead of trying to efficiently solve the above reduction to agnostic learning, we take a dramatically different approach:
            we reduce our problem to several \textit{constrained} learning problems.
            The starting point of this approach is practical heuristics used for learning from positive and unlabeled samples which find a hypothesis $H$ that (when viewed as a set) contains the minimum number of unlabeled samples subject to selecting all positive samples \cite{bekker2020learning}.
            This can be formalized as the following program: given sets $P$ and $U$ of positive and unlabeled samples and a tolerance $\rho\geq 0$ %
            \[
                \argmin_{H\in \hyH}~~ \frac{\abs{H\cap U}}{\abs{U}}
                \,,\quad\text{ such that}\,,\quad 
                \frac{\abs{H\cap P}}{\abs{P}} \geq 1-\rho
                \,.
                \tag{\pERM{}}
            \]

        \paragraph{Sample-Efficient Algorithm (see \cref{sec:sampleEfficiency}).}
            In general, the optimal solution of \pERM{} can be far from $\optset.$
            At the core, this is because \cref{asmp:smoothness} allows $\supp(\cDtrue)$ to be a strict subset of $\supp(\cD)$.
            To gain some intuition about why this introduces challenges, consider \cref{fig:intersection_halfspaces}, which shows that due to $\supp(\cDtrue)\subsetneq \supp(\cD)$, a low-VC-dimension hypothesis $H$ (say a halfspace), can appear like a high-VC-dimension halfspace (say intersections of $k$ halfspaces for large $k$) when provided samples from $\cD$.\footnote{In some of the earlier applications (\eg{}, learning from smoothed positive examples and the practical applications in \cref{sec:intro:realWorldApplications}) one cannot ensure that $\supp(\cD)=\supp(\cDtrue)$ and, so it is crucial to develop algorithms that rely solely on \cref{asmp:smoothness}.}
            To obtain a sample-efficient algorithm, our main insight is to solve many instances of the \pERM{} program. 
            Concretely, we solve $\nfrac{1}{\eps}$ instances of the program, where the instances are constructed iteratively and, roughly speaking, the $i$-th instance focuses on the region $B_{i-1}$ of the domain where the $(i-1)$-th instance is not guaranteed to perform ``well'' (by deleting all unlabeled examples which are not in $B_{i-1}$). 
            At the end we output the intersection of the hypotheses $H_1, H_2,\dots,H_{1/\eps}$ which are solutions to each instance.
            Since in practical applications one must often deal with $\supp(\cDtrue)\subsetneq \supp(\cD)$ (\cref{sec:intro:realWorldApplications}), we believe this approach of solving many instances of \pERM{} may also be of interest in practice.

        \begin{figure}[h!]
            \centering
            \begin{subfigure}[b]{0.3\textwidth}
                \includegraphics[width=0.85\textwidth]{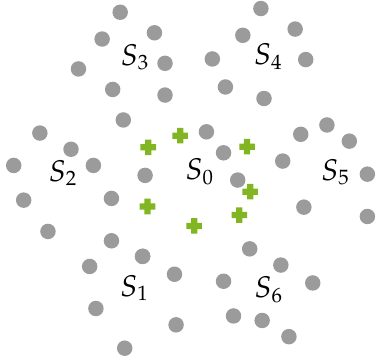}
                \caption{\textit{{Positive and Imperfect Unlabeled} Samples}}
                \label{fig:sintersection_halfspaces_1}
            \end{subfigure}
            ~~~
            \begin{subfigure}[b]{0.32\textwidth}
                \includegraphics[width=\textwidth]{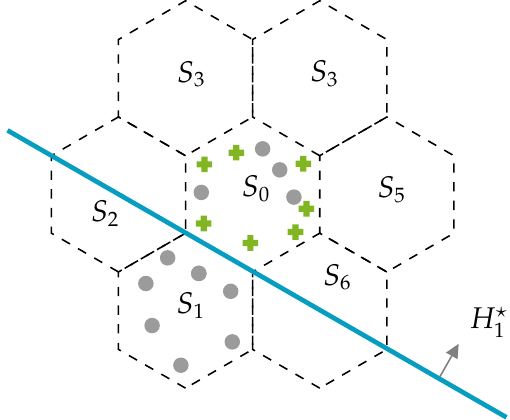}
                \caption{\textit{{Positive and Unlabeled} Samples with $\optset$ as halfspace}}
                \label{fig:intersect_halfspaces_2}
            \end{subfigure}
            ~~~
            \begin{subfigure}[b]{0.32\textwidth}
                \includegraphics[width=\textwidth]{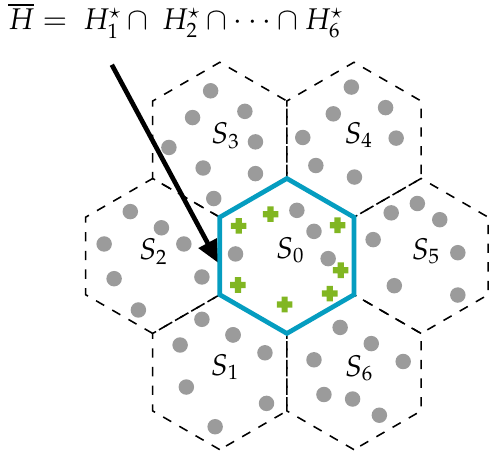}
                \caption{\textit{{Positive and Unlabeled} Samples with $\optset$ as an intersection of six halfspaces}}
                \label{fig:intersect_halfspaces_3}
            \end{subfigure}
            \caption{
                This figure illustrates one challenge in learning with {Positive and Imperfect Unlabeled} samples {which arise in smoothed positive-only learning}.
                Subfigure (a) presents the observed positive and (imperfect) unlabeled samples.
                Given these examples, it is impossible to distinguish between Scenarios (b) and (c). 
                Hence, in particular, due to imperfections in the unlabeled samples a hypothesis with low VC-dimension (\eg{}, a single halfspace in (b)) can appear like a high-VC-dimension hypothesis (\eg{}, intersections of many halfspaces in (c)).
            } 
            \label{fig:intersection_halfspaces}
        \end{figure}

        \paragraph{Computationally Efficient Algorithm (see \cref{sec:computationalEfficiency}).}
            A natural idea is to convert a sample-efficient algorithm into a computationally efficient one by solving \pERM{} efficiently.
            However, one cannot use techniques for designing efficient learning algorithms in the realizable setting to solve this problem since the optimal value  of \pERM{} is generally bounded away from zero (so does not fit the standard template for efficient realizable learning algorithms that reduce the problem to checking feasibility of, \eg{}, a linear program).  
            At this point, given the challenges we discussed for efficiently implementing ERM (in the previous sub-section), efficiently implementing \pERM{} might appear hopeless: 
            instead of solving an unconstrained learning problem, which is already hard, we now have to solve a constrained learning problem -- that is seemingly harder.

            A key advantage of this approach is that it decouples the samples from $\cD$ and $\cPtrue$.
            Therefore, while solving this program, we avoid dealing with mixture distributions, which we have seen are challenging to deal with.
            We note that this is why \cref{infthm:main} only requires $L_1$-approximation with respect to $\cD$, which we have sample access to.
            This contrasts the attempts we discussed earlier in \cref{sec:intro:technicalOverview:challenges} that require a stronger approximation guarantee (in $L_q$-norm, for which no results are known when $q \geq 3$) or use the existence of some ``bridge'' distributions that are only known to exist for specific distribution families (\eg{}, Gaussians).  
            In contrast, our algorithm only needs an approximation guarantee in $L_1$-norm (regardless of the value of $q$ in \cref{asmp:smoothness}) and makes no further assumptions about the distribution family $\cD$ and $\cDtrue$ belong to (in particular, they can be non-Gaussian).

            Next, we describe how we design an efficient approximation algorithm to solve (one instance of) \pERM{}. 
            This combined with the iterative sample-efficient algorithm described above is in spirit sufficient to complete the proof of \cref{thm:main}.
            That said, compared to the analysis for sample complexity, we need to account for some additional challenges.
                First, we need to ensure that one can efficiently obtain samples from each of the {positive-only learning} instances constructed (\cref{lem:main:rejectionSampling}). 
                Second, to solve each instance efficiently, we need to ensure that the $L_1$-approximation guarantee is maintained even though the $i$-th instance focuses on the subdomain $\R^d\setminus B_{i-1}$, which is determined by the $(i-1)$-th instance.
                Finally, since we only solve each \pERM{} instance approximately, it is possible that after discarding unlabeled samples within $B_{i-1}$, it no longer holds that $\supp(\cD)\supseteq \supp(\cPtrue)$, resulting in invalid {smooth positive-only learning} instances and we need to adapt the construction to avoid this (\cref{lem:main:assumptions}).

        \paragraph{An Approximation Algorithm for \pERM{} (see \cref{thm:constReg} and  \cref{sec:efficientPERM}).} 
            Up to this point, we have made several careful decisions to arrive at the program we want to solve; while avoiding several bottlenecks in other natural approaches. 
            Now, we are ready to introduce the final ingredients of our approach. 
            The final component of our proof is a constrained version of \lreg{} that ``approximately'' solves \pERM{}.
            Constrained versions of \lreg{} have appeared in other works, in particular for efficient reliable learning, which, roughly speaking, require classifiers to make no errors when they make a prediction but allows them to abstain from making a prediction on a fraction of the samples \cite{kanade2014independentRealiable,kalai2012reliable}.
            We provide a new constrained version of \lreg{} -- which we call \constrainedreg{} -- along with novel analysis showing that this algorithm is a pseudo-approximation algorithm for \pERM{} whenever $\hyH$ can be approximated by polynomials in $L_1$-norm with respect to $\cD$.
            
            At a high level, the algorithm operates as follows: it solves a constrained polynomial regression problem to find a polynomial $p(\cdot)$ and outputs a polynomial threshold function $\wh{h}(\cdot)=\mathds{1}\sinbrace{x \colon p(x)\geq t}$ constructed by thresholding $p(\cdot)$ at a suitable value $t$.
            The proof of correctness leverages ideas from the analysis of \lreg{}, but needs to account for some technicalities:
            First, we show that as long as $t$ is suitably bounded away from 1, then using standard uniform convergence arguments $\wh{h}$ is feasible for \pERM{}.
            This adjustment, however, is inconsistent with \lreg{}, which selects the best threshold in $[0, 1]$.
            We show that \lreg{}'s analysis is robust to this change.
            Second, unlike \lreg{}, we aim to show that \constrainedreg{} outputs a hypothesis that is feasible for a constrained optimization problem (\pERM{}); and significant additional analysis is needed to establish feasibility and pseudo-optimality.
        
        \subsubsection{Learning in the List Decoding Model}
            The List Decoding Model (\cref{def:list}) is even more challenging than the previous model because we lack sample access to a distribution close to $\cDtrue$. A natural idea is to combine the list of distributions $\cD_1,\dots,\cD_\ell$ to approximate $\cDtrue$, for example via the uniform mixture $\cM$. This approach is promising since one can show that $(\cDtrue,\cM)$ satisfies \cref{asmp:smoothness} with $\wt{\sigma}=\sigma \ell^{-1/q}$ and $q'=q$: indeed, for any measurable set $S$, we have $\cM(S)^{1/q}\geq \ell^{-1/q}\cD_{i^\star}(S)^{1/q}\geq \sigma\ell^{-1/q}\cDtrue(S)$, where the last inequality holds as $(\cDtrue,\cD_{i^\star}$) satisfies \cref{asmp:smoothness}. While this suffices for bounding sample complexity, achieving computational efficiency is challenging because we lack guarantees on $\cD_i$ for $i\neq i^\star$, causing polynomial approximability of $\optset$ with respect to $\cM$ to fail. Other approaches that combine $\cD_1,\dots,\cD_\ell$ to yield a distribution $\cM$ close to $\cDtrue$ in the sense of \cref{asmp:smoothness} also encounter the failure of polynomial approximation. 
            Instead of trying to combine the list of distributions, we run the algorithm from \cref{infthm:main} separately on each unlabeled distribution $\cD_i$ to obtain hypotheses $H_1,H_2,\dots,H_\ell$ and output the intersection of these hypotheses $\bigcap_i H_i$.

    \subsection{Practical Motivation for {Smoothed Positive-Only Learning}}
    \label{sec:realWorldApplications}
    \label{sec:intro:realWorldApplications} 
    \noindent While we have highlighted our results insofar as how they apply to other learning theory problems, we next motivate the study of smoothed {analysis} of positive-only learning itself.
    We begin by discussing real-world applications where obtaining ``perfect'' unlabeled examples is difficult or even impossible, and, hence, existing methods must grapple with \textit{imperfect} unlabeled data -- precisely the challenge that we formalize. %
    For further examples, see the surveys by \citet{elkan2008pu,plessis15convex} and a recent book \cite{jaskie2022positive} on learning from positive and unlabeled data.
    
    \paragraph{Applications in Bioinformatics and Medicine.} %
    Consider \textit{disease-causing gene identification} -- a central task in Genetics and Molecular Biology with major healthcare implications \cite{mordelet2011prodige,yang2012positive}.
    Identification is very expensive due to its reliance on highly skilled human labor and precise experimental procedures. 
    Here, ML algorithms, if given sufficient and good-quality data, can expedite discovery by flagging promising candidates for experimental validation.
    However, due to the high false negative rates in experimental methods, high-quality negative examples are scarce; a gene that appears benign under one condition may be disease-causing under another \cite{reviewPUBioformatics2021}.

    Similar issues affect a wide array of applications in Bioinformatics and Medicine -- from biological sequence classification, to drug interaction identification, to identifying undiagnosed diseases from health records \cite{elkan2008pu}.
    In response, starting from the seminal works of \cite{hur2006proteinProtein,elkan2008pu}, a substantial body of work designs and uses heuristics for PU learning; a partial list is as follows \cite{ming2016pupylation,pengyi2019adaSampling,zhao2008geneFunction,xiao2008niologicalSequence,yang2012positive,luigi2010geneRegulatory}.
    For an in-depth discussion on the challenges and promise of PU learning across several Bioinformatics applications, see \citet{reviewPUBioformatics2021}.

    \itparagraph{Reasons for Imperfect Unlabeled Examples: Source Mismatch.}
        The feature distribution of unlabeled samples often differs from that of the positive examples when drawn from different sources \cite{jaskie2022positive}.
        For instance, validated disease-causing genes (positive examples) come from specialized studies, while unlabeled examples may be sampled from all known genes \cite{mordelet2011prodige,reviewPUBioformatics2021}.
        Similarly, in medical applications, positive data might be collected from specialized and costly programs at a single hospital, whereas unlabeled data are aggregated nationwide, leading to discrepancies in feature distribution (\eg{}, demographics and measurement conditions) \cite{Hassanzadeh2018ClinicalDC}. 
        In both illustrative examples, the imperfect unlabeled distribution has a strictly larger support (\eg{}, it is the set of all known genes) than the (perfect) unlabeled distribution (\eg{}, which are supported on only relevant genes), so the upper bound in \cref{asmp:smoothness} is natural.

  \paragraph{Applications in User-Specified Identification Tasks.}
      User-specified identification tasks -- also known as one-vs-rest classification, inlier-based outlier detection, and anomaly detection -- involve a user labeling a \textit{small} set of positive examples from a large pool of unlabeled data.
      Examples include automated face tagging, email classification, and remote sensing\footnote{Remote sensing identifies images with target concepts (\eg{}, roads, rivers) from Earth and planetary imagery.} \cite{plessis15convex,bekker2020learning,wenkai2011remoteSensingPU}.

      \itparagraph{Reasons for Imperfect Unlabeled Examples.}
      Here, reasons for covariate shift include mismatch in the time of collection of positive and unlabeled samples and corruptions in the unlabeled samples.
      In more detail, in some cases, unlabeled data is augmented with data generated by generative AI methods \cite{doi:10.1126/science.adi6000,rubin2004multiple,raghunathan2021syntheticData,jordon2022synthetic}, which
            can introduce complex or even adversarial shifts.
            Nevertheless, tools from the learning-with-corrupted-data literature can help identify a \textit{list} of unlabeled sources that ``approximates'' the true distribution (\cref{rem:list}); that turns out to be sufficient for learning (\cref{sec:intro:application:listDecoding}).
      Again in this setting, since unlabeled data is augmented, we expect the imperfect unlabeled samples to have a larger support than the perfect unlabeled samples---making the bound in \cref{asmp:smoothness} natural.

    \subsection{Takeaways and Open Problems}\label{sec:takeawaysOpenProblems}\label{sec:intro:takeawaysOpenProblems}
        {We initiate the smooth analysis of positive-only learning. While the problem seems important on its own} {due to its connections to} practical settings (\cref{sec:realWorldApplications}), it also has surprising connections to truncated statistics. 
        
        {The smooth perspective on the problem leads to algorithms that} work under significantly weaker assumptions {(just smoothness and $L_1$-approximability)} than those required by existing algorithms in {practical applications and truncated statistics} ({\eg{}, exact access to unlabeled samples or $L_2$-approximability; see } \cref{sec:intro:realWorldApplications,sec:intro:application}). 
        This applicability to seemingly unrelated problems is an important aspect of our results, that we hope will inspire further research.
        To this end, we identify and present some open problems that arise from our work before discussing further related works, formal statements of our results, and proofs. 
        
        Our algorithm is a modification of the \lreg{} algorithm, and proving its correctness required significant new analysis.
        In light of this, our first open problem is:
        \begin{openproblem}
            \textit{Is there a black-box reduction from efficient {smooth positive-only learning} to efficient agnostic learning?}
        \end{openproblem}
        Such a reduction will, in particular, demonstrate that the computational complexity of {smooth positive-only learning} is at most as hard as that of agnostic learning.

        Next, while for a general hypothesis class, exponential dependence on $\nfrac{1}{\eps}$ is necessary for any SQ algorithm (see \cref{rem:SQlowerbound}), it may be possible to design polynomial-time algorithms for specific hypothesis classes and distribution families for $\cD$ and $\cDtrue$. 
        In this direction, we leave the following open problems whose counterparts in the realizable PAC learning model are already solved (by \citet{vempala2010convexGaussian}  for the first problem and by the folklore implementation of ERM via linear programming for the second problem \cite{mohri2018foundations}).
        \begin{openproblem}         
             Is there an algorithm that in the {smooth positive-only learning} model finds an $\wh{h}$ satisfying the guarantee of \cref{thm:main} in $\poly(\nfrac{d}{\eps})$  time for the following classes $\hyH$ and distributions $(\cDtrue,\cD)$?
            \begin{itemize}[itemsep=0pt]
                \item[$\triangleright$] $\hyH$ is the set of intersections of $O(1)$ halfspaces and $\sinparen{\cDtrue,\cD}$ are unknown Gaussians; and 
                \item[$\triangleright$] $\hyH$ is {the set of (single) halfspaces} and $\sinparen{\cDtrue,\cD}$ are arbitrary distributions.
            \end{itemize}
        \end{openproblem} 
        Even in the special case where $\cDtrue$ and $\cD$ are Gaussian and $\hyH$ is the family of halfspaces (a special case of Problem 2.1), this might require a different approach than solving \pERM{}, as we believe solving \pERM{} is \np{}-hard in this case.
        Indeed, in this special case, a $\poly(\nfrac{d}{\eps})$ algorithm is known, but it utilizes a very different moment-based approach \cite{lee2024unknown}.
        Along these lines, since the initial arXiv posting of this work, \cite{kouridakis2026truncation} designed a $\poly(\nfrac{dk}{\eps})$ time implementation of \cref{thm:sampleComplexity} when $\hyH$ is the set of unions of at most $k$ intervals in $\R$.
        
        Next, consider the smooth positive-only learning model where smoothness holds with $q=1$.
        \cref{thm:sampleComplexity} shows that the sample complexity is at most $\wt{O}(\nfrac{\vc{}(\hyH)}{\eps^2})$.
        Since {smooth positive-only learning} is harder than realizable learning, we also know that at least $\wt{\Omega}(\nfrac{\vc{}(\hyH)}{\eps})$ samples are required.
        This raises the question:
        \begin{openproblem}
            What is the optimal sample complexity of {smooth positive-only learning} under smoothness with $q=1$?
        \end{openproblem}
        In this work, we study the standard uniform rates model.
        However, this can be overly pessimistic because it allows the worst-case data distribution to vary with the target accuracy $\varepsilon$, even though in practice the data distribution remains fixed as one increases the sample size to achieve higher accuracy (see \cite{haussler1990survey}).
        Motivated by this observation  \citet*{bousquet2021universal} studied the learning curves in the \textit{universal rates model} for the PAC setting, and revealed a surprisingly different landscape of rates.
        Therefore, the following natural question arises: 
        \begin{openproblem}
            What is the characterization of universal rates for binary classification in the smooth positive-only learning model?    
        \end{openproblem}
 
    \subsection{Related Work} \label{sec:intro:relatedWorks}
        We have already discussed relevant results on learning from positive-only samples.
        {Our} results are {also} related to several learning models. 
        We briefly survey them below and defer additional works to  \cref{sec:appendix:related}. 
        
        \paragraph{Smoothed Analysis.} 
            \citet{spielman2004smooth} initiated the study of algorithms whose inputs are perturbed by Gaussian noise; they, in particular, showed that the simplex algorithm runs in polynomial time in this model. 
            Since then, the  \textit{smoothed analysis} model has found applications in many areas, often enabling significantly stronger theoretical guarantees than the worst-case model.
            For instance, in online learning, a line of works \cite{haghtalab2022efficientSmooth,haghtalab2020smooth,haghtalab2024smoothJACM,block2024performanceempiricalriskminimization,blanchard2024agnostic} showed that if the adversary is restricted to select strategies from smooth distributions, then the lower bounds based on the Littlestone dimension can be bypassed and any VC class becomes learnable.
            In the same spirit \citet*{pmlr-v247-chandrasekaran24a} obtained a wide array of faster PAC learning algorithms when the learner only has to compete with the best classifier that is robust to small random Gaussian perturbations of its inputs.
            We refer the reader to \cite[Chapter 13]{roughgarden2021beyond} and \cite{beier2004typical} for more discussion and applications outside of learning theory.
         
        \paragraph{Learning from Positive and Perfect Unlabeled Samples.} 
        \citet{denis1990positive} initiated the study of learning from positive-only samples when the underlying distribution $\cDtrue$ is known.
        Denis termed this learning from positive and unlabeled samples (or PU learning) and showed that any class that is PAC-learnable is also PAC-learnable from positive and unlabeled samples (in the realizable setting). 
        To show this, they reduced (realizable) PAC-learnable from positive and unlabeled samples to learning in the CPCN model with noise bounded away from $\nfrac{1}{2}$.
        Since then, a growing number of works in Applied Machine Learning design efficient heuristics for learning with positive and unlabeled samples, \eg{}, \cite{elkan2008pu,mordelet2011prodige,plessis15convex,sakai2019covariate, bekker2020learning,reviewPUBioformatics2021}; we refer the reader to the survey of \citet{bekker2020learning} for an overview. 
        The heuristics in this line of work inspire the optimization program (\pERM{}) at the core of our approach; although this program by itself is not sufficient to learn in the smooth model where $\cDtrue$ is not exactly known and instead we need to use an iterative procedure.

        \paragraph{Learning from Positive-Only Samples.}
        We have already discussed several works on learning from positive-only samples earlier in this work \cite{natarajan1987learning,shvaytser1990positiveonly,kivinen1995one,frieze1996linearTransformations,Nguyen2009learningParallelepiped,anderson2013simplices,de2015satisfying,Kontonis2019EfficientTS,canonne2020satisfying,he2023robustMoments,lee2024unknown}.
        Here, we mention two additional works:
        Curiously, preceding \citet{natarajan1987learning}, the seminal work of \citet{valiant1984theory} already introduced a variant of the problem of PAC learning from positive-only examples with two differences: Valiant (1) required the learner to have zero false-positives (\ie{}, have one-sided error) and (2) gave the learner membership access to the target concept $\hstar$.
        Further, since the initial arXiv posting of this work, \citet{mansouri2025learning} continued the study of learning from positive-only examples:
        First, they extend our lower bound in \cref{thm:impossibilityPositiveSamples} by lower bounding the number of unlabeled samples required using a new concept of ``claw numbers.''
        Second, they also study and prove results for three additional extensions of positive-only learning with different forms of distribution shifts.
        More generally, learning with positive samples also appears in modern settings motivated by (1) generative models -- where one typically observes valid text during training but there is no canonical distribution over invalid outputs \citep{kleinberg2024language,kalavasis2025limitslanguagegenerationtradeoffs} -- and (2) in treatment-effect estimation in observational studies, beyond unconfoundedness, which can be cast as a problem of mean estimation from positive-only samples \cite{cai2025makestreatmenteffectsidentifiable}.
        The idea itself dates back to Gold's seminal model of identification in the limit \citep{gold1967language,angluin1979finding}, where the learner is presented with positive examples but no negative examples.

        \paragraph{Reliable Learning.} 
            A line of works studies reliable learning \cite{kalai2012reliable,kalai2021arbitraryCS,kanade2014independentRealiable,goel2017reliableRELU,durgin2019reliableCSPs,goldwasser2020beyondPurturbations,diakonikolas2024reliable}, a model introduced by \citet{kalai2012reliable}.
            This model captures scenarios where one type of error, \eg{}, false positives, is costlier than another type, \eg{}, false negatives. 
            Some works \cite{kalai2012reliable,kanade2014independentRealiable} on fully reliable learning (which requires the learner to make almost no errors but allows it to abstain) design constrained \lreg{} algorithms, that abstain on some examples. 
            While there are several differences between our work and theirs, a key difference is that they have access to both positive and negative samples and, hence, their constraint is defined with respect to the same distribution as their objective. 
            In our setting, however, we work without negative samples and our unlabeled samples have a different distribution than the ``true'' feature distribution; apart from significant differences in the analysis, this change also prohibits proper learning (see \cref{rem:necessityOfImproper}) and requires us to run multiple carefully constructed instances of the constrained version of \lreg{} we develop.

        \paragraph{List-Decodable Learning.}
            There is a growing array of list-decodable learning algorithms that continue to work even when a large fraction of their inputs are corrupted \cite{charikar2017learning,karmalkar2019list,Cherapanamjeri2020DecodableNearlyLinear,bakshi2021Subspace,Abhimanyu2023DecodableBatches,Diakonikolas2023DecodableCovariance,kothari2022decodableCovariance,Raghavendra2020Decoding,Diakonikolas2022DecodableSparseMean,Diakonikolas2021SQ_LB_Decoding,Diakonikolas2021PCA_Decodable}.
            Our results on the list decoding model (\cref{def:list}) enable us to utilize these algorithms to enable learning even in scenarios where a large fraction of the unlabeled samples are corrupted.
            We refer the reader to \cref{sec:appendix:related} for a more extensive discussion of list-decoding algorithms.

    \section{Preliminaries}\label{sec:preliminary}
    
    \textbf{Notation.}\quad 
    Given a point $x\in \R^d$ and a set $S\subseteq\R^d$, we use $\mathds{1}_S(x)$ and $\ind\{x \in S\}$ to denote the indicator that $x\in S$.
    We use standard notation for vector and matrix norms:
    For a vector $z\in \R^d$ and $p\geq 1$, the $L_p$-norm of $z$ is $\norm{z}_p\coloneqq \inparen{\sum_i \abs{z_i}^p}^{\sfrac{1}{p}}$ and, taking limit $p\to\infty$ implies, $\norm{z}_\infty=\max_i \abs{z_i}$.
    Given a distribution $\cD$ over $\R^d$ and a measurable set $S\subseteq\R^d$, $\cD(S)$ denotes the mass of $S$ under $\cD$, \ie{}, $\Pr_{x\sim \cD}[x\in S]$.
    We overload this notation at times in applications in truncated statistics (in \cref{sec:intro:application,sec:application:unknownTruncation,sec:application:detectingTruncation}) to also mean the truncation of a distribution $\cD$ to the measurable set $S$, \ie{}, if $\d\cD$ is the density of $\cD$, $\cD(S)$ is the conditional distribution $\propto \ind_S(x) \cdot \d\cD(x)$. 
    The difference between the two would always be clear from context.

    \paragraph{Hypothesis Classes and PTFs.} 
    We use both Boolean functions (\eg{}, $h\colon \R^d \rightarrow \inbrace{0,1}$) and sets $(H \subseteq \R^d)$ to denote hypotheses, with the relation $H = \inbrace{x \in \R^d \colon h(x) = 1}$. 
    Our algorithm, like most efficient learning algorithms, to learn general hypothesis classes in non-realizable settings, relies on polynomial threshold functions (PTFs).
    Hence, PTFs frequently show up in our analysis.

    \begin{definition}[Polynomial Threshold Functions (PTF)]\label{def:ptfs:polynomials}
        Given $k\geq 1$, let $\hyP(k)$ be the collection of all polynomial threshold functions of degree-$k$.
        A polynomial threshold function (PTF) of degree $k$ is a function defined by a polynomial $p\colon \R^d \rightarrow \R$ of maximum degree $k$ and a threshold $t \in \R$ in the following sense:
        $H({p,t})\coloneqq \mathds{1}\sinbrace{p(x)\geq t}$. 
    \end{definition}

    \paragraph{{Sample Distributions}.}
        Given a hypothesis class $\hyH$ and a target function $\hstar$, we use $\optset$ -- which is the set corresponding to $\hstar$ -- to denote the set of positive examples.
        We use $\cPtrue$ to denote the distribution of positive covariates or features, which satisfies $\supp(\cPtrue)=\optset$.
        $\cDtrue$ denotes the distributions of all features (regardless of whether their label is positive or not).
        $\cPtrue$ and $\cDtrue$ are related: $\cPtrue$ is the truncation of $\cDtrue$ to $\optset$ -- the set of positive examples.

    \paragraph{Divergences Between Distributions.}
        In our applications, we often need bounds on the ``closeness'' of two distributions.
        For this, we recall some common notions. 
        Consider two distributions $\cP$ and $\cQ$.
        The total variation distance between $\cP$ and $\cQ$ is $\tv{\cP}{\cQ}=\inparen{\nfrac{1}{2}}\int_{x}\abs{\d \cP(x)-\d \cQ(x)}\d x$.
        The $\chi^2$ divergence from $\cP$ to $\cQ$ is  $\chidiv{\cP}{\cQ}
        = \int \inparen{\nfrac{\d\cP}{\d\cQ}}^2 \d \cQ - 1$.
        Similarly, the Rényi divergence of order $r \geq 1$ from $\cP$ to $\cQ$ is defined as $\renyi{r}{\cP}{\cQ} = \frac{1}{r - 1} \log \int (\sfrac{\d \cP}{\d \cQ})^r \d \cQ$.

\section{Learning Model and Main Results}\label{sec:model_and_result}
    In this section, we formally define the model of {smoothed analysis of learning from positive-only samples} and state our main results. 

    \subsection{{Smoothed Analysis Model for Positive-Only Learning}}\label{sec:model}
        To begin with, we have a domain $\cX$ of covariates or features {(often a subset of $\R^d$)}, either discrete or continuous, and a distribution $\cDtrue$ over the domain.
        We are interested in binary classification and operate in the realizable setting, where for each feature $x$ there is a corresponding label $y(x)\in \zo$, that is realized by an unknown Boolean function $h^\star\colon \cX\to \zo$.
        $h^\star$ can be arbitrary, and the only assumption we make is that there is a known hypothesis class $\hyH$ (which is a collection of Boolean functions $h\colon \cX\to \zo$) that contains $h^\star$.
        Of specific interest is the set of samples with a positive label and, for convenience, we use $\optset$ to denote it, \ie{}, 
        \changetag{Support of Positive Examples) \ (#1} 
        \[
            \phantom{ABCDEFGHIJK}
            \optset\coloneqq \inbrace{x\in \supp(\cD)\colon h^\star(x)=1}\,.
            \yesnum\label{eq:setOfPositiveExamples}
        \]
        \changetag{#1}
 
        \noindent We first introduce PAC learning with positive-only samples (without smoothness).  
        \begin{problem}[PAC learning from Positive-Only Samples]
        \label{prob:positiveOnly}
         An instance of PAC learning from positive-only samples is specified by a hypothesis class  $\hyH\subseteq\zo^{\cX}$, a target concept $\hstar\in \hyH$, and an unknown distribution $\cDtrue$ over the instance space $\cX$.
         {The learner has access to a stream of positive samples drawn \iid{} from $\cPtrue$ (the truncation of $\cDtrue$ to $H^\star$).}
        The goal is to output a hypothesis $\wh{h}: \cX \to \zo$ such that with probability at least $1-\delta$ over the draw of samples and any randomness in the learning algorithm:
        \[
            \Pr\nolimits_{X \sim \cDtrue}\sinparen{\wh{h}(X) \neq h^\star(X)} \leq \eps\,.
        \]
        A hypothesis class $\hyH$ is PAC learnable from positive-only samples if there exists an algorithm and sample complexity function $m(\eps, \delta)$ such that for every distribution $\cDtrue$ and target $h^\star \in \hyH$, the algorithm achieves the above guarantee when given $m(\eps, \delta)$ positive samples.
         \end{problem}
         \noindent It is useful to contrast \cref{prob:positiveOnly} with the standard PAC model. 
         In \cref{prob:positiveOnly}, the learner only observes positive examples, yet the learner's performance is still evaluated under the ``complete'' distribution $\cDtrue$.  
         Since the learner never sees any draws from the negative region, the false positive component of the error cannot be directly estimated from observed samples. 
         Consequently, ERM is no longer a meaningful principle: the hypothesis that labels every point as positive is always consistent with the observed data and is therefore an empirical risk minimizer, despite having potentially high error under $\cDtrue$. 
         Hence, learners must take a fundamentally different approach.
       
        \paragraph{Smooth Analysis Model.}
            As we mentioned in \cref{sec:intro}, without any assumptions on the underlying distribution $\cDtrue$ it is impossible to learn any interesting concept class from positive-only samples. 
            In this work, we study the smooth analysis model where in one is given access to a sample access to distribution $\cD$ such that $\cDtrue$ is smooth with respect to $\cD$ in the following sense.
            \smoothness*
            \noindent 
            It is useful because it allows us to upper bound the amount of mass that the true distribution $\cDtrue$ assigns to any {(measurable)} set $S$ based on the mass $\cD$ assigns to $S$.
            This generalizes the notions of smoothness considered by recent works on smoothed online learning \cite{haghtalab2020smooth,haghtalab2021smooth,haghtalab2022efficientSmooth,haghtalab2024smoothJACM,blanchard2024agnostic}; concretely, their notion of smoothness is equivalent to the special case $q=1$.

            \begin{remark}[Extensions to More General Notions of Smoothness]
                \label{rem:generalSmoothness}
                The techniques we use to bound the sample complexity and computational complexity of {smooth positive-only learning} extend beyond \cref{asmp:smoothness}.
                They suffice for {smooth positive-only learning} under the following weakening of \cref{asmp:smoothness}:
                there is a known rate function $r\colon \R_{\geq 0}\to \R_{\geq 0}$ satisfying $\lim_{m\to 0^+} r(m)=0$ such that the pair of distributions $\inparen{\cD, \cDtrue}$ satisfy
                \[
                    \text{for any measurable set $S$}\,,\quad 
                    r(\cDtrue(S))\leq \cD(S)\,.
                    \yesnum\label{eq:mostGeneral:smoothness}
                \]
                This captures \cref{asmp:smoothness} when $r(m)=(\sigma m)^{q}$.
                This additional generality also allows our framework to extend to settings where we can only ensure $\kl{\cDtrue}{\cD}=O(1)$ (see \cref{sec:klBoundsSatisfyGeneralSmoothness} for details).
                This is important as \cref{asmp:smoothness} does not hold in this setting for any finite $\sigma$ or $q$.
                Regarding results under \eqref{eq:mostGeneral:smoothness}, using our techniques one can deduce that the sample complexity of the {smooth positive-only learning} is $O\binparen{\poly{(r(\eps))}\cdot \inparen{\vc{}(\hyH)+\log{\nfrac{1}{\delta}}}}$ and the computational complexity is $\poly(d^k,\nfrac{1}{\zeta},\log{\nfrac{1}{\delta}})$ for $\zeta\leq r(\poly(\alpha,\eps))$.
            \end{remark}

        \paragraph{Main Problem.}
        Now, we are ready to state the main task we study in this work.
        
        \begin{problem}[Smooth PAC Learning with Positive-Only Samples] 
        \label{prob:PUlearning}
            An instance of {smooth PAC learning with positive-only samples} is specified by a hypothesis $h^\star\in \hyH$, that induces a set of positive samples $\optset$ (\cref{eq:setOfPositiveExamples}), a distribution $\cDtrue$ of features, and 
            a distribution $\cD$ {such that the pair $(\cDtrue,\cD)$ satisfies} \cref{asmp:smoothness}.
            One has access to two streams of samples:
            \begin{enumerate}
                \item \emph{({Positive} Stream)}\quad Samples $x_1,x_2,\dots \sim \cPtrue$ where $\cPtrue$ is the truncation of $\cDtrue$ to $\optset;$ and 
                \item \emph{({Imperfect Unlabeled} Stream)}\quad Samples $x_1,x_2,\dots \sim \cD$. 
            \end{enumerate}
            Given $\eps,\delta\in (0,1)$, the goal is to output a hypothesis $h\colon \cX\to \zo$ such that, with high probability, $h$ is close to $h^\star$ in the usual sense of PAC-learning: with probability $1-\delta$, $\Pr_{\cDtrue}\inparen{h^\star(x)\neq h(x)} \leq \eps.$
        \end{problem}
        \noindent Our computationally efficient algorithm will also require the following mild assumption.
        \posFraction*
        \noindent This assumption requires that the fraction of positive examples have at least a constant mass under $\cDtrue$.
        This assumption, by now, is standard in the works on efficient learning with positive-only samples (\eg{}, \cite{lee2024unknown,Kontonis2019EfficientTS}). %
        That said, this assumption can be dropped if one can estimate $\Pr_{\cDtrue}\inparen{h^\star(x)\neq 1}$; as if we can estimate $\Pr_{\cDtrue}\inparen{h^\star(x)\neq 1}$ and, in particular, check if it is at most $\eps$, then we can simply output the all-zero hypothesis and only incur an $\eps$ error.
            
        \paragraph{Polynomial Approximation.}
            The other ingredient required by our efficient algorithm is a bound on the complexity of $\hyH$.
            The core notion of complexity of concept classes we consider is their approximability by polynomials, which was introduced by \citet*{kalai2008agnostically}.
            This notion of complexity is a very widely studied condition that is the backbone of most of the efficient learning algorithms that deal with agnostic noise \cite{kalai2008agnostically,kane2013learning,Blais:2010aa,10.1145/1806689.1806763,10.1007/978-3-642-15369-3_44,10.1145/2395116.2395118,5670941}. 
            Formally, it is defined as follows.
            \begin{definition}[Polynomial Approximability]\label{def:poly_approx}
                Fix a concept $h\colon \R^d \to \inbrace{0,1}$, distribution $\cD$ over $\R^d$, numbers $k,p\geq 1$, and constant $\eps>0$.
                $h(\cdot)$ is said to be $\eps$-approximable by degree-$k$ polynomials in $L_p$-norm with respect to $\cD$ if there exists a degree-$k$ polynomial $f\colon \R^d\to\R$ such that 
                \[
                    \inparen{\Ex_{x\sim \cD}{\abs{f(x)-h(x)}^p}}^{\sfrac{1}{p}} < \eps\,.  
                \]
                (For sets, we simply replace $h(x)$ with the indicator of the set, \ie{}, $\ind_H(x)$.)
                Moreover, a hypothesis class $\hyH$ is said to be $\eps$-approximable by degree-$k$ polynomials in $L_p$-norm with respect to a family of distributions $\hyD$ if, for each $h\in \hyH$ and $\cD\in \hyD$, $h$ is  $\eps$-approximable by degree-$k$ polynomials in $L_p$-norm with respect to $\cD$. 
            \end{definition}
            \begin{table}[ht!]
            \centering
            \small 
            \begin{tabular}{c c c}
             Hypothesis Class $\hyH$
                & Distribution Family & 
                    Upper bounds on $k$ \\
            \midrule
            Degree-$t$ Polynomial Threshold Functions  \cite{kane2011gsa} &
                Gaussian & 
                        $O(t^2/{\eps^4})$ \\
            Intersections of $t$ Halfspaces \cite{kalai2008agnostically} & 
                    Gaussian  & 
                        $O({\log t}/{\eps^4})$\\
            General Convex Sets \cite{ball1993gsaConvexSets} &
                 Gaussian & 
                    $O(\sqrt{d}/{\eps^4})$ \\
            \hline
            Arbitrary Functions of $t$ Halfspaces \cite{ksv2026sandwiching}
                &
                    \begin{tabular}{@{}c@{}}
                        Log-Concave and\\
                        $1$-Strictly Sub-Exponential\end{tabular}
                    & {$\wt{O}(t^5/\eps^4)$}\\
            \hline
            Decision trees of size $s$
            \cite{klivans2023testable} 
                & Uniform on $\zo^d$ 
                    & ${O}(\log(s/\eps))$ \\
            Intersections of $t$ Halfspaces 
            \cite{gopalan_2010_fooling}
                & Uniform on $\zo^d$ 
                    & $\tilde{O}(t^4/\eps^2)$ \\
            Decision tree of halfspaces of size $s$, depth $t$
            \cite{gopalan_2010_fooling}
                & Uniform on $\zo^d$ 
                    & $\tilde{O}(s^2 t^4/\eps^2)$ \\
            \end{tabular}
            
                \caption{
                    \textit{
                        Upper bounds on the degree of polynomials required for $\eps$-approximating $\hyH$ in $L_1$-norm with range $[0,\infty)$.
                        For the Gaussian case, the bounds are implied by combining results from \cite{kane2011gsa,klivans2008gaussian,ball1993gsaConvexSets} with \cite{2026nonNegativeApproximation}.
                    }
                }
                \label{tab:gsa:GaussianSurfaceArea}
        \end{table} 
        
            \noindent Our algorithms will always work whenever the target $h^\star$ (or equivalently, the set $\optset$ defined by $h^\star$) is approximable by polynomials with respect to  $L_1$-norm.
            It is worthwhile to note that if a hypothesis is approximable in $L_2$-norm, then it is also approximable in the $L_1$-norm because for any random variable $Z$, $\Ex[\abs{Z}]\leq \sqrt{\Ex[Z^2]}$.
            Hence, $L_1$-approximability is a weaker requirement than $L_2$-approximability.
            In fact, it can be significantly weaker: there are hypothesis classes and distributions for which $L_1$-approximation guarantees are known but \textit{no} $L_2$-approximation guarantees are known; we refer the reader to the work of \citet{kane2013learning} for an extensive discussion about this.
 
    \subsection{Main Results}\label{sec:mainresult}
        In this section, we present our main results for {smooth positive-only learning} (\cref{prob:PUlearning}). %
        Our first result, \cref{thm:sampleComplexity}, bounds the sample complexity and we restate it below.  
        \sampleComplexity*
        \noindent As mentioned in \cref{sec:intro}, this result shows that VC dimension also characterizes {smooth positive-only learning, which is a stark contrast to characterizations for positive-only learning \textit{without} smoothness}.
        The proof of \cref{thm:sampleComplexity} appears in \cref{sec:perm}.

        \begin{remark}[Proper vs.\ Improper Learning]
            The learner in \cref{thm:sampleComplexity} is improper: it outputs a hypothesis that is an intersection of $(\nfrac{1}{\eps})$-hypotheses in our original class. 
            This is necessary as simple examples show that proper learning is impossible {with positive-only samples even with smoothness} (\cref{rem:necessityOfImproper}).
        \end{remark}

        \noindent Our next result bounds the {smooth} computational complexity of {positive-only learning}.
        \begin{theorem}[Computationally Efficient Algorithm]\label{thm:main}
            Let \cref{asmp:smoothness,asmp:posFraction} hold with constant $\alpha$ and fix a constant $C>0$.
            Fix any $\eps\in (0,\nfrac{\alpha}{6})$ and $\delta\in (0,\nfrac{1}{2})$ and $k$ such that degree-$k$ polynomials with range $[-C,\infty)$ $\zeta$-approximate $\hyH$ with respect to $\cD$ in $L_1$-norm for $\zeta\leq \inparen{C+1}^{-4q}
O(\sigma\eps\alpha)^{8q^2+5q}$. 
            Then, there is an algorithm (\cref{alg:main}) that, given 
                any $\eps,\delta\in (0,1)$, 
                the constants $\sigma,q$ from \cref{asmp:smoothness},
                the constant $\alpha$ from \cref{asmp:posFraction}, $C$, and
                $n=\poly(d^k,\nfrac{1}{\zeta},\log{\nfrac{1}{\delta}})$ samples from $\cPtrue$ and $\cD$,  
            outputs a hypothesis ${h}\colon \cX\to\zo$ such that, with probability $1-\delta$, 
            \[
                \Pr\nolimits_{\cDtrue}\inparen{{h}(x)\neq h^\star(x)}\leq \eps\,.
            \]
            The algorithm runs in time $\poly(n)$.
        \end{theorem}
        As mentioned in \cref{sec:intro}, \cref{thm:main} has several applications to learning from positive examples and truncated statistics, which we outline in \cref{sec:intro:application}. 
        The proof of \cref{thm:main} appears in \cref{sec:computationalEfficiency}.  
        
        To get a sense of the running time in \cref{thm:main}, see \cref{tab:gsa:GaussianSurfaceArea}.
        It is worth noting that the degree of polynomials $k$ required scales with $\nfrac{1}{\eps}$ and, typically, behaves as $k=\poly(\nfrac{1}{\eps})$.
        Therefore, the running time in \cref{thm:main} typically has an exponential dependence on $\nfrac{1}{\eps}$.
        This exponential dependence in the running time is necessary for any SQ-learning algorithm -- such as the one in \cref{thm:main} -- due to recent SQ lower bounds \cite{diakonikolas2024statistical} (\cref{rem:SQlowerbound}).
        \begin{remark}[SQ Lower Bound]\label{rem:SQlowerbound}
            \citet{diakonikolas2024statistical} provide an SQ-lower bound for the -- seemingly -- unrelated task of Gaussian parameter estimation with samples truncated to an unknown set.
            However, if we can avoid the exponential dependence on $\nfrac{1}{\eps}$ in \cref{thm:main} then, due to the recent results \cite{lee2024unknown}, we will also obtain a $\poly(\nfrac{d}{\eps})$ time algorithm for the problem studied by \cite{diakonikolas2024statistical} thereby contradicting their lower bound.
            See \cref{sec:application:unknownTruncation} for further discussion of parameter estimation from truncated samples and how our results relate to it.
        \end{remark}
        \cref{thm:main} requires $L_1$-approximability of the target $h^\star$ with respect to $\cD$ (which we have sample access to) and not with respect to $\cDtrue$ (which we do not have sample access to).
        In some cases, this enables one to use the sample access to $\cD$ to check that the $L_1$-approximation requirement indeed holds (see \cref{rem:testing}).
    
        \begin{remark}[Testing Approximability]
            \label{rem:testing}
            \cref{thm:main} requires the hypothesis class $\hyH$ to be approximable by polynomials with respect to $\cD$ in $L_1$-norm.
            Many guarantees on polynomial approximability are known with respect to specific distribution families such as Gaussians.
            One way to test this guarantee is to test whether the provided samples indeed come from a Gaussian distribution.
            However, most existing distributional testers have prohibitively large running time -- exponential in $d$ \cite{rubinfeld2023testing}.
            Instead, if $\hyH$ satisfies a slightly stronger property that it has sandwiching polynomials with respect to the underlying distribution, then one can use Theorem 3.2 of \citet{gollakota2023momentMatching} that gives a tester that runs in $\poly(d^{k},\nfrac{1}{\eps})$ time. %
            In particular, the tester comes with the following guarantees.
            \begin{enumerate}[itemsep=0pt]
                \item {(Completeness)}\quad It says \textsf{Yes} if samples are drawn from distribution $\cD$;
                \item {(Soundness)}\quad If this tester says \textsf{Yes} on samples $S$, then with high probability, $\eps$-$L_1$-approximating polynomials exist with respect to the empirical distribution $\wh{\cD}$ over the samples $S$.\footnote{
                    To be precise, their result actually shows the existence of $L_1$-sandwiching polynomials, which is sufficient for us since $L_1$-sandwiching (whose definition is not relevant here) is a stronger requirement than $L_1$-approximation.
                }
            \end{enumerate}
            Now, we can run the algorithm in \cref{thm:main} with respect to the empirical distribution $\wh{\cD}$.
            To use \cref{thm:main}, we need to verify that $\cDtrue$ is smooth with respect to $\cDtrue$ in the sense of \cref{asmp:smoothness}.
            By standard VC theory for sufficiently large samples, uniform convergence holds with respect to degree-$k$ polynomials, yielding the following additive-relaxation version of \cref{asmp:smoothness}:
            for any measurable set $T$, $\cDtrue(T)\leq {(\nfrac{1}{\sigma})} \, \inparen{\smash{\wh{\cD}}(T)+\eps}^{1/q}$.
            This weaker guarantee is sufficient to execute our proof as we only use assumption \cref{asmp:smoothness} \mbox{for sets $T$ whose mass is bounded away from $0$}.
        \end{remark}

        \noindent A final remark is that our results (on both sample and computational complexity) can be adapted to work with more general notions of smoothness (see \cref{rem:generalSmoothness}).

        \paragraph{Optimal Sample Complexity for $q=1$.}
        Next, we improve the sample complexity in the important $q=1$ case if the following additional assumption holds. %
        \begin{restatable}[Density Lower Bound]{assumption}{lowerBound}\label{asmp:densityLB}
            There are known constants $\sigma\in (0,1]$, such that, for any measurable set $S$, $\cDtrue(S)\geq 
                \sigma
                \cdot \cD(S).$
        \end{restatable}
        \vspace{-5mm}
        \begin{remark}\label{rem:densityLB}
            While \cref{asmp:densityLB} is similar to \cref{asmp:smoothness}, it is substantially stronger: 
            \eg{}, when the (imperfect) unlabeled samples are gathered from an unreliable source, they are bound to contain some elements not in support of the true unlabeled distribution $\cDtrue$, and such elements are enough to violate \cref{asmp:densityLB}.
        \end{remark} 
        With \cref{asmp:densityLB}, we obtain the following sample complexity. 
        \begin{restatable}[Optimal Sample Complexity with \cref{asmp:densityLB}]{theorem}{optimalSampleComplexity}
            \label{thm:sampleComplexity:bothSided}
            Suppose \cref{asmp:smoothness,asmp:densityLB} hold with $q=1$. 
            Fix any $\eps,\delta\in (0,\nfrac{1}{2})$.
            There is an algorithm that, given $\eps,\delta$ and $n=\wt{O}\binparen{\!\inparen{\nfrac{1}{\eps \sigma^2}}\cdot \inparen{\vc{}(\hyH)+\log{\nfrac{1}{\delta}}}}$ independent samples from $\cPtrue$ and $\cD$, outputs a hypothesis $\wh{h}$ such that, with probability $1-\delta$,
            \[ 
                \Pr\nolimits_{\cDtrue}\binparen{\wh{h}(x)\neq h^\star(x)}\leq \eps\,.
            \]
        \end{restatable}
        Note that this sample complexity is optimal for {smooth positive-only learning} (up to logarithmic factors), since it is optimal for realizable PAC learning which is an easier problem.\footnote{In realizable PAC learning, one has access to {labeled} samples from $\cDtrue$. 
        The problem can be turned into a smooth positive-only learning problem by ignoring labels of the samples, which is sufficient to ensure smoothness with $q = \sigma = 1$.}
        The proof of \cref{thm:sampleComplexity:bothSided} appears in \cref{sec:proofof:thm:sampleComplexity:bothSided}.
        \begin{remark}
            As mentioned earlier, under \cref{asmp:densityLB}, prior works \cite{denis1990positive,lee2024unknown} can also be used for positive-only learning, but \cref{thm:sampleComplexity:bothSided} improves their sample complexity: \cite{lee2024unknown} requires $O(\eps^{-4})$ samples versus our $O(\eps^{-1})$; \cite{denis1990positive} only handles $q=1$ where it matches our sample complexity of $O(\eps^{-1})$ while we can also handle $q>1$ (\cref{thm:sampleComplexity}). Without this assumption (\cref{asmp:densityLB}), these works are inapplicable while our main results (\cref{thm:sampleComplexity,thm:main}) remain valid.
        \end{remark}

\addtocontents{toc}{\protect\setcounter{tocdepth}{2}}

\section{Proofs of Main Results (\cref{thm:main,thm:sampleComplexity})}\label{sec:mainProof} 
    In this section, we prove our main results on the {smooth} sample complexity (\cref{thm:sampleComplexity}) and {smooth} computational complexity (\cref{thm:main}) {of positive-only learning}.

    The starting point of our approach is practical heuristics used for learning from positive and unlabeled samples {which is a significantly simpler problem where $\cD=\cDtrue$ (see \cref{sec:intro:relatedWorks}).}
    At a high level, these heuristics find a hypothesis $H$ containing the minimum number of unlabeled samples subject to selecting all positive samples \cite{bekker2020learning}.
    We formalize this idea as the following constrained empirical risk minimization (ERM) problem, which we call \textit{Pessimistic ERM:}
    given tolerance $\rho\geq 0$, sets of positive and unlabeled samples $P$ and $U$ respectively, and a hypothesis class $\hyH$, the Pessimistic-ERM problem is  
    \[
        \argmin_{H\in \hyH}~~ \frac{\abs{H\cap U}}{\abs{U}}
        \,,\quad\text{ such that}\,,\quad 
        \frac{\abs{H\cap P}}{\abs{P}} \geq 1-\rho
        \,.
        \tag{\pERM{}}
    \]

    \paragraph{Outline of this section.}
        First, in \cref{sec:perm}, we prove \cref{thm:sampleComplexity} using an algorithm that solves many instances of \pERM{} and aggregates the solutions.
        Second, in \cref{sec:efficientPERM}, we reduce computationally efficient {smooth positive-only learning} to developing a computationally efficient approximation algorithm for \pERM{}.
        This proves \cref{thm:main} up to designing an efficient approximation algorithm for \pERM{}, which we do in \cref{thm:constReg}.

    \subsection{Proof of \cref{thm:sampleComplexity} (Sample-Efficient Algorithm)}
        \label{sec:perm}
        \label{sec:sampleEfficiency}
    In this section, we prove \cref{thm:sampleComplexity} which bounds the sample complexity of {smooth positive-only learning}.
    Our algorithm is an extension of the \pERM{} and solves $\nfrac{1}{\eps}$ carefully-constructed instances of \pERM{}; we believe this approach could be useful in practical settings where \pERM{} is used (\cref{sec:realWorldApplications}).

    Our proof relies on the following standard second-order uniform convergence bound, which gives sharper uniform convergence rates for hypotheses with small error.
    \begin{theorem}[Second-Order Uniform Convergence \protect{\cite[Section 5.1.2]{Boucheron_Bousquet_Lugosi_2005}}]
        \label{thm:unifConvergence}
        Let $\hyF$ be any concept class on $\R^d$ with finite VC-dimension, \ie{}, $\vc{(\hyF)}<\infty$, and let $\cD$ be any distribution over $\R^d$.
        For any confidence level $\delta\in (0,\nfrac{1}{2})$ and any fixed set $F^\star\subseteq \R^d$, with probability $1-\delta$, over the draw $X\sim \cD^n$, every $F\in \hyF$,
        \begin{align*}
            \cD\sinparen{F\triangle F^\star}
                &\leq 
                    D_n\sinparen{F\triangle F^\star}
                    + \sqrt{
                        D_n\sinparen{F\triangle F^\star}
                        \cdot 
                        s(\hyF,n,\delta)
                    }
                    + 
                    4 s(\hyF,n,\delta)\,,
                \yesnum\label{eq:sampleComplexity:1}\\
                D_n\sinparen{F\triangle F^\star}
                &\leq 
                    \cD\sinparen{F\triangle F^\star}
                    + \sqrt{
                        \cD\sinparen{F\triangle F^\star}
                        \cdot 
                        s(\hyF,n,\delta)
                    }
                    + 
                    4 s(\hyF,n,\delta)\,,
                \yesnum\label{eq:sampleComplexity:2}
        \end{align*}
        where $s(\hyF,n,\delta)$ and $D_n\sinparen{F\triangle F^\star}$ are defined as follows 
        \[
            s(\hyF,n,\delta) \coloneqq \frac{2\vc{(\hyF)}\log{(n+1)}+\log{\nfrac{4}{\delta}}}{n}
            \qquadand
            D_n\sinparen{F\triangle F^\star}
            \coloneqq 
            \frac{\abs{X\cap \sinparen{F\triangle F^\star}}}{\abs{X}}\,.
        \]
    \end{theorem}
    \begin{remark}[Necessity of Improper Learners for Positive-Only Learning with Smoothness]
        \label{rem:necessityOfImproper}
    To illustrate that improper learners are necessary, consider the class of one-dimensional halfspaces (\ie{}, thresholds with a ``direction''). Define
        \[
        h_{\leq 1} = \{ x \in \R : x \leq 1 \} \qquadand h_{\geq -1} = \{ x \in \R : x \geq -1 \}\,.
        \]
        Let the two ``true'' unlabeled distributions be 
        \[
            \cDtrue_1 = \unif[-2,1] \qquadand \cDtrue_{-1} = \unif[-1,2]\,.
        \]
        We now consider two smooth positive-only learning instances defined by (a): unlabeled distribution $\cDtrue_1$ with optimal hypothesis $h_{\geq -1}$ and (b): unlabeled distribution $\cDtrue_{-1}$ with optimal hypothesis $h_{\leq 1}$.
        Note that both instances yield the same positive sample distribution, $\cPtrue = \unif[-1,1]$.
        Let the imperfect unlabeled distribution for both instances be $\cD = \unif[-2,2]$; which satisfies \cref{asmp:smoothness} with parameters $(\sigma,q)=(\nfrac{3}{4},1)$ in both instances.
        Since any deterministic learner cannot distinguish between these two instances, if it is proper (\ie{}, its output is restricted to one-dimensional halfspaces), it must incur a constant error on at least one of the instances. The same holds with probability $\nfrac{1}{2}$ for randomized learners.
    \end{remark}
    
    \subsubsection{Proof of \cref{thm:sampleComplexity}}
        
        \begin{proof}[Proof of \cref{thm:sampleComplexity}]
        The algorithm is presented in \cref{alg:sampleComplexity}.
        It solves $T\coloneqq (\sigma\eps)^{-q}$ carefully selected instances of \pERM{} to obtain hypotheses $H_1,H_2,\dots,H_T$ and outputs $\bigcap_{i=1}^{T} H_i$.
        The instances of \pERM{} are constructed iteratively, where $i$-th problem removes all unlabeled samples not in $H_1 \cap H_2 \cap\dots \cap H_{i-1}.$
        \begin{algorithm}[htb!]
        \caption{\textsf{Iterative-\pERM{}} (from \cref{thm:sampleComplexity})}
        \begin{algorithmic}[1]
        \Procedure{\rm \textsf{Iterative-\pERM{}}}{$\eps,\delta,\sigma,q,\hyH$}
            \vspace{2mm}
            \State Set $n=\wt{O}\binparen{\!\inparen{\nfrac{1}{\eps\sigma}}^{2q}\cdot \inparen{\vc{}(\hyH)+\log{\nfrac{1}{\delta}}}}$
            \State Obtain $n$ positive samples $P$ and $n$ (imperfect) unlabeled samples $U$
            \vspace{2mm}
            \State Set $T=(\sigma\eps)^{-q}$
            \vspace{2mm}
            \For{$1\leq i\leq T$}
                \State Define the survival set $S_i=H_1\cap H_2\cap\dots\cap H_{i-1}$ if $i>1$ and, otherwise, $S_i=\R^d$
                \State Compute the solution $H_i$ of \pERM{} specified by tolerance $\rho=0$, positive
                \item[] \phantom{....}~~~~~~~ samples $P$, unlabeled samples $U\cap S_i$, and hypothesis class $\hyH$
            \EndFor{}
            \vspace{2mm}
            \State \textbf{return} $H\coloneqq H_1\cap H_2\cap\dots\cap H_T.$
        \EndProcedure
        \end{algorithmic}
        \label{alg:sampleComplexity}
        \end{algorithm} 
        \noindent We divide the proof of correctness into two parts.
        
        \paragraph{Step A (Bounding the number of false negatives):}
            Each of $H_1,H_2,\dots,H_T$ is feasible for the \pERM{} constraint with tolerance $\rho=0$ and positive samples $P$ with $\abs{P}=n$.
            We claim that this implies that with probability $1-\nfrac{\delta}{2}$, for each $1\leq i\leq T $, $\cDtrue\!\inparen{\optset\setminus H_i}\leq O\!\inparen{(\sigma\eps)^{2q}}$.
            We prove this claim below and for now proceed with the proof.
            Since $H\coloneqq H_1\cap H_2\cap\dots\cap H_T$, an immediate consequence of the above is that with probability $1-\nfrac{\delta}{2}$, 
            \[
                \cDtrue\!\inparen{\optset\setminus H}
                \leq \sum_i \cDtrue\!\inparen{\optset\setminus H_i}
                \leq O\!\inparen{\frac{(\sigma\eps)^{2q}}{(\sigma\eps)^{q}}}
                = O((\sigma\eps)^{q})
                \leq O(\eps)\,.
            \]

            \noindent \textit{Proof of Claim.}~~
            Next, we prove our earlier claim.
            Let $S\in\hyH$ be any solution of \pERM{} with $\rho=0$.
            By feasibility, it satisfies $\nfrac{\abs{S \cap P}}{\abs{P}}=1$ and, hence, $\nfrac{\abs{ P\setminus S}}{\abs{P}}=0$.
            Due to the choice of the sample size $n$ and \cref{thm:unifConvergence}, with probability $1-\inparen{\nfrac{\delta}{2}}$, for each $R\in \hyH\cup \inbrace{G^c\colon G\in \hyH}$,\footnote{Here, we use the facts that $\vc{(\inbrace{G^c\colon G\in \hyH})}=\vc{(\hyH)}$ and $\vc{(\hyH_1\cup\hyH_2)}\leq {\vc{(\hyH_1)}+\vc{(\hyH_2)}+1}$ \cite{mohri2018foundations}.}
            \[
                \cPtrue(R) 
                \leq 
                \frac{\abs{R\cap P}}{\abs{P}}
                +
                \sqrt{\frac{\abs{R\cap P}}{\abs{P}}\cdot (\sigma\eps)^{2q}}
                + (\sigma\eps)^{2q}\,.
            \]
            In particular, applying this to $R=S^c$ (note that $S^c\in\{G^c: G\in \hyH\}$), and using $\nfrac{\abs{ P\cap S^c}}{\abs{P}}=\nfrac{\abs{ P\setminus S}}{\abs{P}}=0$, implies that 
            \[
                \cPtrue(S^c) \leq (\sigma\eps)^{2q}\,.
            \]
            Finally, since $\cDtrue( R )\leq \cPtrue( R )$ for any $ R \subseteq \optset$, it follows that 
            \[
                \cDtrue(\optset\setminus  S)
                \leq \cPtrue(\optset\setminus  S )
                = \cPtrue(S^c)
                \leq  (\sigma\eps)^{2q}\,.
            \]  
        
        \paragraph{Step B (Bounding the number of false positives):}
            Next, we bound the amount of false positives.
            We claim that 
            \[
                \abs{U\setminus H_1}
                \geq 
                \abs{(U\cap H_1)\setminus H_2}
                \geq 
                \abs{(U\cap H_1\cap H_2)\setminus H_3}
                \geq 
                \dots 
                \geq 
                \abs{
                    (U\cap H_1\cap \dots \cap H_{T-1})
                    \setminus H_{T}
                }\,.
                \yesnum\label{eq:oneSided:monotone}
            \]
            To prove this, consider any $1\leq i\leq T -1$ and toward a contradiction suppose that 
            \[
                \abs{(U\cap H_1\cap \dots \cap H_{i-1})\setminus H_i}
                <
                \abs{(U\cap H_1\cap \dots \cap H_{i})\setminus H_{i+1}}\,.
                \yesnum\label{eq:onesided:algo:contradiction}
            \]
            Hence, in particular,
            \begin{align*}
                \abs{(U\cap H_1\cap \dots \cap H_{i-1})\setminus H_{i+1}}
                ~\geq~ \abs{(U\cap H_1\cap \dots \cap H_{i})\setminus H_{i+1}}
                ~\Stackrel{\eqref{eq:onesided:algo:contradiction}}{>}~ \abs{(U\cap H_1\cap \dots \cap H_{i-1})\setminus H_i}\,.
            \end{align*}
            This contradicts the fact $H_i$ is the optimal solution of \pERM{} defined by tolerance $\rho=0$, positive samples $P$, and unlabeled samples $U\cap H_1\cap \dots \cap H_{i-1}$.
            Therefore, due to the contradiction, our supposition in \cref{eq:onesided:algo:contradiction} must be wrong and, hence, \cref{eq:oneSided:monotone} is true.

            Next, we will use the above claim to conclude that $\abs{
                    (U\cap H_1\cap H_2\cap \dots \cap H_{T-1})
                    \setminus H_{T}
                }
                \leq (\sigma\eps)^q \abs{U}$.
            Toward this, observe that the following collection of sets
            \[
                \inbrace{ R_j\coloneqq\, (U\cap H_1\cap H_2\cap \dots \cap H_{j-1})
                    \setminus H_{j} \mid 1\leq j\leq T}\,,
            \]
            is disjoint (as once a point is removed at step $i$, it is not present in the unlabeled pool for any later step), and, hence, $\sum_i \abs{R_i} \leq \abs{U}$.
            This, along with \eqref{eq:oneSided:monotone} which states $\abs{R_1}\geq \abs{R_2}\geq \dots\geq \abs{R_T}$, implies 
            \[
                \abs{
                    (U\cap H_1\cap H_2\cap \dots \cap H_{T-1})
                    \setminus H_{T}
                }
                ~~=~~\abs{R_T}
                ~~\leq~~ \frac{\abs{U}}{T}
                ~~~~~~~\Stackrel{(T=(\sigma\eps)^{-q})}{=}~~~~~~~ (\sigma\eps)^{q}\abs{U}\,.
                \yesnum\label{eq:oneSided:upperBound}
            \] 
            \noindent Now, we are ready to prove the claim.
                    Consider the $T$-th instance of \pERM{} constructed in the above algorithm, \ie{}, the one defined by unlabeled samples $U\cap H_1\cap H_2\cap \dots \cap H_{T-1}$.
                    Since $\optset\supseteq P$, $\optset$ is feasible for this instance.
                    But since $H_{T}$ is the optimal solution of instance {$T$},
                    \[
                        \abs{U\cap H_1\cap H_2\cap \dots \cap H_{T-1} \cap \optset}
                        \geq 
                        \abs{U\cap H_1\cap H_2\cap \dots \cap H_{T-1} \cap H_{T}}\,.
                    \]
                    Taking complements implies 
                    \[
                        \abs{U\cap H_1\cap H_2\cap \dots \cap H_{T-1} \setminus \optset}
                        \leq 
                        \abs{U\cap H_1\cap H_2\cap \dots \cap H_{T-1} \setminus H_{T}}
                        ~~\Stackrel{\eqref{eq:oneSided:upperBound}}{\leq}~~ 
                            (\sigma\eps)^{q} \abs{U}\,.
                    \]
                    Since $H\subseteq H_1\cap H_2\cap \dots \cap H_{T}$, the above in particular implies that
                    \[
                        \abs{U\cap H \setminus \optset}
                        \leq 
                        \abs{U\cap H_1\cap H_2\cap \dots \cap H_{T-1} \setminus \optset}
                        \leq (\sigma\eps)^{q}\abs{U}\,.
    \yesnum\label{eq:oneSided:empiricalUB}
                    \]
                    Now, second-order uniform convergence (\cref{thm:unifConvergence}) with respect to $U$ and the following hypothesis class 
                    \[
                        \hyH_{\cap T} = \inbrace{G_1\cap G_2\cap \dots\cap G_{T}\setminus\optset\colon G_1,G_2,\dots,G_{T} \in \hyH}\,,
                    \]
                    implies that with probability $1-\nfrac{\delta}{2}$
                    \[
                        \cD(H\setminus \optset) \leq (\sigma\eps)^{q} + \wt{O}\!\inparen{\sqrt{
                            \frac{\abs{U\cap H\setminus \optset}}{\abs{U}}\cdot 
                            \frac{\vc{}\!\inparen{\hyH_{\cap T}}+\log{\nfrac{1}{\delta}}}{n}}}
                        + \wt{O}\!\inparen{{\frac{\vc{}\!\inparen{\hyH_{\cap T}}+\log{\nfrac{1}{\delta}}}{n}}}\,.
                    \]
                    Substituting the value of $n$, using that $\vc{}\!\inparen{\hyH_{\cap T}}=\vc{{(\hyH)}}\cdot O(T\log{T})$, and $\nfrac{\abs{U\cap H\setminus\optset}}{\abs{U}}\leq (\sigma\eps)^{q}$ (from \cref{eq:oneSided:empiricalUB}) implies that with probability $1-\nfrac{\delta}{2}$
                    \[
                        \cD(H\setminus \optset) \leq (\sigma\eps)^{q} + \wt{O}\!\inparen{\sqrt{
                            (\sigma\eps)^{q}\cdot \frac{\vc{}(\hyH) T +\log{\nfrac{1}{\delta}}}{n}}
                        }
                        + \wt{O}\!\inparen{
                            \frac{\vc{}(\hyH) T+\log{\nfrac{1}{\delta}}}{n}
                        }
                        \leq O((\sigma\eps)^{q})
                        \,.
                    \]
                    Therefore, \cref{asmp:smoothness}, implies 
                    \[
                        \cDtrue(H\setminus \optset)
                        \leq O(\eps)\,.
                    \]
                    The result follows by scaling $\eps$ by a constant factor and taking a union bound over the events in Steps A and B.
    \end{proof}

    \subsection{Proof of \cref{thm:main} (Computationally Efficient Algorithm)}
        \label{sec:computationalEfficiency}
        In this section, we prove \cref{thm:main}.
        The key step in the proof is to design the following pseudo-approximation algorithm for \pERM{}.
        \begin{restatable}[]{theorem}{constReg}
            \label{thm:constReg}
            Suppose \cref{asmp:smoothness,asmp:posFraction} hold and fix a constant $C>0$.
            Fix any $\delta\in (0,\nfrac{1}{2})$, $k$ such that degree-$k$ polynomials with range in $[-C,\infty)$ $\zeta$-approximate $\hyH$ with respect to $\cD$ in $L_1$-norm for $0<\zeta\leq O\sinparen{\nfrac{\alpha\sigma}{(C+1)}}^{2q}$, and $n=\poly\!\inparen{d^k, \nfrac{1}{\zeta}, \log{\nfrac{1}{\delta}}}$.
            There exists a polynomial-time algorithm (\cref{alg:constReg}) that given $n$ independent samples from the smooth positive-only learning model (\cref{prob:PUlearning}), outputs a degree-$k$ PTF $H$, s.t., the following hold with probability $1-\delta$
            \begin{enumerate}
                \item $\cDtrue\sinparen{\optset \setminus H}\leq O\sinparen{\nfrac{(C+1)}{\alpha\sigma}} \cdot \zeta^{\sfrac{1}{(4q)}}$; and 
                \item $\cD(H)\leq \opt_{\rm PERM} +
                    O\sinparen{\nfrac{(C+1)}{\alpha\sigma}}\cdot {\zeta^{\sfrac{1}{(4q)}}}+O(\delta)$. 
            \end{enumerate} 
            \mbox{Where $\opt_{\rm PERM}\coloneqq \min_{H\in \hyH \,:\, \cPtrue(H)=1} \cD(H)$ is the optimal value of \pERM{} with tolerance $\rho=0$.}\footnote{Note that \pERM{} is defined with respect to samples and, hence, is a random variable. $\opt_{\rm PERM}(\rho)$ is the limiting optimal value as the number of samples goes to infinity.}
        \end{restatable}
        In the remainder of this section, we prove \cref{thm:main} assuming \cref{thm:constReg}, and then present the proof of \cref{thm:constReg} in the next section.

        \paragraph{Outline of Proof of \cref{thm:main} Assuming \cref{thm:constReg}.}
        The proof of \cref{thm:constReg} follows the strategy of the sample complexity proof (\cref{sec:perm}). 
        We iteratively construct instances of \pERM{} and define the final output as the intersection of the (approximate) solutions obtained at each iteration. 
        Compared to the sample complexity setting, several additional challenges arise since here we only find an approximate solution of the \pERM{} instances. 
        In particular, we must ensure that:
        \begin{enumerate}
            \item The hypothesis class $\hyH$ remains well-approximated by low-degree polynomials with respect to the updated unlabeled distributions.
            \item The smoothness condition in \cref{asmp:smoothness} holds at every iteration.
        \end{enumerate}
        To address the former challenge, we include a check that the remaining region has at least a constant fraction of the mass. 
        To address the latter challenge, we also truncate the distribution of positive samples in addition to the distribution of unlabeled samples, and account for this in the analysis.
        Moreover, we truncate the distribution of positive samples and incorporate this modification into our analysis, ensuring that the surviving (\ie{}, non-truncated) region has sufficient mass to bound the runtime of obtaining samples from the truncated distributions via rejection sampling.

    \begin{figure}[h!]
        \centering
        \begin{tikzpicture}[node distance=2.5cm, scale=0.83, transform shape]
        \node[myNode, line width=2pt, text width=9cm] (correctness) at (0,-2) 
            {{Computationally Efficient Algorithm (\cref{alg:main}) \\[1mm] \cref{thm:main}}};

         \node[myNode, line width=1pt, text width=4cm](claim45) at (-7.5, -4.2)
            {{Guarantee on the solution of each \pERM{} instance \\[1mm] \cref{lem:main:chaining}}};

        \node[myNode, line width=2pt, text width=7cm] (thmConstReg) at (0, -4.2)
            {{Efficient Approximation Algorithm for \pERM{} (\cref{alg:constReg})\\[1mm] \cref{thm:constReg}}};

        \node[myNode, line width=1pt, text width=2.2cm](lemRejSamp) at (7.5, -4.2)
            {\cref{lem:main:rejectionSampling}};

        \node[myNode, line width=1pt, text width=4cm](lemNew) at (7.5, -6.2)
            {{{Efficient sampling from created instances}}};

        \draw[-stealth, line width=0.5mm](lemNew) -- (lemRejSamp);

        \node[myNodeFlex, line width=1pt, text width=4.5cm](lem44) at (-7.5, -7)
            {{Assumptions required by \cref{thm:constReg} hold in each iteration (\cref{lem:main:assumptions})}};

        \node[myNodeFlex, line width=1pt, text width=2.25cm](r2) at (-3.5, -6.5)
            {{Feasibility \\ [1mm] \cref{result:constReg:feasibility}}};

        \node[myNodeFlex, line width=1pt, text width=2.25cm](r1) at (0, -6.5)
        {{{Runtime} \\[1mm] \cref{result:constReg:polyTimeAlg}}};

        \node[myNodeFlex, line width=1pt, text width=2.25cm](r3) at (3.5, -6.5)
            {{Optimality \\ [1mm] \cref{result:constReg:optimality}}};

        \draw[-stealth, line width=0.5mm] (claim45) -- (correctness); 
        \draw[-stealth, line width=0.5mm] (lem44) -- (claim45); 
        \draw[-stealth, line width=0.5mm] (thmConstReg) -- (claim45); 
        \draw[-stealth, line width=0.5mm] (thmConstReg) -- (lemRejSamp); 
        \draw[-stealth, line width=0.5mm] (thmConstReg) -- (correctness); 
        \draw[-stealth, line width=0.5mm] (lemRejSamp) -- (correctness);

        \draw[-stealth, line width=0.5mm](r1) -- (thmConstReg);
        \draw[-stealth, line width=0.5mm](r2) -- (thmConstReg);
        \draw[-stealth, line width=0.5mm](r3) -- (thmConstReg);

        \node[myNode, line width=1pt, text width=4cm] (lem48) at (-3.5, -9)
            {{Repetitions are feasible \\[1mm] \cref{lem:constReg:feasibility}}};

        \node[myNodeFlex, line width=1pt, text width=4.25cm] (lem411) at (1, -9)
            {{Boosting via repetition \\ [1mm] \cref{lem:constReg:boost}}};
        \node[myNodeFlex, line width=1pt, text width=4.5cm] (lem410) at (6, -9)
            {{\mbox{Optimal's expected value} \\ [1mm] \cref{lem:constReg:expectedValue}}};

        \node[myNode, line width=1pt, text width=3.75cm](lem47) at (-5.5, -11.2)
            {{Repetitions satisfy uniform convergence \\ [1mm] \cref{lem:constReg:unifConvergenceHoldsWHP}}};
        \node[myNode, line width=1pt, text width=3.5cm](lem49) at (-1.5, -11.2)
            {{Repetitions feasible over samples \\ [1mm] \cref{lem:constReg:feasibility:unconditional}}};

        \draw[-stealth, line width=0.5mm](lem48) -- (r2);
        \draw[-stealth, line width=0.5mm](lem47) -- (lem48);
        \draw[-stealth, line width=0.5mm](lem49) -- (lem48);

        \node[myNodeNarrow, line width=1pt, text width=2.2cm](lem412) at (3.25, -11.2)
            {{\cref{lem:constReg:expectedValue:upperBound}}};
        \node[myNodeNarrow, line width=1pt, text width=2.2cm](lem413) at (6, -11.2)
            {{\cref{lem:constReg:expectedValue:feasibility}}};
        \node[myNodeNarrow, line width=1pt, text width=2.2cm](lem414) at (8.75, -11.2)
            {{\cref{lem:constReg:expectedValue:ub2}}};

        \draw[-stealth, line width=0.5mm](lem411) -- (r3);
        \draw[-stealth, line width=0.5mm](lem410) -- (r3);
        \draw[-stealth, line width=0.5mm](lem412) -- (lem410);
        \draw[-stealth, line width=0.5mm](lem413) -- (lem410);
        \draw[-stealth, line width=0.5mm](lem414) -- (lem410);
        
        \end{tikzpicture}
        \caption{\textit{Outline of the proof of \cref{thm:main}.} 
        The proof of \cref{thm:main} follows a structure analogous to the proof of the sample complexity of smooth positive-only learning (\cref{thm:sampleComplexity}): 
        we construct a sequence of instances of \pERM{} and output the intersection of all the (approximate) solutions $H_1,H_2,\dots$ to the instances created.
        To approximately solve each instance created, we use an approximation algorithm for \pERM{} (constructed in \cref{thm:constReg}).
        This is the main new technical ingredient in the proof of \cref{thm:main} compared to \cref{thm:sampleComplexity}.
        }
        \label{fig:outline:compComplexity}
    \end{figure}
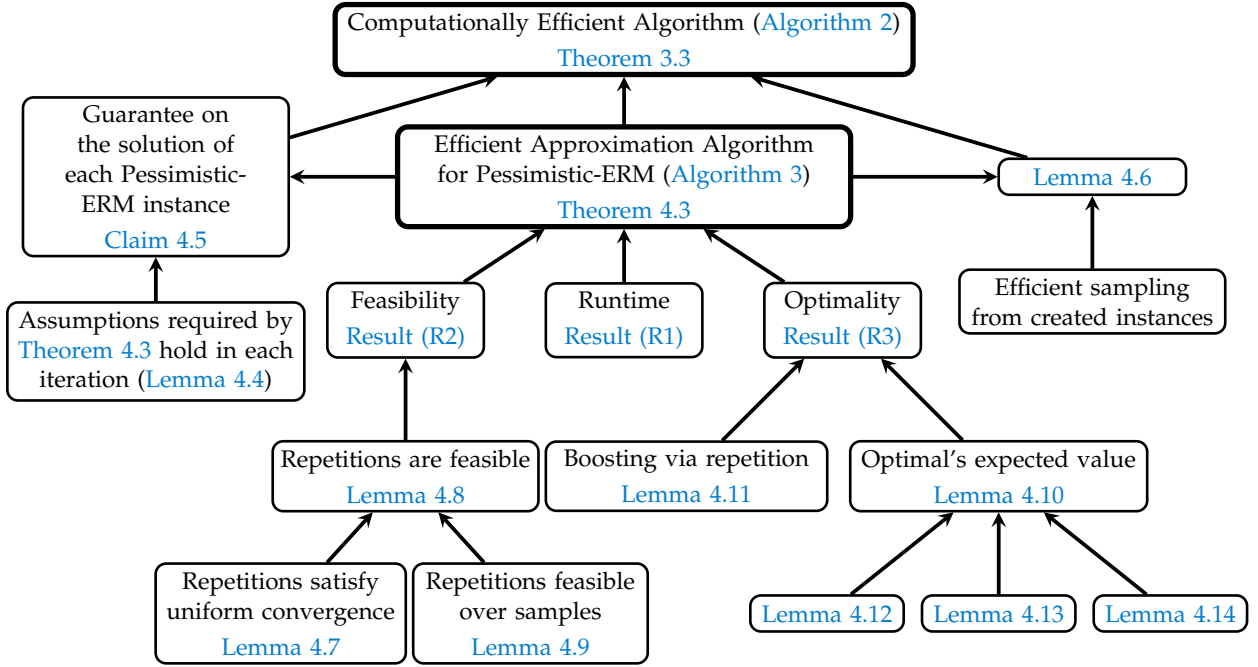

    \paragraph{Proof of \cref{thm:main}.}
        We use the algorithm in \cref{alg:main}.
            First, we rescale $\eps$ and $\delta$ to ensure that $\eps\leq \nfrac{\alpha}{6}$ and $\delta\leq (\sigma\eps)^{2q}$. %
        Due to the scaling in $\eps$, the relation between $\zeta$ and $\eps$ becomes
        \[
            \zeta\leq O\!\inparen{\frac{(\sigma\alpha\eps)^{8q^2+5q}}{(C+1)^{4q}}} 
        \]
        where $\eps$ is the rescaled accuracy parameter, and we will use this updated relation in our analysis.
        Due to the scaling in $\delta$, the running time increases by logarithmic factors in $\max\!\inbrace{\nfrac{1}{\sigma},\nfrac{1}{\eps^q}}$.
        The following constants show up in our analysis:
        \[
            \gamma \coloneqq (\sigma \eps)^q\,,\quad
            T\coloneqq \frac{1}{\gamma}\,,\quad
            \widetilde{\zeta}\coloneqq \Theta\!\inparen{\frac{\zeta}{\gamma}}\,,\quad
            \widetilde{\alpha}\coloneqq \frac{5\alpha}{6}\,,\quad
            \widetilde{\sigma}\coloneqq \frac{5\alpha\sigma}{6}\,,\quadand
            \wt{\delta}\coloneqq \frac{\delta}{20T}\,.
            \yesnum\label{eq:proof:new:defs}
        \]
        In particular, $\delta\leq \gamma^2$.

        \begin{algorithm}[htb!]
        \caption{\textsf{Iterative-}\constrainedreg{} (from \cref{thm:main})}
        \begin{algorithmic}[1]
        \Procedure{\rm \textsf{Iterative-}\constrainedreg{}}{$\eps,\delta,\sigma,q,\alpha,k,C$}
            \vspace{2mm}
            \State {If $\eps \geq \nfrac{\alpha}{6}$, set $\eps=\nfrac{\alpha}{6}$}
            \State Set the accuracy parameters $\zeta=O\sinparen{\sinparen{{C+1}}^{-4q}\cdot (\sigma\eps\alpha)^{8q^2+5q}}$
            and $\gamma=(\sigma \eps)^q$
            \State Set the number of iterations as $T=\nfrac{1}{\gamma}$
            \State Set internal accuracy and confidence parameters as $\widetilde{\zeta}=O\!\inparen{\nfrac{\zeta}{\gamma}}$ and $\wt{\delta}=\nfrac{\delta}{20T}$
            \vspace{2mm}
            \vspace{2mm}
            \For{$1\leq i\leq T$}
                \State Define the survival set $S_i=H_1\cap H_2\cap\dots \cap H_{i-1}$ if $i>1$ and, otherwise, $S_i=\R^d$
                \vspace{3mm}
                \item[] \phantom{.}\qquad~~ \textit{\#~~Step 1 (Ensure Low-Degree Approximation Exists)}
                \State Obtain $m=\wt{O}((\nfrac{1}{\gamma^3}) \cdot (d^k+\log{\nfrac{T}{\delta}}))$ samples $U$ from $\cD$ %
                \If{$i > 1$ and $\abs{U\cap S_i}\leq \gamma \abs{U}$} \label{step:alg:main:massCheck}
                    \State Break out of the for-loop
                \EndIf{}
                \vspace{3mm}
                \item[] \phantom{.}\qquad~~ \textit{\#~~Step 2 (Update Sample Distributions)}
                \State Define $\cD_i,\cPtrue_i,\cDtrue_i$ as the truncations of $\cD,\cPtrue,\cDtrue$ to $S_i$
                \State Construct sampling oracles for $\cD_i$ and $\cPtrue_i$ via rejection sampling from $\cD$ and $\cPtrue$
                \vspace{3mm}
                \item[] \phantom{.}\qquad~~ \textit{\#~~Step 3 (Solve $i$-th \pERM{} Instance)}
                \State Set $\wt{\sigma}\gets\nfrac{5\alpha\sigma}{6}$ and $\wt{\alpha}\gets\nfrac{5\alpha}{6}$
                \State \mbox{$H_i~\longleftarrow$~\hyperref[alg:constReg]{{\rm \textsf{Boosted-}}\constrainedreg{}}($\widetilde{\zeta},\wt{\delta},\wt{\sigma},\wt{\alpha},q,k,C$) with sample access to $\cD_i$ and $\cPtrue_i$}
                \item[] \phantom{.}\qquad~~ \textit{\#~~Here, \hyperref[alg:constReg]{{\rm \textsf{Boosted-}}\constrainedreg{}} is the algorithm referenced in \cref{thm:constReg}}
            \EndFor{}
            \vspace{2mm}
            \State \textbf{return} $H\coloneqq H_1\cap H_2\cap\dots\cap H_\ell$ where $\ell$ is the last iteration of the for-loop completed
        \EndProcedure
        \end{algorithmic}
        \label{alg:main}
        \end{algorithm}

\paragraph{Assumptions Required by \cref{thm:constReg} Hold.}
For each $1\leq i\leq T$, let $\evG_i$ denote the event that either iteration $i$ of \cref{alg:main} is not reached, or else the output hypothesis $H_i$ of the $i$-th call to \hyperref[alg:constReg]{{\rm \textsf{Boosted-}}\constrainedreg{}} satisfies
\begin{align}
    \cDtrue_i(\optset\setminus H_i)
        \leq \gamma^2
        \quadand\quad
    \cD_i(H_i)
        \leq \opt_{\rm PERM}(i)+\gamma^2\,,
    \label{eq:main:constRegGuarantee}
\end{align}
where, whenever iteration $i$ is reached,
\[
    \opt_{\rm PERM}(i)\coloneqq \min_{H\in \hyH}\cD_i(H)\,,
    \quadtext{such that\,,} \cPtrue_i(H)=1\,.
\]
We begin by proving that, conditioned on the good events from the earlier iterations, the $i$-th truncated instance satisfies all of the assumptions needed to invoke \cref{thm:constReg}.

\begin{lemma}\label{lem:main:assumptions}
    Fix any iteration $i$ that is reached by the for-loop in \cref{alg:main}, and condition on the event $\bigcap_{j=1}^{i-1}\evG_j$.
    Then the following hold:
    \begin{enumerate}[leftmargin=15pt,itemsep=0pt]
        \item \label{item:main:assumptions:mass}
            $\cPtrue(S_i)\geq \nfrac{5}{6},$
            $\cDtrue(S_i)\geq \nfrac{5\alpha}{6},$ and 
            $\cD(S_i)\geq \inparen{\nfrac{5\alpha\sigma}{6}}^q \geq \gamma$.
        \item \label{item:main:assumptions:approx}
        Degree-$k$ polynomials with range in $[-C,\infty)$ $\widetilde{\zeta}$-approximate $\hyH$ with respect to $\cD_i$ in $L_1$-norm.
        \item \label{item:main:assumptions:smoothness}
        The triple $(\cDtrue_i,\cPtrue_i,\cD_i)$ satisfies \cref{asmp:smoothness,asmp:posFraction} with parameters $(q,\wt{\sigma},\wt{\alpha})$.
        \item \label{item:main:assumptions:range}
        The approximation parameter $\widetilde{\zeta}$ satisfies
        $\widetilde{\zeta}\leq O\sinparen{\!\inparen{\sfrac{\widetilde{\alpha}\widetilde{\sigma}}{(C+1)}}^{2q}}.$
    \end{enumerate}
\end{lemma}
\begin{proof}[Proof of \cref{lem:main:assumptions}]
    We prove the four items in order.

    \paragraph{Proof of \cref{item:main:assumptions:mass}.}
    For every completed iteration $j\in\inbrace{1,2,\dots,i-1}$, the event $\evG_j$ implies
    \[
        \cDtrue\!\inparen{S_j\cap \optset\setminus H_j}
        = \cDtrue(S_j)\cdot \cDtrue_j(\optset\setminus H_j)
        \leq \gamma^2\,.
    \]
    Therefore,
    \[
        \cPtrue(S_j\setminus H_j)
        = \frac{\cDtrue(S_j\cap \optset\setminus H_j)}{\cDtrue(\optset)}
        \leq \frac{\gamma^2}{\alpha}\,.
    \] 
    Since $S_{j+1}=S_j\cap H_j$, the sets $S_1\setminus H_1,
        S_2\setminus H_2,
        \dots,
        S_{i-1}\setminus H_{i-1}$
    are disjoint with union $\R^d\setminus S_i$.
    Hence, 
    \begin{align*}
        \cPtrue(S_i)
        = 1-\sum_{j=1}^{i-1}\cPtrue(S_j\setminus H_j) 
        \stackrel{}{\geq}
        1-\frac{(i-1)\gamma^2}{\alpha} 
        ~~\stackrel{i\leq T+1}{\geq}~~
        1-\frac{T\gamma^2}{\alpha} 
        ~~\stackrel{T=\nfrac{1}{\gamma}}{=}~~
        1-\frac{\gamma}{\alpha} 
        ~~\stackrel{\gamma\leq \nfrac{\alpha}{6}}{\geq}~~
        \frac{5}{6}\,.
        \yesnum\label{eq:proof:new:56}
    \end{align*} 
    It follows that $\cDtrue(S_i)
        \geq \cDtrue(S_i\cap \optset)
        = \cDtrue(\optset)\cdot \cPtrue(S_i)
        \geq \sfrac{5\alpha}{6}.$
    Applying \cref{asmp:smoothness} to the set $S_i$, we obtain
    \[
        \cD(S_i)
        ~\geq~ \inparen{\sigma\cdot \cDtrue(S_i)}^q
        ~\geq~ \inparen{\frac{5\alpha\sigma}{6}}^q
        ~\qquad\Stackrel{\substack{\eps\leq \nfrac{\alpha}{6}\,,~ \gamma=(\sigma\eps)^q}}{\geq}\qquad~
        \gamma\
        \yesnum\label{eq:proof:new:item1}\,.
    \]
    This proves Item~\ref{item:main:assumptions:mass}.

    \paragraph{Proof of \cref{item:main:assumptions:approx}.}
    Fix any $H'\in \hyH$.
    Since degree-$k$ polynomials with range in $[-C,\infty)$ $\zeta$-approximate $\hyH$ with respect to $\cD$ in $L_1$-norm, there exists a degree-$k$ polynomial $p'\colon \R^d\to \R$ with range in $[-C,\infty)$ such that
    $\Ex_{\cD}\insquare{\abs{p'(x)-\mathds{1}\sinbrace{x\in H'}}}\leq \zeta.$
    Since $\cD_i$ is the truncation of $\cD$ to $S_i$, a change of measure gives
    \begin{align*}
        \Ex_{\cD_i}\insquare{\abs{p'(x)-\mathds{1}\sinbrace{x\in H'}}}
        = \frac{
            \Ex_{\cD}\insquare{
                \mathds{1}\sinbrace{x\in S_i}\cdot
                \abs{p'(x)-\mathds{1}\sinbrace{x\in H'}}
            }
        }{\cD(S_i)}
        \leq \frac{\zeta}{\cD(S_i)}
        ~~\Stackrel{\eqref{eq:proof:new:item1}}{\leq}~~ 
            O\!\inparen{\frac{\zeta}{\gamma}}
        ~~\Stackrel{\eqref{eq:proof:new:defs}}{=}~~\wt{\zeta}\,,
    \end{align*} 
    Since this holds for all $H'\in \hyH$ and the range of $p'$ is unchanged, Item~\ref{item:main:assumptions:approx} follows.

    \paragraph{Proof of \cref{item:main:assumptions:smoothness}.}
    By construction, $\cPtrue_i$ is the truncation of $\cDtrue_i$ to the set $\optset$.
    Moreover, 
    \[
        \cDtrue_i(\optset)
        = \frac{\cDtrue(S_i\cap \optset)}{\cDtrue(S_i)}
        \geq \cDtrue(S_i\cap \optset)
        = \cDtrue(\optset)\cdot \cPtrue(S_i)
        \Stackrel{\eqref{eq:proof:new:56}}{\geq}~~ \frac{5\alpha}{6}
        ~~\Stackrel{\eqref{eq:proof:new:defs}}{=}~~\widetilde{\alpha}\,.
    \]
    Thus, \cref{asmp:posFraction} holds for $(\cDtrue_i,\cPtrue_i,\cD_i)$ with parameter $\widetilde{\alpha}$.
    To verify \cref{asmp:smoothness}, consider any measurable set $T\subseteq \R^d$.
    Then
    \begin{align*}
        \cDtrue_i(T)
        &= \frac{\cDtrue(T\cap S_i)}{\cDtrue(S_i)}\\
        &\leq \frac{\cD(T\cap S_i)^{1/q}}{\sigma\cdot \cDtrue(S_i)}
        \tag{using \cref{asmp:smoothness}}\\
        &= \frac{\cD(S_i)^{1/q}}{\sigma\cdot \cDtrue(S_i)}\cdot \cD_i(T)^{1/q}\\
        &\leq \frac{1}{\sigma\cdot \cDtrue(S_i)}\cdot \cD_i(T)^{1/q}\\
        &\leq \frac{1}{\widetilde{\sigma}}\cdot \cD_i(T)^{1/q}\,,
        \tag{using \cref{eq:proof:new:item1,eq:proof:new:defs}}
    \end{align*}
    Therefore, \cref{asmp:smoothness} holds for $(\cDtrue_i,\cPtrue_i,\cD_i)$ with parameters $q$ and $\widetilde{\sigma}$.

    \paragraph{Proof of \cref{item:main:assumptions:range}.}
    First, applying \cref{asmp:smoothness} to $\R^d$ shows that $\sigma\leq 1$.
    Since also $\alpha\leq 1$ and $\eps\in(0,1)$, the theorem hypothesis on $\zeta$ implies
    \begin{align*}
        \widetilde{\zeta}
        \quad
        &=\quad O\!\inparen{\frac{\zeta}{\gamma}}\\
        &\leq\quad O\!\inparen{
            \frac{1}{(C+1)^{4q}}
            \cdot
            \frac{(\sigma\alpha\eps)^{8q^2+5q}}{(\sigma\eps)^q}
        }\\
        &=\quad O\!\inparen{
            \frac{1}{(C+1)^{4q}}
            \cdot
            \sigma^{8q^2+4q}\alpha^{8q^2+5q}\eps^{8q^2+4q}
        }\\
        &\leq\quad O\!\inparen{
            \frac{\alpha^{4q}\sigma^{2q}}{(C+1)^{4q}}
        }\\
        &\Stackrel{\widetilde{\alpha}\widetilde{\sigma}=\Theta(\alpha^2\sigma)}{\leq}\quad O\!\inparen{
            \inparen{\frac{\widetilde{\alpha}\widetilde{\sigma}}{C+1}}^{2q}
        }\,. 
        \qedhere{}
    \end{align*}
\end{proof}

\paragraph{Guarantee of \cref{thm:constReg}.}
We now use \cref{lem:main:assumptions} inductively to show that the guarantee of \cref{thm:constReg} holds in every executed iteration.

\begin{lemma}\label{lem:main:constRegGuaranteeWHP}
    For every $1\leq i\leq T$, conditioned on the event $\bigcap_{j=1}^{i-1}\evG_j$, the event $\evG_i$ holds with probability at least $1-\wt{\delta}$.
    Consequently, $ \Pr\sinparen{\bigcap_{i=1}^{T}\evG_i}\geq 1-\nfrac{\delta}{20}.$
\end{lemma}
\begin{proof}
    Fix any $1\leq i\leq T$ and condition on the event $\bigcap_{j=1}^{i-1}\evG_j$.
    If iteration $i$ is not reached by \cref{alg:main}, then $\evG_i$ holds vacuously.
    Otherwise, iteration $i$ is reached.
    By \cref{lem:main:assumptions}, the $i$-th truncated instance satisfies all of the assumptions required to invoke \cref{thm:constReg} with parameters $(\wt{\zeta},\wt{\delta},\wt{\sigma},q,\wt{\alpha},k,C)$.
    Therefore, with probability at least $1-\wt{\delta}$, the $i$-th call to \hyperref[alg:constReg]{{\rm \textsf{Boosted-}}\constrainedreg{}} returns a hypothesis $H_i$ satisfying
    \begin{align*}
        \cDtrue_i(\optset\setminus H_i)
        \leq
        O\!\inparen{\frac{C+1}{\widetilde{\alpha}\widetilde{\sigma}}}\cdot \widetilde{\zeta}^{1/(4q)}
        \quadand
        \cD_i(H_i)
        \leq
        \opt_{\rm PERM}(i)
        +
        O\!\inparen{\frac{C+1}{\widetilde{\alpha}\widetilde{\sigma}}}\cdot \widetilde{\zeta}^{1/(4q)}
        +
        O(\wt{\delta})\,.
    \end{align*}
    Since $\widetilde{\alpha}\widetilde{\sigma}=\Theta(\alpha^2\sigma)$ and $\widetilde{\zeta}=O(\zeta/\gamma)$, it holds that
    \[
        O\!\inparen{\frac{C+1}{\widetilde{\alpha}\widetilde{\sigma}}}\cdot \widetilde{\zeta}^{1/(4q)}
        \leq
        O\!\inparen{\frac{C+1}{\alpha^2\sigma}}\cdot \inparen{\frac{\zeta}{\gamma}}^{1/(4q)}\,.
    \]
    Further, by construction $\inparen{\nfrac{\zeta}{\gamma}}^{1/(4q)}
        \leq
        O\!\inparen{
            (\sigma\alpha\eps)^{2q+1}/ (C+1)
        }.$
    Hence,
    \[
        O\!\inparen{\frac{C+1}{\alpha^2\sigma}}\cdot \inparen{\frac{\zeta}{\gamma}}^{1/(4q)}
        \leq
        O\!\inparen{
            \frac{1}{\alpha^2\sigma}\cdot (\sigma\alpha\eps)^{2q+1}
        }
        =
        O\!\inparen{
            \sigma^{2q}\alpha^{2q-1}\eps^{2q+1}
        }
        \leq
        O\!\inparen{\sigma^{2q}\eps^{2q+1}}
        \leq
        \gamma^2\,,
    \]
    where we used $\alpha\leq 1$ and $\eps\leq 1$.
    Also, $\wt{\delta}\leq \delta\leq \gamma^2$.
    Therefore, conditioned on $\bigcap_{j=1}^{i-1}\evG_j$, the event $\evG_i$ holds with probability at least $1-\wt{\delta}$.
    Finally,
    \[
        \Pr\!\inparen{\bigcap\nolimits_{i=1}^{T}\evG_i}
        ~~=~~
        \prod_{i=1}^{T}
        \Pr\!\inparen{
            \evG_i \,\middle|\, \bigcap\nolimits_{j=1}^{i-1}\evG_j
        }
        ~~\geq~~
        (1-\wt{\delta})^T
        ~~\geq~~
        1-T\wt{\delta}
        ~~\quad\Stackrel{\wt{\delta}=\sfrac{\delta}{(20T)}}{\geq}~~\quad
        1-\frac{\delta}{20}\,.\qedhere{}
    \]
\end{proof}
In the remainder of the proof, we condition on the event $\bigcap_{i=1}^{T}\evG_i$ from \cref{lem:main:constRegGuaranteeWHP}.
On this event, \eqref{eq:main:constRegGuarantee} holds for every executed iteration.

    \paragraph{Correctness of \cref{alg:main}.}
Now, we are ready to complete the proof of correctness of \cref{alg:main}.\footnote{We continue to work on the event $\bigcap_{i=1}^{T}\evG_i$ from \cref{lem:main:constRegGuaranteeWHP}, and we additionally intersect it with the uniform-convergence event used in Step~\ref{step:alg:main:massCheck} for every iteration reached by the algorithm. By \cref{lem:main:constRegGuaranteeWHP}, the former event holds with probability at least $1-\nfrac{\delta}{20}$. By the choice of sample sizes in Step~\ref{step:alg:main:massCheck} and a union bound over at most $T$ reached iterations, the latter event holds with probability at least $1-\nfrac{\delta}{20}$. Hence, the combined event holds with probability at least $1-\nfrac{\delta}{10}$, which we absorb into $\delta$ by the constant-factor rescaling at the beginning of the proof.}
We consider two cases depending on whether the for-loop exits.

\itparagraph{Case A (For-loop breaks at iteration $i$):}
Suppose that the if statement first breaks at iteration $i$.
Hence, it must be that $ \abs{U\cap S_i}\leq \gamma\abs{U}.$
By the same uniform-convergence event, $\cD(S_i)\leq 2\gamma.$
Since the last completed iteration is $i-1$, the output is $H=H_1\cap H_2\cap \dots \cap H_{i-1}=S_i.$
Therefore,
\[
    \cD(H\setminus \optset)\leq \cD(H)=\cD(S_i)\leq 2\gamma\,.
\]
Applying \cref{asmp:smoothness},
\[
    \cDtrue(H\setminus \optset)
    ~~\leq~~ \frac{1}{\sigma}\cdot \cD(H\setminus \optset)^{1/q}
    ~~\leq~~ \frac{1}{\sigma}\cdot (2\gamma)^{1/q}
    \qquad\Stackrel{q\geq 1,~\gamma=(\sigma\eps)^q}{\leq}\qquad 2\eps\,.
    \yesnum\label{eq:proof:new:ub1}
\]
Also, $\optset\setminus H
    = \bigcup_{j=1}^{i-1}\inparen{S_j\cap \optset\setminus H_j}$
and the sets in the union are disjoint.
Hence,
\begin{align*}
    \cDtrue(\optset\setminus H)
    = \sum_{j=1}^{i-1}\cDtrue\!\inparen{S_j\cap \optset\setminus H_j}
    = \sum_{j=1}^{i-1}\cDtrue(S_j)\cdot \cDtrue_j(\optset\setminus H_j)
    \Stackrel{\eqref{eq:main:constRegGuarantee}}{\leq} (i-1)\gamma^2
    ~\quad\Stackrel{i\leq T=\nfrac{1}{\gamma}}{\leq}\quad~ \gamma
    \leq \eps\,.
    \yesnum\label{eq:proof:new:ub2}
\end{align*}
Combining \cref{eq:proof:new:ub1,eq:proof:new:ub2} yields
\[
    \cDtrue(\optset\triangle H)\leq 3\eps\,.
\]
By the constant-factor rescaling of $\eps$ performed at the beginning of the proof, this is at most $\eps$, completing the proof of correctness in this case.

\itparagraph{Case B (For-loop never breaks):}
Now suppose that the break condition is never satisfied and, hence, all $T$ iterations of the for-loop execute.
In this case, $H=H_1\cap H_2\cap \dots \cap H_T=S_{T+1}$
Also, as before
\[
    \optset\setminus H
    = \bigcup_{i=1}^{T}\inparen{S_i\cap \optset\setminus H_i}\,,
\]
and the sets in the union are disjoint.
Hence,
\begin{align*}
    \cDtrue(\optset\setminus H)
    = \sum_{i=1}^{T}\cDtrue\inparen{S_i\cap \optset\setminus H_i}
    = \sum_{i=1}^{T}\cDtrue(S_i)\cdot \cDtrue_i(\optset\setminus H_i)
    ~~\Stackrel{\eqref{eq:main:constRegGuarantee}}{\leq}~~ T\gamma^2
    \quad\Stackrel{T=\nfrac{1}{\gamma}}{=}\quad \gamma\,.
    \yesnum\label{eq:computational:caseB}
\end{align*}
To bound $\cDtrue(H\setminus \optset)$, we first bound $\cD(H\setminus \optset)$.

\begin{claim}[Potential drop]
    \label{lem:main:chaining}
    For every $1\leq i\leq T$, it holds that
    $\cD(S_{i+1}\setminus \optset)
        \leq 
        \cD(S_i\cap \optset)-\cD(S_{i+1}\cap \optset)+\gamma^2.$
    Consequently, $\cD(H\setminus \optset)=\cD(S_{T+1}\setminus \optset)
        \leq \frac{1}{T}+\gamma^2
        \leq 2\gamma.$
\end{claim}
\begin{proof}[Proof of \cref{lem:main:chaining}]
    Fix any $1\leq i \leq T$.
    Since $\optset$ is feasible for the $i$-th instance of \pERM{}, 
    \[
        \opt_{\rm PERM}(i)\leq \cD_i(\optset)\,.
    \]
    Therefore, \eqref{eq:main:constRegGuarantee} implies
    \[
        \cD_i(H_i)\leq \cD_i(\optset)+\gamma^2\,.
    \]
    Multiplying by $\cD(S_i)$ and using that $S_{i+1}=S_i\cap H_i$, we obtain
    \begin{align*}
        \cD(S_{i+1})
        = \cD(S_i)\cdot \cD_i(H_i)
        \leq \cD(S_i)\cdot \cD_i(\optset)+\gamma^2\cD(S_i)
        = \cD(S_i\cap \optset)+\gamma^2\cD(S_i)
        \leq \cD(S_i\cap \optset)+\gamma^2\,.
    \end{align*}
    Subtracting $\cD(S_{i+1}\cap \optset)$ from both sides gives, as required, that
    \[
        \cD(S_{i+1}\setminus \optset)
        \leq
        \cD(S_i\cap \optset)-\cD(S_{i+1}\cap \optset)+\gamma^2\,.
    \]
    Summing the above inequality over $i=1,2,\dots,T$ gives
    \begin{align*}
        \sum_{i=1}^{T}\cD(S_{i+1}\setminus \optset)
        &\leq \sum_{i=1}^{T}\inparen{\cD(S_i\cap \optset)-\cD(S_{i+1}\cap \optset)}+T\gamma^2\\
        &= \cD(S_1\cap \optset)-\cD(S_{T+1}\cap \optset)+T\gamma^2\\
        &\leq 1+T\gamma^2\,.
    \end{align*}
    Moreover, because $S_{i+1}\supseteq S_{T+1}=H$ for every $i$, we have
        $\cD(S_{i+1}\setminus \optset)\geq \cD(H\setminus \optset).$
    Therefore,
    \[
        T\cdot \cD(H\setminus \optset)
        \leq \sum_{i=1}^{T}\cD(S_{i+1}\setminus \optset)
        \leq 1+T\gamma^2\,,
    \]
    and, hence,
    \[
        \cD(H\setminus \optset)
        \leq \frac{1}{T}+\gamma^2
        \leq 2\gamma\,.
        \qedhere
    \]
\end{proof}
Applying \cref{asmp:smoothness}, we have
\[
    \cDtrue(H\setminus \optset)
    \leq \frac{1}{\sigma}\cdot \cD(H\setminus \optset)^{1/q}
    \leq \frac{1}{\sigma}\cdot (2\gamma)^{1/q}
    \leq 2\eps\,.
\]
Combining this with \eqref{eq:computational:caseB} implies
\[
    \cDtrue(\optset\triangle H)\leq 3\eps\,.
\]
By the constant-factor rescaling of $\eps$ performed at the beginning of the proof, this is at most $\eps$, completing the proof of correctness in Case B.

    \paragraph{Computational Efficiency.}
        Finally, we show that \cref{alg:main} runs in sample polynomial time.
        
        \begin{lemma}\label{lem:main:rejectionSampling}
            Consider the setting of \cref{thm:main}. \cref{alg:main} runs in time $\poly(n)$ with probability $1-\delta$ where $n=\poly(d^k,\nfrac{1}{\zeta},\log{\nfrac{1}{\delta}})$.
        \end{lemma}
        \vspace{-7mm}
        \begin{proof}[Proof of \cref{lem:main:rejectionSampling}]
            The only computationally non-trivial step in \cref{alg:main} is the call to boosted constraint regression and the construction of and sampling from distributions $\cD_i$ and $\cPtrue_i$ via rejection sampling.
Condition on the event $\bigcap_{i=1}^{T}\evG_i$ from \cref{lem:main:constRegGuaranteeWHP}.
On this event, every reached iteration satisfies the hypotheses of \cref{thm:constReg} by \cref{lem:main:assumptions}.
Hence, each reached call to boosted constraint regression with parameters $(\widetilde{\zeta},\wt{\delta},\wt{\sigma},q,\wt{\alpha},k,C)$, viewed as an algorithm with unit-cost sample access to $\cD_i$ and $\cPtrue_i$, uses
$\poly\!\inparen{d^k,\nfrac{1}{\widetilde{\zeta}},\log{\nfrac{1}{\wt{\delta}}}}$
time and oracle calls.
Since $\widetilde{\zeta}=\Theta(\nfrac{\zeta}{\gamma})$ and $\gamma\leq 1$, we have
$\nfrac{1}{\widetilde{\zeta}}
    = O\!\inparen{\nfrac{\gamma}{\zeta}}
    \leq O\!\inparen{\nfrac{1}{\zeta}}.$
Further, as $\wt{\delta}=\nfrac{\delta}{20T}$,
$\log{\nfrac{1}{\wt{\delta}}}
    = \log{\nfrac{1}{\delta}} + \log T + O(1).$
Also, since $T=\nfrac{1}{\gamma}=(\sigma\eps)^{-q}$ and 
$\zeta\leq O\sinparen{(C+1)^{-4q}\cdot (\sigma\alpha\eps)^{8q^2+5q}},$
it follows that $T\leq \poly\!\inparen{\nfrac{1}{\zeta}}$ and, therefore,
$\log T \leq O\!\inparen{\log{\nfrac{1}{\zeta}}}.$
Hence, the oracle complexity of each reached call is at most
$\poly\!\inparen{d^k,\nfrac{1}{\zeta},\log{\nfrac{1}{\delta}}}.$
            Hence, it remains to bound the time of rejection sampling.
            
            For this, it suffices to show a lower bound on the mass of $\cD(S_i)$ and $\cPtrue(S_i)$ for each $i$ so that it is not too small, since the inverse of this quantity is the expected number of samples from $\cD$ and $\cPtrue$ one would need to see a single sample in $S_i$. 
            We already proved in \cref{item:main:assumptions:mass} of \cref{lem:main:assumptions} that 
            \[
                \cPtrue(S_i)\geq \frac{5}{6}
            \quadand
            \cD(S_i)\geq \inparen{\frac{5\alpha\sigma}{6}}^q \geq \gamma\,.
            \]
            This ensures that each call to the sampling oracle in \cref{alg:main}, with probability $1-\eta$, runs in time $O((\nfrac{1}{\gamma})\log{\nfrac{1}{\eta}})$ (for any $\eta\in(0,1)$). 
            Selecting $\eta$ to be a suitably small multiple of $\delta$ and taking a union bound over $\poly(d^k,\nfrac{1}{\zeta},\log{\nfrac{1}{\delta}})$ calls from \cref{alg:main} implies the claim with only a logarithmic increase in the running time. 
        \end{proof}

    \subsection{Proof of \cref{thm:constReg} (Efficient Approximation Algorithm for \pERM{})}
        \label{sec:constReg}\label{sec:efficientPERM}
        In this section, we present a pseudo-approximation algorithm for \pERM{} -- \cref{alg:constReg}.
        \cref{alg:constReg} comes with the guarantees in \cref{thm:constReg}, which we restate below.
        \constReg*
        \noindent The algorithm in \cref{thm:constReg} is summarized in \cref{alg:constReg}.

        \begin{algorithm}[htb!]
        \caption{\textsf{Boosted-}\constrainedreg{}}
        \begin{algorithmic}[1]
        \Procedure{{\rm \textsf{Boosted-}}\constrainedreg{}}{$\zeta,\delta,\sigma,\alpha,q,k,C$}
            \State Initialize variable $i=1$

            \vspace{2mm}
            
            \State\textit{\#~~Subroutine A (Obtain Samples)}
                \State Obtain sets $P$ and $U$ of $\wt{\Omega}\inparen{
                    \inparen{\nfrac{1}{\zeta^2}}\cdot
                        \inparen{d^k+\log{\nfrac{1}{\delta}}}
                }$ positive and unlabeled samples \label{step:constReg:getSamples} 
                \item[]\quad~ respectively from $\cPtrue$ and $\cD$ (\cref{prob:PUlearning}) 
                \vspace{4mm}
            \State \textit{\#~~Subroutine B (Solve Constrained Regression)}
            \State Find the optimal solution $p_i(\cdot)$ of the following program for $\rho=\Theta\inparen{\frac{C+1}{\alpha\sigma} \cdot \zeta^{1/(2q)}}$ \label{step:constReg:findP} 
            \changetag{Constrained Regression) \ (#1}
            \[
                \min_{{\rm deg}(p)\leq k} \frac{\sum_{x\in U} \abs{{p(x)}}}{\abs{U}}
                \,,~~
                \text{s.t.}\,,~~
                \frac{\sum_{x\in P} \min\!\inbrace{p(x), 1}}{\abs{P}}     \geq 1  -  \rho
                \,.
                \yesnum\label{eq:constReg:prog}
            \]
            \changetag{#1}
            \State \textit{\#~~This program can be solved by formulating it as a linear program over monomials of degree at}
            \State \textit{\#~~most $k$ and using any polynomial-time linear program solver (see \cref{sec:module:linearProgram} for details)}
            \State Let $H_i=\mathds{1}\sinbrace{p_i(\cdot)\geq t_i}$ where $0\leq t_i\leq 1-\sqrt{\rho}$ is chosen to minimize $\sum_{x\in U}  \mathds{1}\sinbrace{p_i(x)\geq t_i}$\label{step:constReg:selectThreshold}  
            \vspace{4mm}
            \State \textit{\#~~Subroutine C (Boost Performance)}
            \State Repeat the Subroutines A and B for $T=O\sinparen{\!\inparen{\nfrac{1}{\sqrt{\zeta}}}\cdot \log{\nfrac{1}{\delta}}}$ times to get $H_1,H_2,\dots,H_T$ 
            \State Select a fresh set $U'$ of $\wt{O}\sinparen{\!\inparen{\nfrac{1}{\zeta}}\cdot \log{\nfrac{1}{\delta}}}$ unlabeled samples from $\cD$ (\cref{prob:PUlearning})
                \label{step:constReg:freshSamples}
            \State \textbf{return} $H_i$ that minimizes $\abs{H_i\cap U'}$ among $1\leq i\leq T$ 
        \EndProcedure
        \end{algorithmic}
        \label{alg:constReg}
        \end{algorithm} 

        \medskip 
        \paragraph{Notation and Conventions.}
        Before proceeding to the proof, we explain some of the conventions we follow in this section.
        First of all, recall that for each $\ell\geq 1$, $\hyP(\ell)$ is the set of all degree-$\ell$ PTFs over $\R^d$. 
        Define $H_{\opt}\in \hyH$ to be any optimal hypothesis for \pERM{}, i.e.,
        \[
            H_{\opt}\in
                \argmin_{H\in \hyH}
                \cD(H)
                \,, 
                \quadtext{such that\,,}
                    \cPtrue(H) = 1
                    \,.
                \yesnum\label{eq:constReg:optimalSolution}
        \]
        The condition $\cPtrue(H_{\opt})=1$ ensures that $H_{\opt}\supseteq \optset\cap \supp(\cDtrue)$ (except on a measure 0 set) but, in general, when $\supp(\cDtrue)\neq \R^d$, $H_{\opt}$ may be very different from $\optset$.
        Finally, by definition
        \[
            \opt_{\rm PERM} = \cD(H_{\opt})\,.
            \yesnum\label{eq:constReg:expectedValue:defOPT}
        \]
        Next, as a convention, we fix $p_{\opt}\colon \R^d\to [-C,\infty)$ to denote the polynomial which $\zeta$-approximates $H_{\opt}$ with respect to $\cD$ (for $\zeta\leq O\sinparen{\frac{\alpha\sigma}{C+1}}^{2q}$ as specified in \cref{thm:constReg}). 
        In other words, $p_{\opt}(\cdot)$ is a polynomial satisfying 
        \begin{align*}
            \Ex_{\cD}\abs{p_{\opt}(x) - \mathds{1}\sinbrace{x\in H_{\opt}}}
                \leq \zeta\,,\qquadand
            \text{for all $x$},\quad 
                p_{\opt}(x)\geq -C\,.
                \yesnum\label{asmp:poptLowerBound}
        \end{align*}

        \paragraph{Outline of Proof of \cref{thm:constReg}.}
        Now we are ready to prove \cref{thm:constReg}.
        For this, we need to prove the following results 
        \begin{enumerate}[label=(R\arabic*)]
            \item \cref{alg:constReg} is a polynomial-time algorithm;
                \label[result]{result:constReg:polyTimeAlg}
            \item $H$ satisfies $\cDtrue(\optset\setminus H)\leq O\sinparen{\nfrac{(C+1)}{(\alpha\sigma)}} \cdot \zeta^{1/(4q)}$ with probability $1-\delta$;
                \label[result]{result:constReg:feasibility}
            \item $H$ satisfies that $\cD(H)\leq \opt_{\rm PERM} + {O(\nfrac{(C+1)}{(\alpha\sigma)})\, \zeta^{1/(4q)}} + O(\delta)$ with probability $1-\delta$
            
                \label[result]{result:constReg:optimality}
        \end{enumerate} 
        We present the proofs of these results in the order above.

        \subsubsection{Proof of \cref{result:constReg:polyTimeAlg}}
            The only computationally non-trivial step in \cref{alg:constReg} is Step~\ref{step:constReg:findP}, which solves a constrained optimization problem over polynomials.
            This can be done in polynomial time by casting the problem as a linear program.
            We present the resulting linear program in \cref{sec:module:linearProgram}.
            
        \subsubsection{Proof of \cref{result:constReg:feasibility}}
            Now that we know that \cref{alg:constReg} can be implemented in polynomial time, we proceed to show that the hypothesis $H$ output by it satisfies $\cDtrue(\optset\setminus H)$ is small.
            To show this, we will show that each hypothesis $H_1, H_2, \dots, H_T$ constructed by \cref{alg:constReg} is feasible for \pERM{} with a small $\rho$ which will turn out to be sufficient due to uniform convergence and the constraint in \pERM{}.
            Define the following events 
            \begin{enumerate}[label=(E)] %
                \item[] Let $\evE_i$ (for $1\leq i\leq T$) be the event that the samples in the $i$-th repetition of Subroutine A satisfy $\zeta$-uniform convergence with respect to degree-$k$ PTFs.\label[event]{def:constReg:eventEi}
            \end{enumerate}
            We will show that conditioned on all of the events $\inbrace{\evE_i\colon 1\leq i\leq T}$ happening, each of $H_1,\dots,H_T$, will satisfy $\cDtrue(\optset\setminus H_i)\leq O\sinparen{\nfrac{(C+1)}{(\alpha\sigma)}}\cdot \zeta^{1/(4q)}$.
            This will be sufficient to prove \cref{result:constReg:feasibility} since $H$ is selected from among these hypotheses and these events hold with a high probability:  
            \begin{lemma}[Uniform Convergence in Each Repetition]\label{lem:constReg:unifConvergenceHoldsWHP}
                With probability $1-\delta$, all of the events $\inbrace{\evE_i\colon 1\leq i\leq T}$ hold. 
            \end{lemma}
            \cref{lem:constReg:unifConvergenceHoldsWHP} follows from standard uniform convergence bounds, and its proof appears in \cref{sec:proofof:lem:constReg:unifConvergenceHoldsWHP}.
            Next, we show that, conditioned on the above set of events happening, feasibility holds.
            \begin{lemma}[{Feasibility for \pERM{} instance in \cref{thm:constReg}}]\label{lem:constReg:feasibility}
                Conditioned on all of the events $\inbrace{\evE_i\colon 1\leq i\leq T}$ happening, for each $1\leq i\leq T$, $
                    \cDtrue(\optset\setminus H_i)\leq O\sinparen{\sqrt{\sfrac{(C+1)}{\alpha\sigma}}}\cdot \zeta^{1/(4q)}\,.
                $
            \end{lemma} 
            For each $1\leq i\leq T$, let $P_i$ and $U_i$ be the samples drawn in Subroutine A in the $i$-th repetition.
            To prove \cref{lem:constReg:feasibility}, we first show that $H_i$ is feasible for \pERM{} over samples $P_i$ and $U_i$.

            \begin{lemma}[{Feasibility for \pERM{} Instances in Each Repetition}]\label{lem:constReg:feasibility:unconditional}
                For each $i$, $H_i$ is feasible for \pERM{} over samples $P_i$ and $U_i$ with tolerance $\sqrt{\rho}= O(\sqrt{\nfrac{(C+1)}{(\alpha\sigma)} })\cdot \zeta^{1/(4q)}$.
            \end{lemma} 
            It is worth noting that unlike \cref{lem:constReg:feasibility}, this result holds \textit{unconditionally}, in other words, it holds for all draws of $P_i$ and $U_i$.
            The first reason for this is that, for each $P_i$ and $U_i$, $H_i$ is a PTF constructed from thresholding a polynomial $p_i$ that is feasible for Program~\eqref{eq:constReg:prog} (Constrained Regression).
            However, this by itself is not sufficient: as the proof below shows it is important that the threshold $t_i$ (which determines $H_i$ as $\mathds{1}\sinbrace{p_i(\cdot)\geq t_i}$) satisfies $t_i\leq 1-\sqrt{\rho}$.
            
            \begin{proof}[Proof of \cref{lem:constReg:feasibility:unconditional}]
                As mentioned above, $H_i$ is the PTF $\mathds{1}\sinbrace{p_i(\cdot)\geq t_i}$ where the polynomial $p_i$ is feasible for Program~\eqref{eq:constReg:prog} (Constrained Regression) and the threshold $t_i$ satisfies $0\leq t_i \leq 1- \sqrt{\rho}$.
                Since $p_i$ is feasible for Program~\eqref{eq:constReg:prog} (Constrained Regression) with positive samples $P_i$, 
                \[
                    \sum_{x\in P_i} \frac{\min\!\inbrace{p(x), 1}}{\abs{P_i}}
                    \geq 1-\rho\,.
                \]
                Let $\unif(P_i)$ be the uniform distribution over elements of $P_i$.
                The result follows by applying Markov's inequality on $\max\!\inbrace{1-p(x),0}$:
                since $1-\min\!\inbrace{p(x),1}=\max\!\inbrace{1-p(x),0}$, the above inequality implies that $\Ex_{\unif(P_i)}[\max\!\inbrace{1-p(x),0}]\leq \rho$.
                Hence, by Markov's inequality, for each $\eta\in (0,1)$, 
                \[
                    \Pr_{x\sim \unif(P_i)}[\max\!\inbrace{1-p(x),0} > \eta]\leq \frac{\rho}{\eta}\,.
                    \yesnum\label{eq:constReg:feasibility:markov}
                \]
                Substituting $\eta=\sqrt{\rho}$, it follows that 
                \[
                    \frac{\sum_{x\in P_i} \mathbb{I}\!\insquare{   x\in H_i } }{\abs{P_i}}
                    =
                    \sum_{x\in P_i} 
                    \frac{\mathbb{I}\!\insquare{
                        p(x) \geq t_i
                    }}{\abs{P_i}}
                    ~~~~~~\Stackrel{t_i \leq 1-\sqrt{\rho}}{\geq}~~~~~~
                    \sum_{x\in P_i} 
                    \frac{\mathbb{I}\!\insquare{
                        p(x) > 1 - \sqrt{\rho}
                    }}{\abs{P_i}}
                    ~~\Stackrel{\eqref{eq:constReg:feasibility:markov}}{\geq}~~
                    {1-\sqrt{\rho}}\,.
                \]
                Therefore, $H_i$ is feasible for \pERM{} over samples $P_i$  and $U_i$ with tolerance $\sqrt{\rho}$.
            \end{proof}
            Now, \cref{lem:constReg:feasibility} follows by appropriately using uniform convergence as we explain below. 
            \begin{proof}[Proof of \cref{lem:constReg:feasibility}]
                Condition on all of the above events happening. 
                Fix any $1\leq i\leq T$ and let $\tau\coloneqq \sqrt{\rho}=O\!\inparen{\sqrt{\nfrac{(C+1)}{(\alpha\sigma)}}}\cdot \zeta^{1/(4q)}$.
                Since (1) $H_i$ is feasible for \pERM{} over samples $P_i$ and $U_i$ with tolerance $\tau$ (\cref{lem:constReg:feasibility:unconditional}), and (2) $P_i$ and $U_i$ satisfy $\zeta$-uniform-convergence with respect to degree-$k$ PTFs due to event $\evE_i$, 
                we get that $\cPtrue(H_i)\geq 1-\zeta-\tau.$
                Therefore, $\cDtrue(\optset\setminus H_i)\leq \cPtrue(\optset\setminus H_i)\leq \zeta + \tau.$
                Since $\tau=O\!\inparen{\sqrt{\nfrac{(C+1)}{(\alpha\sigma)}}}\cdot \zeta^{1/(4q)}$, $\frac{C+1}{\alpha\sigma}\geq 1$, $q\geq 1$, and $\zeta\in(0,1]$, 
                $
                    \zeta
                    \leq
                    \sqrt{\frac{C+1}{\alpha\sigma}}\cdot \zeta^{1/(4q)}
                    = O(\tau).
                $
                Hence, as required,
                $
                    \cDtrue(\optset\setminus H_i)\leq O(\tau)
                    \leq O\!\inparen{\sqrt{\frac{C+1}{\alpha\sigma}}}\cdot \zeta^{1/(4q)}.
                $
            \end{proof}
            \color{black}

        \subsubsection{Proof of \cref{result:constReg:optimality}}
            Recall the definition of $H_{\opt}$ from \cref{eq:constReg:optimalSolution} and that $\opt_{\rm PERM} = \cD(H_{\opt})$ (from \cref{eq:constReg:expectedValue:defOPT}).
            We divide the proof of \cref{result:constReg:optimality} into two parts.
            The first part shows that the \textit{expectation} of the objective value attained by each hypothesis $H_1,H_2,\dots,H_T$ from \cref{alg:constReg} is close to $\cD(H_{\opt})$ (\cref{lem:constReg:expectedValue}).
            The second part shows that the selected $H$ has an objective value close to $\cD(H_{\opt})$ with high probability (\cref{lem:constReg:boost}).

        \paragraph{Part 1 (Expected Value of Optimal).}
            \begin{restatable}[Expected Value of Optimal]{lemma}{constrainedRegExpValue}\label{lem:constReg:expectedValue}
                For each $1\leq i\leq T$, $H_i$ satisfies
                \[
                    \Ex_{U_i\sim \cD,~ P_i\sim \cPtrue}\insquare{
                        \frac{\abs{H_i\cap U_i}}{\abs{U_i}} 
                    }
                    \leq 
                    \cD(H_{\opt})
                    + O\!\inparen{\sqrt{\frac{C+1}{\alpha\sigma}}} \cdot \zeta^{1/(4q)}
                    + O(\delta)
                    \,. 
                \]
            \end{restatable}  

            \paragraph{Outline of the Proof of \cref{lem:constReg:expectedValue} (\cref{sec:proofof:lem:constReg:expectedValue}).}
            Next, we briefly outline the proof of \cref{lem:constReg:expectedValue} and present the complete proof in the next section (\cref{sec:proofof:lem:constReg:expectedValue}).
            This proof is the  part of our analysis that is inspired by the analysis of \lreg{} by \citet*{kalai2008agnostically}.
            The proof of \cref{lem:constReg:expectedValue} follows the ideas in the analysis of \lreg{}, but needs to account for some technicalities:
            First, unlike \lreg{} which selects the best threshold in $[0,1]$ (\cref{alg:l1reg}), \cref{alg:constReg} has to select a threshold bounded away from 1 to ensure that the output hypothesis will be feasible for \pERM{} (which is required to establish \cref{result:constReg:feasibility} and, in particular, to prove \cref{lem:constReg:feasibility:unconditional}).
            In the proof, we observe that this only changes the expected value by a small amount.
            Second, the analysis of \lreg{} relies on the polynomial $p_{\opt}$ (which approximates $H_{\opt}$) being feasible for the corresponding optimization problem.
            For \lreg{}, this always holds since the optimization problem solved by \lreg{} is unconstrained (\cref{alg:l1reg}).
            For \cref{alg:constReg}, this is not always the case.
            But, we show that $p_{\opt}$ is feasible with high probability, and carry out our proof conditioned on this event.
            This conditioning does not skew the expectations significantly as the involved random variables are suitably bounded.

        \paragraph{Part 2 (Boosting via Repetition).}
        Next, using a standard argument one can show that selecting the ``best performing'' hypothesis from $H_1,H_2,\dots,H_T$ (as evaluated on a fresh set of samples) has an objective value close to optimal with probability $1-\delta$.
        The proof of \cref{lem:constReg:boost} is standard and is provided in \cref{sec:proofof:lem:constReg:boost} for completeness.
        \begin{lemma}[Boosting via Repetition]\label{lem:constReg:boost}
            The hypothesis $H$ output by \cref{alg:constReg}, with probability $1-\delta$, satisfies that 
            \[
                \cD(H)
                \leq 
                    \cD(H_{\opt})
                    +  
                    O\!\inparen{\sqrt{\frac{C+1}{\alpha\sigma}}}
                    \cdot \zeta^{1/(4q)}
                    +
                    O(\delta)
                \,.
            \]
        \end{lemma}

        \subsubsection{Proof of \cref{lem:constReg:expectedValue} {(Expected Value of Optimal)}}\label{sec:proofof:lem:constReg:expectedValue}
            In this section, we prove \cref{lem:constReg:expectedValue}, which we restate below.
            \constrainedRegExpValue*
            \noindent Fix any $1\leq i\leq T$.
            We divide the proof of \cref{lem:constReg:expectedValue} into the following lemmas.
            \begin{lemma}\label{lem:constReg:expectedValue:upperBound} 
                It holds that 
                $\nfrac{\abs{U_i\cap H_i}}{\abs{U_i}}\leq \frac{1}{1-\sqrt{\rho}} \cdot \sum_{x\in U_i}  \nfrac{\abs{p_i(x)}}{\abs{U_i}},$ where $p_i(\cdot)$ is the polynomial constructed in the $i$-th iteration of Subroutine B in \cref{alg:constReg}.
            \end{lemma}
            Recall that $p_{\opt}$ is a polynomial satisfying 
            \[
                \Ex_{{\cD}}\abs{p_{\opt}(x) - \mathds{1}\sinbrace{x\in H_{\opt}}}\leq \zeta\,.
                \yesnum\label{eq:lem:constReg:pstarDef}
            \]
            Let $\evF$ be the event that $p_{\opt}(\cdot)$ is feasible for Program~\eqref{eq:constReg:prog} in all iterations $1\leq i\leq T$.  
        \begin{lemma}\label{lem:constReg:expectedValue:feasibility}  
            It holds that $\Pr[\evF]\geq 1-\delta$. 
        \end{lemma}
        \mbox{Next, we show, conditioned on $\evF$, the expected value of $\inparen{\nfrac{1}{\abs{U_i}}}\sum_{x\in U_i} \abs{p_i(x)}$ is at most $\cD(H_{\opt})+\zeta$.}
        \begin{lemma}\label{lem:constReg:expectedValue:ub2}  
            $\Ex_{U_i\sim \cD}\insquare{
                \inparen{\nfrac{1}{\abs{U_i}}}\sum_{x\in U_i} \abs{p_i(x)}\mid \evF
            } \leq \cD(H_{\opt}) + \zeta.$
        \end{lemma} 
 
        \noindent 
        Now, we are ready to complete the proof of \cref{lem:constReg:expectedValue}.
            \begin{proof}[Proof of \cref{lem:constReg:expectedValue}]
            Since $\nfrac{\abs{U_i\cap H_i}}{\abs{U_i}}\in [0,1]$, \cref{lem:constReg:expectedValue:feasibility} implies that 
            \[
                \Ex\insquare{
                    \frac{\abs{U_i\cap H_i}}{\abs{U_i}}
                }
                - \Ex\insquare{
                    \frac{\abs{U_i\cap H_i}}{\abs{U_i}}\given \evF
                }
                \leq \delta\,.
            \]
            Further, \cref{lem:constReg:expectedValue:upperBound,lem:constReg:expectedValue:ub2} imply that 
            \[
                \Ex\insquare{
                    \frac{\abs{U_i\cap H_i}}{\abs{U_i}}\given \evF
                }
                \leq \frac{1}{1-\sqrt{\rho}}\Ex_{U_i\sim \cD,~ P_i\sim \cPtrue}\insquare{
                    \inparen{
                        \frac{\sum_{x\in U_i} \abs{p_i(x)}}{\abs{U_i}}}
                    \given \evF
                }
                \leq
                \frac{1}{1-\sqrt{\rho}}
                \cdot \inparen{
                    \cD(H_{\opt})
                    + \zeta
                }\,.
            \]
            Further, since $\rho=\Theta\!\inparen{\nfrac{(C+1)}{\alpha\sigma} \cdot \zeta^{1/(2q)}}\leq \nfrac{1}{4}$ and $\nfrac{1}{(1-z)}\leq 1+2z$ for all $z\in [0,\nfrac{1}{2}]$, 
            \[
                \Ex\insquare{
                    \frac{\abs{U_i\cap H_i}}{\abs{U_i}}\given \evF
                }
                \leq
                \inparen{
                    \cD(H_{\opt})
                    + \zeta
                }
                + 2\sqrt{\rho}\inparen{
                    \cD(H_{\opt})
                    + \zeta
                }\,.
            \]
            Finally, since $\cD(H_{\opt})\leq 1$, $\zeta\leq 1$, and $\nfrac{(C+1)}{\alpha\sigma}\geq 1$, we have $\zeta \leq \sqrt{\nfrac{(C+1)}{\alpha\sigma}}\cdot \zeta^{1/(4q)}$ and
                $\sqrt{\rho}
                =
                O\!\inparen{
                    \sqrt{\nfrac{(C+1)}{\alpha\sigma}}\cdot \zeta^{1/(4q)}
                }.$
            Therefore,
            \[
                \Ex\insquare{
                    \frac{\abs{U_i\cap H_i}}{\abs{U_i}}\given \evF
                }
                \leq
                \cD(H_{\opt})
                +
                O\!\inparen{\sqrt{\frac{C+1}{\alpha\sigma}}}\cdot \zeta^{1/(4q)}\,.
            \]
            Chaining the above inequalities implies \cref{lem:constReg:expectedValue}.
        \end{proof} 
        In the remainder of this section, we prove \cref{lem:constReg:expectedValue:upperBound,lem:constReg:expectedValue:feasibility,lem:constReg:expectedValue:ub2}. %

        \begin{proof}[Proof of \cref{lem:constReg:expectedValue:upperBound}]
            Recall that $H_i=\mathds{1}\sinbrace{p_i(x)\geq t_i}$.
            Since \cref{alg:constReg} selects the threshold $t_i\in [0, 1-\sqrt{\rho}]$, that minimizes $\nfrac{\abs{U_i\cap H_i}}{\abs{U_i}}$, it follows that 
            \[
                \frac{\abs{U_i\cap H_i}}{\abs{U_i}}
                \leq \Ex_{t\sim {\rm Unif}[0, 1-\sqrt{\rho}]} 
                \frac{
                    \sum_{x\in U_i} \mathbb{I}\!\insquare{0\leq t\leq p_i(x)}
                }{\abs{U_i}}
                \leq  \frac{1}{{1 - \sqrt{\rho}}}\cdot 
                \frac{
                    \sum_{x\in U_i} \abs{p_i(x)}
                }{\abs{U_i}}
                \,.
                \yesnum\label{eq:constRegAnalysis1}
            \]
        \end{proof}

\begin{proof}[Proof of \cref{lem:constReg:expectedValue:feasibility}]
    Let $T$ denote the number of repetitions in \cref{alg:constReg}.
    Fix any repetition $i$, and write $P= P_i$ (be the draw of the positive samples in this iteration).
    Define
    \[
        Y(x)\coloneqq \min\!\inbrace{p_{\opt}(x),1}\,.
    \]
    By \eqref{asmp:poptLowerBound}, $Y(x)\in[-C,1]$ for all $x$.
    Fix any $\eta\in (0,1)$.
    Since $\Ex_{\cD}\abs{\mathds{1}\sinbrace{x\in H_{\opt}} - p_{\opt}(x)}\leq \zeta$,
    \[
        \Pr_{x\sim \cD}\insquare{
            p_{\opt}(x) < 1-\eta
            ~~\text{and}~~
            x\in H_{\opt}
        }
        \leq \frac{\zeta}{\eta}\,.
    \]
    Therefore,
    \[
        \Pr_{x\sim \cD}\insquare{
            p_{\opt}(x) < 1-\eta
            ~~\text{and}~~
            x\in H_{\opt}\cap \optset
        }
        \leq \frac{\zeta}{\eta}\,,
    \]
    and now applying \cref{asmp:smoothness} implies that  
    \[
        \Pr_{x\sim \cDtrue}\insquare{
            p_{\opt}(x) < 1-\eta
            ~~\text{and}~~
            x\in H_{\opt} \cap \optset
        }
        \leq \frac{1}{\sigma}\cdot \inparen{\frac{\zeta}{\eta}}^{\sfrac{1}{q}}
        \,.  
    \]
    Since $\cPtrue(H_{\opt})=1$ (by construction; see \eqref{eq:constReg:optimalSolution}), we also know that
    \[
        \Pr_{x\sim \cDtrue}\insquare{
            p_{\opt}(x) < 1-\eta
            ~~\text{and}~~
            x\in \optset\setminus H_{\opt} 
        }
        ~~\leq ~~
        \cDtrue\!\inparen{\optset\setminus H_{\opt}} 
        \quad~~~\Stackrel{\cPtrue(H_{\opt})=1}{=}\quad~~~ 0\,.
    \]
    Combining the above two yields 
    \[
        \Pr_{x\sim \cDtrue}\insquare{
            p_{\opt}(x) < 1-\eta
            ~~\text{and}~~
            x\in \optset
        }
        \leq \frac{1}{\sigma}\cdot \inparen{\frac{\zeta}{\eta}}^{\sfrac{1}{q}}\,.
    \]
    In other words,
    \[
        \Pr_{x\sim \cDtrue}\insquare{
            p_{\opt}(x) < 1-\eta
            \mid x\in \optset
        }\cdot \cDtrue(\optset)
        \leq \frac{1}{\sigma}\cdot \inparen{\frac{\zeta}{\eta}}^{\sfrac{1}{q}}
        \,.  
    \]
    Since $\cDtrue(\optset)\geq \alpha$ and conditioning on $\optset$ is exactly $\cPtrue$, we obtain
    \[
        \Pr_{x\sim \cPtrue}\insquare{
            p_{\opt}(x) < 1-\eta
        }
        \leq 
        \frac{1}{\alpha\sigma}\cdot \inparen{\frac{\zeta}{\eta}}^{\sfrac{1}{q}}\,.
    \]
    Define
    \[
        \mu\coloneqq \Ex_{x\sim \cPtrue}\insquare{Y(x)}\,.
    \]
    Since $Y(x)\geq 1-\eta$ on the event $\inbrace{p_{\opt}(x)\geq 1-\eta}$ and $Y(x)\geq -C$ always,
    \begin{align*}
        \mu
        &\geq
        (1-\eta)\cdot \Pr\inparen{p_{\opt}(x)\geq 1-\eta}
        +
        (-C)\cdot \Pr\inparen{p_{\opt}(x)<1-\eta}\\
        &=
        1-\eta-\inparen{C+1-\eta}\Pr\inparen{p_{\opt}(x)<1-\eta}\\
        &\geq
        1-\eta-\inparen{C+1}\cdot \frac{1}{\alpha\sigma}\cdot \inparen{\frac{\zeta}{\eta}}^{\sfrac{1}{q}}\,.
    \end{align*}
    Now set $\eta\coloneqq \nfrac{\rho}{4}$.
    Since $\rho=\Theta\!\inparen{\nfrac{(C+1)}{\alpha\sigma}\cdot \zeta^{1/(2q)}}$ and $\nfrac{(C+1)}{\alpha\sigma}\geq 1$, the hidden constant in the definition of $\rho$ can be chosen large enough so that $\frac{C+1}{\alpha\sigma}\cdot \inparen{\frac{4\zeta}{\rho}}^{\sfrac{1}{q}}
        \leq \frac{\rho}{4}.$
    Therefore,
    \[
        \mu \geq 1-\frac{\rho}{4}-\frac{\rho}{4}=1-\frac{\rho}{2}\,.
    \]
    Since $-C\leq Y(x)\leq 1$, Hoeffding's inequality implies that
    \[
        \Pr\insquare{
            \frac{1}{\abs{P}}\sum_{x\in P}Y(x) < 1-\rho
        }
        \leq
        \Pr\insquare{
            \frac{1}{\abs{P}}\sum_{x\in P}Y(x) < \mu-\frac{\rho}{2}
        }
        \leq
        \exp\!\inparen{
            -\frac{\abs{P}\rho^2}{2(C+1)^2}
        }.
    \]
    By the choice of hidden constants in Step~\ref{step:constReg:getSamples}, together with $\rho^{-2}=O\!\inparen{\zeta^{-1/q}}\leq O\!\inparen{\zeta^{-2}}$ for $q\geq 1$ and $\zeta\in(0,1]$, each repetition uses at least
    $\abs{P}
        \geq
        2\rho^{-2}(C+1)^2\log\sinparen{\nfrac{T}{\delta}}$
    positive samples.
    Hence, for each fixed repetition $i$, the above failure probability is at most $\nfrac{\delta}{T}$.
    Taking a union bound over the $T$ repetitions yields $\Pr[\evF]\geq 1-\delta$.
\end{proof} 

\begin{proof}[Proof of \cref{lem:constReg:expectedValue:ub2}]
    Fix $i$ and condition on $\evF$.
    Under $\evF$, the polynomial $p_{\mathrm{opt}}$ is feasible for the constrained regression program in repetition $i$.
    Since $p_i$ is an optimal solution of that program, we have pointwise (for the realized sample $U_i$):
    \[
    \frac{1}{\abs{U_i}}\sum_{x\in U_i}\abs{p_i(x)}
    \leq
    \frac{1}{\abs{U_i}}\sum_{x\in U_i}\abs{p_{\mathrm{opt}}(x)}.
    \]
    Taking expectation conditional on $\evF$ and using that $U_i$ is independent of $\evF$,
    \[
    \E\insquare{
        \frac{1}{\abs{U_i}}\sum_{x\in U_i}\abs{p_{\mathrm{opt}}(x)}
        \ \Big|\ \evF
    }
    =
    \E_{x\sim \cD}\insquare{\abs{p_{\mathrm{opt}}(x)}}.
    \]
    Finally, by the triangle inequality,
    \[
    \E_{x\sim \cD}\insquare{\abs{p_{\mathrm{opt}}(x)}}
    \leq
    \E_{x\sim \cD}\insquare{\abs{p_{\mathrm{opt}}(x)-\mathds{1}\!\inbrace{x\in H_{\mathrm{opt}}}}}
    +
    \E_{x\sim \cD}\insquare{\mathds{1}\!\inbrace{x\in H_{\mathrm{opt}}}}
    \leq
    \zeta+\cD\!\inparen{H_{\mathrm{opt}}},
    \]
    which proves the claim.
\end{proof}

\section{Applications to Other Learning Theory Settings}\label{sec:application}
    In this section, we present the formal statements and proofs of the results from \cref{sec:intro:application}.

    \subsection{Truncated Estimation with Unknown Survival Set}  
    \label{sec:application:unknownTruncation}
    In this section, we present the formal statement and proof of \cref{infcor:truncation}, which gives an improved algorithm from truncated estimation with an unknown survival set.

    The most general result for this problem is by \citet{lee2024unknown} and works in the following setting. The unlabeled distribution comes from an exponential family with polynomial sufficient statistics, which is a parametric family of the form
    \[ \cE(x; \theta) \propto h(x) \cdot e^{ \inangle{\theta, t(x)}}\,, \]
    where $t(x)$ is an $m$-dimensional vector whose components are polynomials in $x$.
    The set of parameters $\theta$ for which the density of $\cE(\theta)$ is well defined is called the \textit{natural parameter space} $\overline{\Theta}\subseteq\R^m$.
    Suppose that the ``true'' distribution $\cDtrue$ is described by an unknown target parameter $\theta^\star$ in $\overline{\Theta}$, \ie{}, $\cDtrue=\cE(\theta^\star)$. 
    Given samples from $\cE(\theta^\star)$ truncated to unknown survival set $S^\star \in \hyS$ that is guaranteed to have mass at least $\alpha>0$ under $\cE(\theta^\star)$, \ie{}, $\cE(\theta^\star)(S^\star)\geq \alpha$, the goal is to recover an estimate of $\theta^\star$.
    We begin by describing the assumptions required by existing algorithms \cite{lee2024unknown}. 
    \begin{assumption}\label{asmp:exponentialFamily}
        There is a subset $\Theta$ of the natural parameter space and $\Lambda\geq \lambda > 0$ and $\eta>0$ such that:
        \begin{enumerate}[itemsep=0pt]
            \item \textbf{(Interiority)}           
                $\theta^\star$ is in the $\eta$-relative-interior $\Theta(\eta)$ of $\Theta$.
                \label{unknowntrunc:int}
            \item \textbf{(Projection Oracle)} There is an efficient projection oracle to the $\eta$-relative-interior of $\Theta$.
            \label{unknowntrunc:proj}
            \item \textbf{(Bounded Fisher-Information)} For any $\theta\in \Theta$,
            \label{unknowntrunc:fisher}
            $\lambda I \preceq \cov_{x\sim \cE(\theta)}[t(x)]\preceq \Lambda I$.
            \item \textit{\textbf{(Starting Point)} There is an algorithm to find $\theta_0\in \Theta(\eta)$, s.t., $\norm{\theta^\star-\theta_0}\leq \poly(\nfrac{1}{\alpha})$ efficiently.}
            \label{unknowntrunc:st}
        \end{enumerate}
    \end{assumption}
    In addition to the above assumptions, \citet{lee2024unknown} also require two additional assumptions, which we are able to relax.
    \begin{enumerate}[start=5,itemsep=0pt]
            \item \textit{\textbf{($\chi^2$-Bridge)} For each $\theta_1,\theta_2\in \Theta$ there is a $\theta$, s.t., $\chidiv{\cE(\theta_1)}{\cE(\theta)}, 
            \chidiv{\cE(\theta_2)}{\cE(\theta)}
            \leq  
            e^{\poly\!\inparen{\sfrac{\Lambda}{\lambda}}}$.}
            \label{unknowntrunc:chisquarebridge}
            \item \mbox{\textit{\textbf{($L_2$-Approximable)}} $\hyS$ is $\poly(\eps)$-approximable by degree-$k$ polynomials in $L_2$-norm w.r.t. $\cE$.}
    \end{enumerate}
    They require these additional assumptions to use \lreg{} as a black box, by showing that $L_2$-approximability with respect to the family $\cE$ gives $L_1$-approximability with respect to a mixture of a distribution from $\cE$ and the truncation of $\cE(\theta^\star)$ to $S^\star$. 
    We are able to avoid them by using the algorithm from our main result on smooth positive-only learning (\cref{thm:main}).
    
    \begin{corollary}[Truncated Estimation with Unknown Survival Set]\label{cor:truncation}
        Suppose \cref{asmp:exponentialFamily} (\cref{unknowntrunc:int,unknowntrunc:proj,unknowntrunc:fisher,unknowntrunc:st}) holds
        and one has an efficient sampling oracle for $\cE$.\footnote{By this we mean that given any $\theta \in \Theta$, one can draw an independent sample $x$ from a distribution $\cQ$ satisfying $\tv{\cQ}{\cE(\theta)}\leq \delta$ in $\poly(\nfrac{md}{\delta})$ time.} 
        Fix $\eps, \delta \in (0, \nfrac{1}{2})$ and $k$ such that degree-$k$ polynomials with range $[-C,\infty)$ $\eps^{D}$-approximate $S^\star$ with respect to family $\cE$ in $L_1$-norm for $D=\poly(\nfrac{\Lambda}{\alpha \eta \lambda})$
        and $n = \poly(d^k, (\nfrac{\eps}{C})^{-D}, \log \nfrac{1}{\delta})$. There is an algorithm that given $\eps, \delta$, constants $(\lambda,\Lambda,\deg(t(\cdot)),\eta,\alpha)$, and $n$ independent samples from $\cE(\theta^\star)$ truncated to $S^\star$, in $\poly(n)$ time, outputs an estimate $\theta$ such that with probability $1-\delta$,
        \[
            \tv{\cE(\theta)}{\cE(\theta^\star)} \leq \eps\,.
        \]
    \end{corollary}
    See \cref{sec:intro:application:truncation} for a discussion of the implications of the above result and \cref{tab:unknown_trunc} for a comparison of our results with existing ones.
    
    \begin{table}[ht!]
        \centering
        \small 
        \begin{tabular}{p{2.75cm} | l |  C{1.5cm} | C{2.3cm} | C{1.7cm}}
             Distribution 
                    & Hypothesis Class & \cite{Kontonis2019EfficientTS} & \cite{lee2024unknown} & \mbox{This Work}  \\
             \hline
             \multirow{2}{*}{Gaussian $\cN(\mu,\Sigma)$} & Finite GSA & \mbox{\hspace{-2mm}\textit{Diagonal $\Sigma$}\hspace{-2mm}}  & \tick{} &\tick{} \\
             & $L_1$-approximable by non-negativepolynomials & \cross{} & \cross{} & \tick{} \\ 
             \hline
             \multirow{3}{3.5cm}{Log-concave exponential family satisfying \cref{asmp:exponentialFamily}} 
                & $L_2$-approximable by polynomials %
                    & \cross{} 
                        & \textit{\mbox{\hspace{-2mm}Only if $\chi^2$-Bridge\hspace{-2mm}} exists;} see \eqref{unknowntrunc:chisquarebridge}
                            & \tick{} \\ 
              & Arbitrary functions of $O(1)$ halfspaces & \cross{} & \cross{} & \tick{} \\ 
              & $L_1$-approximable by non-negativepolynomials 
                & \cross{}  & \cross{} & \tick{} \\  
        \end{tabular}
        \caption{
        \textit{Results for Efficient Estimation from Samples Truncated to Unknown Survival Set.} 
            The table includes a \tick{} if there exists an efficient algorithm in the corresponding setting, in the sense that the algorithm runs in {$d^{\poly(1/\eps)}$ time,} and \cross{} otherwise.
        }
        \label{tab:unknown_trunc}
    \end{table}

    \newcommand{\customSeparator}{$\phantom{\frac{\frac{1}{1}}{2}}$}
    \begin{table}[tbh]
        \centering
        \midsepremove{}
        \small 
        \begin{tabular}{ C{1.6cm} | C{3.1cm}| C{3.1cm}|C{5.5cm}|C{2.1cm}}
             & Distribution & Access & Hypothesis Class & {\mbox{Computational}} Complexity \\[2mm]
             \midrule 
             \cite{de2023detectingConvex}  & $\cN(0, I)$ & Known & Convex Sets \customSeparator{} & $\smash{O({d}/\alpha^2)}$ \\[2mm]
             \hline
             \cite{de2024detecting} & Hypercontractive \mbox{and \iid{} Product \textsuperscript{($\star$)}} & {\mbox{\textit{Exact} orthonormal\textsuperscript{($\star$)\hspace{-3mm}}} \mbox{basis is computable}} & Degree-$k$ PTFs \customSeparator{}
              & $\smash{O(d^{k/2}/\alpha^2)}$  \\[2mm]
             \hline
             {\multirow{8}{2cm}{This Work}} 
                &\multirow{3}{*}{Gaussian}  
                & \multicolumn{1}{c|}{\multirow{3}{2.8cm}{Approximate sample access\textsuperscript{($\vartriangle$)}}} 
                    &  Convex Sets  
                        & $\smash{{d^{\sqrt{d}\cdot \poly(\sfrac{1}{\alpha})}}}$\\[2mm]
             && 
                & Degree-$k$ PTF \customSeparator{} & 
                    $\smash{d^{k\cdot \poly(\sfrac{1}{\alpha})}}$\\[2mm]
             && 
                & 
                    Intersection of $t$ Halfspaces \customSeparator{} & 
                        \hspace{-2mm}$\smash{d^{(\log t)\cdot \poly(\sfrac{1}{\alpha})}}$\hspace{-6mm} \\[2mm]
            \cline{2-5}
             & \multirow{3}{3cm}{Log-Concave and $1$-Strictly Sub- Exponential\textsuperscript{($\dagger$)}} & \multicolumn{1}{c|}{\multirow{3}{2.8cm}{Approximate sample access \textsuperscript{($\vartriangle$)}}}  & 
                    \customSeparator{}Functions of $O(1)$ halfspaces & 
                        \customSeparator{}\vspace{-4mm} $d^{\poly(1/\alpha)}$ \\[4mm]
                && & 
                    $L_1$-approximable by degree-$k$ non-negativepolynomials with $\eps=\poly(\alpha)$& 
                        $d^{k\cdot \poly\!\inparen{1/\alpha}}$ \\[3mm]
             \cline{2-5}
             & \multirow{3}{2.8cm}{\mbox{Uniform on $\zo^d$}} & \multicolumn{1}{c|}{\multirow{3}{2.8cm}{Approximate sample access \textsuperscript{($\vartriangle$)}}}  & 
                    Decision trees of size $s$ & 
                       \vspace{-3mm} $d^{(\log s) \cdot \poly(\sfrac{1}{\alpha})}$ \\[2mm]
                && & 
                    Intersection of $t$ Halfspaces & 
                        $d^{t^4\cdot \poly(\sfrac{1}{\alpha})}$ \\[2mm]
                && & 
                    Decision tree of halfspaces of size $s$, depth $t$ &
                        $d^{s^2t^4\cdot\poly(\sfrac{1}{\alpha})}$ \\
        \end{tabular}
        \midsepdefault{}
        \caption{\textit{Computational Complexity of Detecting Truncation} 
        \newline
        \vrule height 0.2pt width 15cm\newline
            {\footnotesize \white{.}\quad  ($\star$)~~ The requirements for the distribution in \cite{de2024detecting} are specified in \cref{asmp:detectingTruncation:hypercontractive,asmp:detectingTruncation:product} and access in \cref{asmp:detectingTruncation:access}.}\\[-1mm]
            {\footnotesize \white{.}\quad ($\dagger$)~~ Here, while we do not make explicit assumptions on the distribution, we require approximability by}\\[-1mm]
            {\footnotesize \white{.}\quad ~~~~~~~polynomials with respect to the distribution, which is typically only available for log-concave distributions or}\\[-1mm]
            {\footnotesize \white{.}\quad ~~~~~~~ the uniform distribution on $\zo^d$.}\\[-1mm]
            {\footnotesize \white{.}\quad  ($\vartriangle$)~~ One has sample access to $\cR$ that is close to $\cQ$ as explained in \cref{cor:detectingTruncation:gen}.
            }
            \label{tab:detect}
        } 
        \vspace{-5mm}
    \end{table}

    \begin{proof}[Proof of \cref{cor:truncation}]
        This result follows by combining \cref{thm:main} with the reduction in Corollary 6.13 in the arXiv version of \cite{lee2024unknown}.
        The reduction converts the estimation problem in \cref{cor:truncation} to the following learning problem in $\poly(\nfrac{d}{(\eps\alpha)})$ time:
            given samples from a distribution $\cD=\cE(\theta_0)$ that is close to $\cDtrue=\cE(\theta^\star)$ in a sense stronger than that of \cref{asmp:smoothness} with $\nfrac{1}{\sigma}=q=D$ for $D=\poly\!\inparen{\nfrac{\Lambda}{(\alpha\eta\lambda)}}$, find a set $S$ such that $\cDtrue(S\triangle S^\star)\leq \eps^{O(D)}$.
        This precisely fits into the framework of smooth positive-only learning with $\optset=S^\star$ and, moreover, since in the model of \cite{lee2024unknown}, $\cDtrue(S^\star)\geq \alpha$ which ensures \cref{asmp:posFraction} allowing us to apply \cref{thm:main}.
    \end{proof}
    \subsection{Detecting Truncation}  
    \label{sec:application:detectingTruncation} 
        In this section, we present the formal statement and proof of \cref{infcor:detectingTruncation}. 
        The following two results formalize \cref{infcor:detectingTruncation} in the cases where we have samples from the underlying distribution $\cQ$ or only from a distribution $\cR$ close to it, respectively.
        \begin{corollary}[Detecting Truncation with Exact Samples]\label{cor:detectingTruncation}
            Fix constants $\beta\in (0,1)$ and $C>0$, set $S^\star\in \hyS \subseteq \R^d$, and distribution $\cQ$ over $\R^d$, such that, $\cQ(S^\star)\leq 1-\beta$.
            The following holds. 
                Let $k$ be such that degree-$k$ polynomials with range $[-C,\infty)$ $\zeta$-approximate $\hyS$ in $L_1$-norm with respect to $\cQ$ for $\zeta\leq (\beta\eps/(C+1))^{O(1)}$.
                There is an algorithm that given $\beta,\delta,k$ and $n=\poly(d^k,\nfrac{1}{\zeta},\log{\nfrac{1}{\delta}})$ samples (a) $\cD\in\inbrace{\cQ,\cQ(S^\star)}$ and (b) $\cQ$, with probability $1-\delta$, correctly detects whether $\cD=\cQ$ or $\cQ=\cQ(S^\star)$.
                    The algorithm runs in $\poly(n)$ time.
        \end{corollary}
        \vspace{-6mm}
        
        \noindent The next corollary strictly generalizes \cref{cor:detectingTruncation}.
        \vspace{-2mm}
        \begin{corollary}[Detecting Truncation with Approximate Samples]
            \label{cor:detectingTruncation:gen}
            Consider the setting in \cref{cor:detectingTruncation}.
            Let $\cR$ be a distribution such that $\cQ$ is $(\sigma,q)$-smooth with respect to $\cR$ (\ie{}, for any measurable set $S$, $\cQ(S)\leq (\nfrac{1}{\sigma})\cdot \cR(S)^{1/q}$).
                Let $k$ be such that degree-$k$ polynomials  with range $[-C,\infty)$ $\zeta$-approximate $\hyS$ in $L_1$-norm with respect to $\cQ$ for $\zeta\leq (\nfrac{\sigma\alpha\beta\eps}{C})^{O(q^2)}$.
                There is an algorithm that given $\beta,\delta,k$ and $n$ samples (a) $\cD\in\inbrace{\cQ,\cQ(S^\star)}$ and (b) $\cR$, with probability $1-\delta$, correctly detects whether $\cD=\cQ$ or $\cQ=\cQ(S^\star)$.
                The algorithm runs in $\poly(n)$ time.
        \end{corollary}
        See \cref{tab:detect} for a summary of results and prior work.

    \begin{proof}[Proof of \cref{cor:detectingTruncation}]
        \cref{cor:detectingTruncation} follows from \cref{cor:detectingTruncation:gen} by substituting $\sigma=q=1$.
    \end{proof}
    \vspace{-6mm}
    \begin{proof}[Proof of \cref{cor:detectingTruncation:gen}]
        In this application, the samples are observed from the ``truncated'' distribution $\cQ(T)$ where $T=\R^d$ if there is no truncation and $T=S^\star$ if there is a truncation.
        Further, since $\cQ(S^\star)\leq 1-\beta$.
        To distinguish the two cases, it suffices to learn $T$ up to $O(\beta)$ accuracy:
        this is because, for any set $T'$ satisfying $\cQ(T\triangle T') < \nfrac{\beta}{2}$, if $\cQ(T') > 1-\nfrac{\beta}{2}$, then it must be that $\cQ(T) > 1-\beta$ which is only possible if $T=\R^d$ (since $\cQ(S^\star)\leq 1-\beta$).
        Now, the result follows from the sample and time complexity of learning in \cref{thm:main}.
    \end{proof}

\subsection{Learning with Positive Samples and Lists of Reference Samples}\label{sec:application:listDecoding}
        The following result formalizes \cref{infthm:list}.
        \begin{theorem}\label{thm:listDecoding}
            Consider the List Decoding Model (\cref{def:list}) and $\eps,\delta\in (0,\nfrac{1}{2})$.
            \begin{itemize}[leftmargin=20pt]
                \item[$\triangleright$] (Sample Complexity)~~ There is an algorithm that, given $\eps$ and 
                $\tilde{O}\!\inparen{
                        \ell\cdot \sinparen{\nfrac{\ell}{\eps\sigma}}^{2q}
                        \cdot
                        \sinparen{\vc(\hyH)+\log \nfrac{\ell}{\delta}}
                    }$ independent samples from $\cDtrue(H^\star)$ and $\cD_1,\cD_2,\dots,\cD_\ell$, outputs a hypothesis $\wh{h}$, such that, with probability $1-\delta$,
                \[
                    \Pr\nolimits_{\cDtrue}\binparen{\wh{h}(x) \neq h^\star(x)} \leq \eps\,.
                \]
                \item[$\triangleright$] (Computational Complexity)~~ 
                    Suppose $\cDtrue(H^\star) \geq \alpha$. 
                    Let $k \geq 1$ be such that $\hyH$ is $\zeta$-approximable by degree-$k$ polynomials with range $[-C,\infty)$ in the $L_1$-norm with respect to $\cD_{i^\star}$ for the index $i^\star$ in \cref{def:list} and $\zeta\leq (\nfrac{\eps\alpha\sigma}{\ell})^{O(q^2)}$. 
                    There is a polynomial-time algorithm that given $n=\poly(d^k, \nfrac{1}{\zeta}, \nfrac{\ell}{\eps}, \log \nfrac{1}{\delta})$ samples from $\cDtrue(H^\star)$ and $\cD_1,\cD_2,\dots,\cD_\ell$, outputs a hypothesis $\wh{h}$ such that, with probability $1-\delta$,
                    \[
                        \Pr\nolimits_{\cDtrue}\binparen{\wh{h}(x) \neq h^\star(x)} \leq \eps\,.
                    \]
                    The algorithm runs in $\poly(n)$ time.
            \end{itemize} 
        \end{theorem} 
        Thus, efficient learning is possible under the List Decoding Model.
        See \cref{sec:intro:application:listDecoding} for a discussion of the above results. %
        We stress that the list-decoding model is very general and, in fact, also allows for some cases where 99\% of the unlabeled examples have been \textit{adversarially} corrupted (see \cref{sec:advCorruption}).
            In the remainder of this section, we prove \cref{thm:listDecoding}.
            \begin{proof}[Proof of \cref{thm:listDecoding}]
                First, we prove the computational complexity result. 
                The proof of the sample complexity result is analogous and is presented at the end.

                \paragraph{Computational Complexity.}
                    The algorithm in \cref{thm:main} is not directly applicable in the List Decoding Model because we do not know the index $i^\star$ of the distribution that is ``close'' to $\cDtrue$.
                    Nevertheless, the following simple meta-algorithm suffices.

            \begin{algorithm}[ht!]
            \caption{Efficient Algorithm to Learn from Positive Samples and a List of Unlabeled Distributions}\label{alg:listDecoding}
            \begin{algorithmic}
                \Require Sample access to $\cPtrue=\cDtrue(\optset)$ and each of $\cD_1,\cD_2,\dots,\cD_\ell$; parameters $\eps,\delta,\alpha,q,\sigma,k,\ell$
                \Ensure A hypothesis $\wh{H}$
                \State Let $\eta \gets \nfrac{\alpha\eps}{16\ell}$ and $T \gets  (\sigma\eta)^{-q} $
                \State Let $P$ be a set of $m=\wt{O}\!\inparen{\nfrac{\ell}{\eps}\cdot \inparen{T\, d^k+\log{\nfrac{\ell}{\delta}}}}$ independent samples from $\cPtrue$
                \State Initialize $S$ as the empty set $\emptyset$
                \For{$1\leq i \leq \ell$}
                    \State Obtain $H^{(i)}$ from \cref{alg:main} with positive sample distribution $\cPtrue$, unlabeled distribution $\cD_i$, accuracy $\eta$, confidence $\nfrac{\delta}{2\ell}$, smoothness parameters $(\sigma,q)$, lower bound on mass $\alpha$, and number $k$
                    \State \textbf{If} $\nfrac{\abs{P\setminus H^{(i)}}}{\abs{P}}\leq \nfrac{\eps}{8\ell}$, \textbf{then} add index $i$ to $S$
                \EndFor
                \State \Return $\wh{H}\coloneqq \bigcap_{i\in S} H^{(i)}$
            \end{algorithmic}
            \end{algorithm}
            \noindent \textit{Computational and Sample Complexity.}
            Define
                    \[
                        \eta \coloneqq \frac{\alpha\eps}{16\ell}
                        \qquadand
                        T \coloneqq (\sigma\eta)^{-q}\,.
                    \]
                    Each call to \cref{alg:main} uses $n_0=\poly\!\inparen{d^k,\nfrac{1}{\zeta},\nfrac{\ell}{\eps},\log \nfrac{\ell}{\delta}}$
                    samples and runs in $\poly(n_0)$ time, and the additional holdout sample $P$ used by \cref{alg:listDecoding} has size $\wt{O}\!\inparen{(\nfrac{\ell}{\eps})\cdot \inparen{Td^k+\log \frac{\ell}{\delta}}}$
                    which is polynomial in the stated parameters.
                    Hence the sample complexity and running time claimed in the theorem follow once we prove correctness.

                \noindent \textit{Correctness.}
                    Let $\hyP_{\cap T}(k)$ denote the class of intersections of at most $T$ degree-$k$ PTFs, and let
                    \[
                        \hyF_T \coloneqq \inbrace{\optset\setminus H \colon H\in \hyP_{\cap T}(k)}\,.
                    \]
                    Since taking complements and intersecting with a fixed set do not increase VC dimension, standard VC bounds give
                    \[
                        \vc(\hyF_T)=\vc(\hyP_{\cap T}(k))=\wt{O}(Td^k)\,.
                    \]
                    Therefore, by uniform convergence, for the chosen size of $P$, with probability at least $1-\delta/2$, the following two implications hold simultaneously for every $H\in \hyP_{\cap T}(k)$:
                    \begin{align}
                        \frac{\abs{P\setminus H}}{\abs{P}}\leq \frac{\eps}{8\ell}
                        \implies
                        \cPtrue(\optset\setminus H)\leq \frac{\eps}{4\ell}
                        \label{eq:listDecoding:comp:uc}
                        \quadand
                        \cPtrue(\optset\setminus H)\leq \frac{\eps}{16\ell}
                        &\implies
                        \frac{\abs{P\setminus H}}{\abs{P}}\leq \frac{\eps}{8\ell}.
                    \end{align}
                    Moreover, by construction, each hypothesis $H^{(i)}$ returned by \cref{alg:main} belongs to $\hyP_{\cap T}(k)$.

                    We now verify that, with probability at least $1-\delta$, the hypotheses $H^{(1)},H^{(2)},\dots,H^{(\ell)}$ satisfy the following properties:
                    \begin{enumerate}[label=(P\arabic*),itemsep=0pt]
                        \item For every $i\in S$, $\cDtrue(\optset\setminus H^{(i)})\leq \nfrac{\eps}{4\ell}$.
                        \label{property:listDecoding:proof:1}
                        \item $\cDtrue(\optset\triangle H^{(i^\star)})\leq \nfrac{\alpha\eps}{16\ell}$.
                        \label{property:listDecoding:proof:2}
                        \item $i^\star\in S$.
                        \label{property:listDecoding:proof:3}
                    \end{enumerate}
                    Property~\ref{property:listDecoding:proof:2} follows from \cref{thm:main} applied to the call with reference distribution $\cD_{i^\star}$, since by definition of the list-decoding model this call satisfies the assumptions in \cref{thm:main}, and its internal accuracy parameter is exactly $\eta=\nfrac{\alpha\eps}{(16\ell)}$.
                    To prove Property~\ref{property:listDecoding:proof:1}, let $i\in S$.
                    By the acceptance rule in \cref{alg:listDecoding},
                    $\nfrac{\sabs{P\setminus H^{(i)}}}{\abs{P}}\leq \nfrac{\eps}{8\ell}.$
                    Hence \eqref{eq:listDecoding:comp:uc} gives
                    $\cPtrue(\optset\setminus H^{(i)})\leq \nfrac{\eps}{4\ell}.$
                    Since $\cDtrue(R)\leq \cPtrue(R)$ for every $R\subseteq \optset$, we obtain
                    $\cDtrue(\optset\setminus H^{(i)})\leq \nfrac{\eps}{4\ell}.$
                    To prove Property~\ref{property:listDecoding:proof:3}, note that Property~\ref{property:listDecoding:proof:2} implies
                    \[
                        \cPtrue(\optset\setminus H^{(i^\star)})
                        = \frac{\cDtrue(\optset\setminus H^{(i^\star)})}{\cDtrue(\optset)}
                        \leq \frac{\alpha\eps/(16\ell)}{\alpha}
                        = \frac{\eps}{16\ell}.
                    \]
                    Therefore \eqref{eq:listDecoding:comp:uc} yields
                    $\nfrac{\sabs{P\setminus H^{(i^\star)}}}{\abs{P}}\leq \nfrac{\eps}{8\ell},$
                    and hence $i^\star\in S$.

                    We now complete the proof of correctness.
                    Condition on Properties~\ref{property:listDecoding:proof:1}, \ref{property:listDecoding:proof:2}, and \ref{property:listDecoding:proof:3}.
                    Since $i^\star\in S$, the output $\wh{H}=\bigcap_{i\in S} H^{(i)}$ satisfies
                    \[
                        \cDtrue\!\inparen{\optset \triangle \wh{H}}
                        =
                        \underbrace{\cDtrue\!\inparen{\optset \setminus \bigcap\nolimits_{i\in S} H^{(i)}}}_{\text{Term 1: False-Negative Rate}}
                        +
                        \underbrace{\cDtrue\!\inparen{\bigcap\nolimits_{i\in S} H^{(i)} \setminus \optset}}_{\text{Term 2: False-Positive Rate}}\,.
                        \yesnum\label{eq:listDecoding:1}
                    \]
                    These terms can be bounded as follows:
                    \begin{align*}
                        \cDtrue\!\inparen{\optset \setminus \bigcap\nolimits_{i\in S} H^{(i)}}
                        &\leq \sum_{i\in S}\cDtrue(\optset\setminus H^{(i)})
                        ~~\Stackrel{{\rm \ref{property:listDecoding:proof:1}}}{\leq}~~
                        \abs{S}\cdot \frac{\eps}{4\ell}
                        \leq \frac{\eps}{4}\,,
                        \yesnum\label{eq:listDecoding:2}\\
                        \cDtrue\!\inparen{\bigcap\nolimits_{i\in S} H^{(i)} \setminus \optset}
                        &~~\Stackrel{{\rm \ref{property:listDecoding:proof:3}}}{\leq}~~
                        \cDtrue(H^{(i^\star)}\setminus \optset)
                        ~~\Stackrel{{\rm \ref{property:listDecoding:proof:2}}}{\leq}~~
                        \frac{\alpha\eps}{16\ell}
                        \leq \frac{\eps}{16}\,.
                        \yesnum\label{eq:listDecoding:3}
                    \end{align*}
                    Combining \cref{eq:listDecoding:1,eq:listDecoding:2,eq:listDecoding:3} gives, as required, 
                    $\cDtrue(\optset\triangle \wh{H})\leq \nfrac{5\eps}{16}<\eps.$

                \paragraph{Sample Complexity.}
                For the sample-complexity part, we use the same meta-algorithm as \cref{alg:listDecoding}, but replace each call to \cref{alg:main} by a call to \cref{alg:sampleComplexity}.
                Let
                \[
                    \eta' \coloneqq \frac{\eps}{C\ell}
                    \qquadand
                    T' \coloneqq (\sigma\eta')^{-q}\,,
                \]
                where $C>0$ is a sufficiently large absolute constant.
                We keep the same acceptance rule: $\nfrac{\sabs{P\setminus H^{(i)}}}{\abs{P}}\leq \nfrac{\eps}{8\ell}.$
                Each call to \cref{alg:sampleComplexity} returns a hypothesis in the class $\hyH_{\cap T'}$ of intersections of at most $T'$ hypotheses from $\hyH$.
                Standard VC bounds give $\vc(\hyH_{\cap T'})=\wt{O}(T'\vc(\hyH)).$
                Let
                $\hyF'\coloneqq \inbrace{\optset\setminus H \colon H\in \hyH_{\cap T'}}$
                Then $\vc(\hyF')\leq \vc(\hyH_{\cap T'})$.
                We choose the fresh positive holdout sample $P$ to have size
                \[
                    m=\wt{O}\!\inparen{
                        \frac{\ell}{\eps}\cdot
                        \inparen{\vc(\hyF')+\log\frac{\ell}{\delta}}
                    }
                    =
                    \wt{O}\!\inparen{
                        \frac{\ell}{\eps}\cdot
                        T'\cdot
                        \inparen{\vc(\hyH)+\log\frac{\ell}{\delta}}
                    }\,.
                \]
                By \cref{thm:unifConvergence}, for this choice of $m$, with probability at least $1-\nfrac{\delta}{2}$, the following two implications hold simultaneously for every $H\in \hyH_{\cap T'}$:
                \begin{align*}
                    \frac{\abs{P\setminus H}}{\abs{P}}\leq \frac{\eps}{8\ell}
                    \implies
                    \cPtrue(\optset\setminus H)\leq \frac{\eps}{4\ell}
                    \qquadand
                    \cPtrue(\optset\setminus H)\leq \frac{\eps}{16\ell}
                    \implies
                    \frac{\abs{P\setminus H}}{\abs{P}}\leq \frac{\eps}{8\ell}\,.
                \end{align*}
                The $\ell$ calls to \cref{alg:sampleComplexity}, each run with accuracy parameter $\eta'$ and confidence $\nfrac{\delta}{2\ell}$, contribute
                \[
                    \wt{O}\!\inparen{
                        \ell\cdot \inparen{\frac{1}{\sigma\eta'}}^{2q}
                        \cdot
                        \sinparen{\vc(\hyH)+\log\nfrac{\ell}{\delta}}
                    }
                    =
                    \wt{O}\!\inparen{
                        \ell\cdot \inparen{\frac{\ell}{\eps\sigma}}^{2q}
                        \cdot
                        \sinparen{\vc(\hyH)+\log \nfrac{\ell}{\delta}}
                    }
                \]
                positive samples.
                The additional holdout term is
                $\wt{O}\!\inparen{
                        (\nfrac{\ell}{\eps})\cdot
                        T'\cdot
                        \inparen{\vc(\hyH)+\log\nfrac{\ell}{\delta}}
                    },$
                which is dominated by the above quantity because $q\geq 1$ and $T'\geq \nfrac{1}{\eps}$.
                This proves the claimed sample bound.
        
                It remains to prove correctness.
                Let $i^\star$ be the good index from \cref{def:list}.
                The call to \cref{alg:sampleComplexity} with reference distribution $\cD_{i^\star}$ returns a hypothesis $H^{(i^\star)}$ satisfying $\cDtrue(\optset\triangle H^{(i^\star)})\leq \eta'$
                with probability at least $1-\nfrac{\delta}{2\ell}$.
                Moreover, Step~A in the proof of \cref{thm:sampleComplexity} shows that the returned intersection hypothesis also satisfies
                $\cPtrue(\optset\setminus H^{(i^\star)})\leq O\!\inparen{(\sigma\eta')^q}.$
                By choosing the absolute constant $C$ sufficiently large, we ensure that
                $O\!\inparen{(\sigma\eta')^q}\leq \nfrac{\eps}{16\ell}.$
                Therefore the second implication above gives
                $\nfrac{\sabs{P\setminus H^{(i^\star)}}}{\abs{P}}\leq 
                    \nfrac{\eps}{8\ell},$
                and hence $i^\star\in S$.

                Now let $i\in S$.
                By the acceptance rule and the first implication above,
                $\cPtrue(\optset\setminus H^{(i)})\leq \nfrac{\eps}{4\ell}.$
                As before, this implies
                $\cDtrue(\optset\setminus H^{(i)})\leq \nfrac{\eps}{4\ell}.$
                Finally, exactly as in the computational case,
                \[
                    \cDtrue\!\inparen{\optset \setminus \bigcap\nolimits_{i\in S}H^{(i)}}\leq \frac{\eps}{4}
                    \qquadand
                    \cDtrue\!\inparen{\bigcap\nolimits_{i\in S}H^{(i)}\setminus \optset}
                    \leq \cDtrue(H^{(i^\star)}\setminus \optset)
                    \leq \eta'
                    \leq \frac{\eps}{16}\,.
                \]
                Therefore, the output hypothesis has error at most $\nfrac{5\eps}{16}<\eps$, completing the proof. 
            \end{proof}

\vspace{-3mm}
\section*{Acknowledgments} 
    We thank Alkis Kalavasis and anonymous reviewers for helpful comments on a draft of this paper which helped improve its presentation. 
    Jane H. Lee was supported by a Graduate Fellowship for STEM Diversity sponsored by the U.S.\ National Security Agency (NSA).

\newpage
\printbibliography
\newpage

\appendix 
\addtocontents{toc}{\protect\setcounter{tocdepth}{1}}

\section{Additional Related Works}\label{sec:appendix:related}

       In this section, we expand on a discussion of related works over \cref{sec:intro:relatedWorks}.

        \paragraph{Other PAC Learning Models.}
        One important class of problems between realizable and agnostic learning model the outcomes as $y_i = h^{\star}(x_i) + \eta_i$, where $\eta_i$ is a structured noise, \eg{}, see \cite{angluin1988learning, decatur1997cpcn, blum1998polynomial, feldman2006new, natarajan2013learning}.  In this paper, we consider a particular aspect of this structured noise that has many practical instantiations. In particular, we consider the case where there is significant asymmetry between positive and negative samples, and the quality of the data is very poor for one of the two types of samples -- say \mbox{negative samples. (See \cref{sec:realWorldApplications} for further examples.) }

        \paragraph{List-Decodable Learning {(Extension of Discussion in \cref{sec:intro:relatedWorks})}.}
            Almost all statistical tasks become information-theoretically impossible when a large fraction of samples is adversarially corrupted.
            The simplest example is perhaps Gaussian mean estimation: if an adversary can corrupt $(1-\gamma)$-fraction of the samples (for $\gamma\in (0,1)$), then the corrupted samples can contain $\nfrac{1}{\gamma}$-clusters that have a large separation, any of which could correspond to the true distribution.
            To circumvent these impossibility results, \citet{charikar2017learning} proposed a relaxed notion of learning, termed \textit{list-decodable learning}, where the learner is allowed to output a list of estimates one of which must be close to the correct parameter.
            Since the formulation of this problem, there has been a growing line of works designing list-decodable algorithms for several fundamental statistical problems  \cite{charikar2017learning,karmalkar2019list,Cherapanamjeri2020DecodableNearlyLinear,bakshi2021Subspace,Abhimanyu2023DecodableBatches,Diakonikolas2023DecodableCovariance,kothari2022decodableCovariance,Raghavendra2020Decoding,Diakonikolas2022DecodableSparseMean,Diakonikolas2021SQ_LB_Decoding,Diakonikolas2021PCA_Decodable}.
            As we mentioned in \cref{sec:intro:relatedWorks}, these (and future) works can be combined with our techniques to learn in the stronger list-decodable learning model.

\section{Additional Preliminaries and the \lreg{} Algorithm}
    In this section, we present the definitions of concept classes mentioned in the main body and the \lreg{} algorithm of \citet*{kalai2008agnostically} (see \cref{alg:l1reg}).
    
    Recall that a halfspace is a Boolean function $h\colon \R^d \rightarrow \{0, 1\}$ of the form $h(x) = \mathds{1}\sinbrace{w^\top x + t \geq 0}$ for some vector $w \in \R^d$ and threshold $t\in \R$. 
    An intersection of $k$ halfspaces $h_1,h_2,\dots,h_k$ is the Boolean function formed by taking the logical AND of the halfspaces, \ie{}, $h_1(x)\cdot h_2(x)\cdot \dots\cdot h_k(x)$.
    Functions of halfspaces are a significant generalization of intersections of halfspaces.
    \begin{definition}[Functions of Halfspaces]
        The concept class of functions of halfspaces consists of all Boolean functions of the form $b(h_1,h_2,\dots,h_k)$ where $b\colon \zo^k\to \zo$ is an arbitrary boolean function and $h_1,h_2,\dots,h_k$ are any $k$ halfspaces.
    \end{definition}
    Next, we formally define the class of decision trees of a specific size and/or depth.
    \begin{definition}[Decision Tree]
        A decision tree is a rooted binary tree in which each internal node is labeled with a variable $x_i$ (for $1\leq i\leq d$) and has two children. 
        Each leaf is labeled with an output from $\{0, 1\}$.
    \end{definition} 
    A decision tree computes a Boolean function in the following standard way: Given an input $x \in \{0,1\}^d$, the value of the function on $x$ is the output in the leaf reached by the path that starts at the root and goes left or right at each internal node according to whether the variable’s value in $x$ is 0 or 1, respectively.
    The \textit{size} of a decision tree is the number of leaves of the tree. The \textit{depth} of a node in a decision tree is the number of edges in the path from the root to the node.
    The \textit{depth of the tree} is the maximum depth of one of its leaves.
    Observe that a decision tree of depth $d$ has size at most $2^d$.
    This naturally defines the following concept classes.
    \begin{itemize}[itemsep=0pt]
        \item The family of \textit{decision trees of depth at most $t$} over $d$ variables; %
        \item The family of \textit{decision trees of size at most $s$} over $d$ variables.
    \end{itemize}
    When the number of variables $d$ is clear from context or not important, we simply express the above classes as decision trees of depth $t$ and of size $s$ respectively.
    Further, by taking the intersections of these classes, one can also define the concept class of decision trees of size at most $s$ and depth at most $t$.
    
    Decision trees of halfspaces are a generalization of decision trees, where each node $i$ is labeled with a halfspace $h_{i_j}$ instead of a single variable $x_{i_j}$.
    \begin{definition}[Decision Tree of Halfspaces]
        A decision tree of halfspaces is a rooted binary tree in which each internal node $i$ is labeled with a halfspace $h_{i_j}$ (for $1\leq i\leq d$) and has two children. 
        Each leaf is labeled with an output from $\{0, 1\}$.
    \end{definition} 
    The computation, on a decision tree of halfspaces. proceeds in a similar fashion as above: the only difference is that one goes to the left child from node $i$ if $h_{i_j}(x)=1$ and goes to the right child, otherwise, when $h_{i_j}(x)=0$.
    This naturally gives rise to the following concept class
    \begin{itemize}
        \item The family of \textit{decision trees of halfspaces of size at most $s$ and depth at most $t$} over $d$ variables. 
    \end{itemize}
    Observe that the above class can also be written as a sum of at most $s$ intersections of $\ell$ halfspaces.
    \begin{algorithm}[h!]
        \caption{\lreg{} \cite{kalai2008agnostically}}
        \begin{algorithmic}[1]
        \Procedure{{\rm\lreg{}}}{$\eps,\delta,k$}
                \vspace{2mm}
                \State \textit{\#~~Construct Samples}
                    \State Let $n=\wt{\Omega}\inparen{\sfrac{\inparen{d^k+\log\inparen{\nfrac{1}{\delta}}}}{\eps}}$ 
                    \State Draw $n$ samples $D=\inparen{(x_i,y_i)}_i$ from $\cD$
                    \vspace{2mm}
                \State \textit{\#~~Solve Regression}
                \State Find a polynomial $p$ of degree at most $k$ which minimizes $\inparen{\nfrac{1}{n}}\sum_{i}\abs{y_i-p(x_i)}$.
                \State \textit{\#~~This can be done by expanding all examples to monomials of degree at most $d$ and then}
                \State \textit{\#~~performing $L_1$-linear-regression using any polynomial-time LP-solver.}
                \State  $H=\inbrace{x\colon p(x)\geq t}$ where $0\leq t\leq 1$ minimizes the error of the hypothesis on $D$
                \vspace{2mm}
                \State \textit{\#~~Boost Performance}
                \State Repeat the above steps $T=O\!\inparen{\!\inparen{\nfrac{1}{\eps}}\log\inparen{\nfrac{1}{\delta}}}$ times to get $H_1,H_2,\dots,H_T$
                \State Evaluate $H_1,H_2,\dots,H_T$ on a fresh set of $O\!\inparen{\!\inparen{\nfrac{1}{\eps^2}}\log\inparen{\nfrac{1}{\delta}}}$ samples
                \State \textbf{return} the best-performing hypothesis on $D$
            \EndProcedure
        \end{algorithmic}
        \label{alg:l1reg}
    \end{algorithm}

\section{Proof of \cref{thm:sampleComplexity:bothSided} (Optimal Sample Complexity with \cref{asmp:densityLB})} 
    \label{sec:proofof:thm:sampleComplexity:bothSided}
    In this section, we prove \cref{thm:sampleComplexity:bothSided}, which obtains the optimal sample complexity for smooth positive-only learning when $q=1$; when an additional assumption (\cref{asmp:densityLB}) holds. 
    \optimalSampleComplexity*
    \noindent
    We prove \cref{thm:sampleComplexity:bothSided} via the following generalization which allows the hypothesis to be agnostic.
    Throughout this section we assume that \cref{asmp:smoothness,asmp:densityLB} hold with $q=1$.
    In particular, for any measurable set $S$,
    \[
        \sigma\cdot \cD(S)\leq \cDtrue(S)\leq \frac{1}{\sigma}\cdot \cD(S).
    \]
    \cref{thm:sampleComplexity:bothSided} can be deduced from the result below by setting $\opt=0$, applying the argument with internal accuracy parameter $\eps'=\nfrac{\sigma^2\eps}{44}$, and using $\optset\in \hyH$ to avoid requiring \cref{asmp:posFraction}; see \cref{rem:deductionDetails} for details.

    \begin{theorem}\label{new:thm:perm:robust}
        Suppose \cref{asmp:posFraction,asmp:smoothness,asmp:densityLB} hold with $q=1$, allow $\optset\not\in \hyH$, and define $\opt\coloneqq \inf_{H\in \hyH} \cDtrue(H\triangle \optset)$.
        Fix any parameters $\eps,\delta, \gamma \in (0,\nfrac{1}{2})$ with $\eps\geq \gamma > {\nfrac{\opt}{\alpha}}$.  
        Consider the \pERM{} instance defined with respect to $n=\wt{O}\!\inparen{\!\inparen{\nfrac{1}{\eps}}\cdot \inparen{\vc{}(\hyH)+\log{\nfrac{1}{\delta}}}}$  positive and unlabeled samples obtained from the two streams in the smooth positive-only learning model (\cref{prob:PUlearning}).
        Then, with probability $1-\delta$, any hypothesis $H\in \hyH$ that satisfies:
        \begin{enumerate}
            \item $H$ is feasible for \pERM{} with tolerance $\rho= \gamma $; and 
            \item $H$ is optimal (though not necessarily feasible) for \pERM{} with tolerance $\rho= 5\eps$,
        \end{enumerate}
        also satisfies
        \[
            \cDtrue\!\inparen{H\triangle \optset} \leq 
                \frac{44\eps}{\sigma^{2}}
                \,.
        \]
    \end{theorem}

    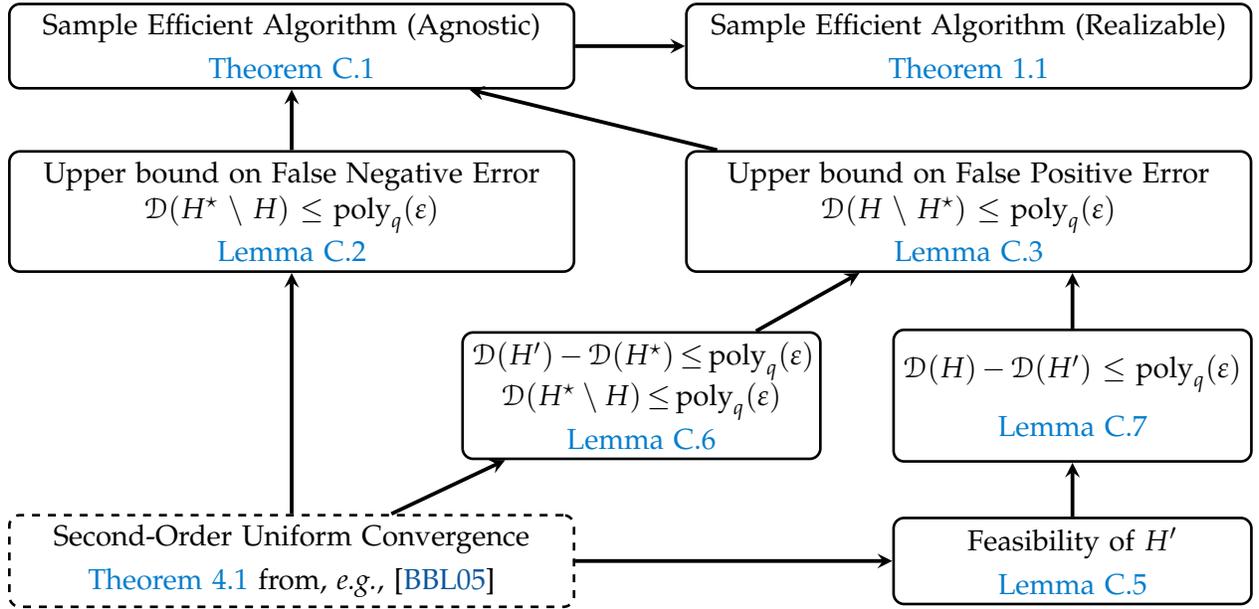
\begin{figure}[h]
        \centering
        \vspace{4mm}
        \begin{tikzpicture}[node distance=2.5cm, auto]
        \node[myNodeNarrow, line width=1pt, text width=7.25cm] (thmAgnostic) at (-4.5,0) 
            {{Sample-Efficient Algorithm (Agnostic)\\[1mm] \cref{new:thm:perm:robust}}};

        \node[myNodeNarrow, line width=1pt, text width=7.25cm] (thmRealizable) at (4.5,0) 
            {{Sample-Efficient Algorithm (Realizable)\\[1mm] \cref{thm:sampleComplexity:bothSided}}};

        \draw[-stealth, line width=0.5mm] (thmAgnostic) -- (thmRealizable); 

        \node[myNode, line width=1pt, text width=7.25cm] (lemEasy) at (-4.5,-2.2) 
            {{Upper bound on False Negative Error\\ $\cD(\optset\setminus H)~{\leq}~O(\eps)$\\[1mm]\cref{lem:perm:easy}}};

        \node[myNode, line width=1pt, text width=7.25cm] (lemHard) at (4.5,-2.2) 
            {{Upper bound on False Positive Error\\ $\cD(H\setminus\optset)~{\leq}~O(\eps)$\\[1mm]\cref{lem:perm:hard}}};

        \draw[-stealth, line width=0.5mm] (lemEasy) -- (thmAgnostic); 
        \draw[-stealth, line width=0.5mm] (lemHard) -- (thmAgnostic); 

        \node[myNode, line width=1pt, text width=4.5cm] (lemTerms23) at (0.15,-4.65) 
            {$\cD(H'){\,-\,}\cD(\optset){\,\leq\,}O(\eps)$\\$\cD(\optset\setminus H){\,\leq\,}O(\eps)$\\[1mm] \cref{lem:sampleComplexity:Hard:UBterms23}};
        \node[myNode, line width=1pt, text width=4.5cm, minimum height=1.75cm] (lemTerm1) at (5.88,-4.65) 
            {$\cD(H)-\cD(H')~{\leq}~O(\eps)$\\[3mm]\cref{lem:sampleComplexity:Hard:UBterm1}};

        \draw[-stealth, line width=0.5mm] (lemTerms23) -- (lemHard); 
        \draw[-stealth, line width=0.5mm] (lemTerm1.north) -- ($(lemTerm1.north |- lemHard.south)$);

        \node[myNode, line width=1pt, text width=4.5cm] (lemFeasibility) at (5.88,-6.85) 
            {Feasibility of $H'$\\[1mm]\cref{lem:sampleComplexity:Hard:feasibility}};

        \draw[-stealth, line width=0.5mm] (lemFeasibility) -- (lemTerm1); 

        \node[myNode, dashed, line width=1pt, text width=7.25cm] (unifConv) at (-4.5,-6.85) 
            {Second-Order Uniform Convergence\\[1mm]\cref{thm:unifConvergence} from, \eg{},  \cite{Boucheron_Bousquet_Lugosi_2005}};

        \draw[-stealth, line width=0.5mm] (unifConv) -- (lemFeasibility); 
        \draw[-stealth, line width=0.5mm] (unifConv) -- (lemTerms23); 
        \draw[-stealth, line width=0.5mm] (unifConv) -- (lemEasy); 
        
        \end{tikzpicture}
        \caption{Outline of the proofs of \cref{thm:sampleComplexity:bothSided,new:thm:perm:robust}.}
        \label{fig:outline:sampleComplexity}
    \end{figure}
    
    \noindent The proof of \cref{new:thm:perm:robust} is divided into two parts: \cref{lem:perm:easy} bounds $\cDtrue(\optset \setminus H)$ and \cref{lem:perm:hard} bounds $\cDtrue(H \setminus \optset)$.
        Taking a union bound over the events in \cref{lem:perm:easy,lem:perm:hard} and recalling that $\optset\triangle H=\inparen{H\setminus \optset}\cup \inparen{\optset\setminus H}$ completes the proof of \cref{new:thm:perm:robust}.
        \begin{lemma}[Upper Bound on False Negative Error]\label{lem:perm:easy} 
            Fix any $\eta\in (0,1).$
            Under the setting of \cref{new:thm:perm:robust}, consider any hypothesis $T\in \hyH$ feasible for the \pERM{} instance in \cref{new:thm:perm:robust} with tolerance $\rho=\eta$.
            Then, with probability $1-\inparen{\nfrac{\delta}{2}}$,
            \[
                \cDtrue\!\inparen{\optset\setminus T} \leq \eta+\sqrt{\eta\eps}+\eps\,.
            \]
            In particular, for $T=H$ and $\gamma\leq \eps$, with probability $1-\inparen{\nfrac{\delta}{2}}$, $\cDtrue\!\inparen{\optset \setminus H}\leq 3\eps$.
        \end{lemma}
        \begin{lemma}[Upper Bound on False Positive Error]\label{lem:perm:hard} 
            Under the setting of \cref{new:thm:perm:robust}, with probability $1-\inparen{\nfrac{\delta}{2}}$,  
            \[
                \cDtrue\!\inparen{H\setminus\optset} \leq 
                \frac{41\eps}{\sigma^2}\,.
            \]
        \end{lemma}
        The proofs of both \cref{lem:perm:easy,lem:perm:hard} rely on the following second-order uniform convergence bound, which gives sharper uniform convergence rates for hypotheses with small errors.

        \begin{remark}
            For simplicity, in \cref{new:thm:perm:robust}, we state the result in the regime $\opt\leq\eps$. %
            In the regime $\opt > \eps$, one can deduce bounds of the form $O\!\inparen{\nfrac{\opt\cdot \eps}{\alpha}}+O(\eps)$ from the same proof with additional bookkeeping.
        \end{remark}
        
        \subsection{Proof of \cref{lem:perm:easy} (Upper Bound on False Negative Error)}
        \begin{proof}[Proof of \cref{lem:perm:easy}]
            Let $T\in\hyH$ be any solution of \pERM{} with $\rho=\eta$. By feasibility, it satisfies $\nfrac{\abs{ T \cap P}}{\abs{P}}\geq 1- \eta $ and, hence, $\nfrac{\abs{ P\setminus T}}{\abs{P}}\leq \eta$.
            Due to the choice of the sample size $n$ and \cref{thm:unifConvergence}, with probability $1-\inparen{\nfrac{\delta}{2}}$, for each $R\in \hyH\cup \inbrace{G^c\colon G\in \hyH}$,\footnote{Here, we use the facts that $\vc{(\inbrace{G^c\colon G\in \hyH})}=\vc{(\hyH)}$ and $\vc{(\hyH_1\cup\hyH_2)}\leq {\vc{(\hyH_1)}+\vc{(\hyH_2)}+1}$ \cite{mohri2018foundations}.}
            \[
                \cPtrue(R) 
                \leq 
                \frac{\abs{R\cap P}}{\abs{P}}
                +
                \sqrt{\frac{\abs{R\cap P}}{\abs{P}}\cdot \eps}
                + \eps\,.
            \]
            In particular, applying this to $R=T^c$ (note that $T^c\in\{G^c: G\in \hyH\}$), and using $\nfrac{\abs{ P\cap T^c}}{\abs{P}}=\nfrac{\abs{ P\setminus T}}{\abs{P}}\leq \eta$, implies that 
            \[
                \cPtrue(T^c) \leq \eta+\sqrt{\eta\eps}+\eps\,.
            \]
            Finally, since $\cDtrue( R )\leq \cPtrue( R )$ for any $ R \subseteq \optset$, it follows that 
            \[
                \cDtrue(\optset\setminus  T )
                \leq \cPtrue(\optset\setminus  T )
                = \cPtrue(T^c)
                \leq  \eta +\sqrt{\eta\eps}+\eps\,.\qedhere{}
            \]  
        \end{proof}

        \subsection{Proof of \cref{lem:perm:hard} {(Upper Bound on False Positive Error)}}
            Our goal is to upper bound the mass of $H\setminus \optset$ under $\cDtrue$. 
            One natural idea is to use the fact the $\eps$-optimality of $H$ for \pERM{}, \ie{},  $\cDgiven(H)$ is smaller than $\cD(H')$ for any other $H'$ feasible for \pERM{}, and then convert from $\cDgiven$ to $
            \cDtrue$ via \cref{asmp:smoothness}.
            This approach does not go through as $\optset$ is not guaranteed to be feasible for \pERM{} (recall that in \cref{new:thm:perm:robust} we do not assume that $\optset\in \hyH$). %
            
            Instead, we construct a hypothesis $H'$ that is both ``close'' to $\optset$ and is also feasible for \pERM{}.
            Recalling that $\opt\coloneqq \inf_{H'\in \hyH}\cDtrue(H'\triangle \optset)$, one can see that for any $\zeta>0$, there is an hypothesis $H'=H'_\zeta$ in $\hyH$ satisfying
            \[
                \cDtrue\sinparen{H'\triangle \optset} 
                \leq \opt + \zeta\,.
                \yesnum\label{eq:perm:closeToInf}
            \]
            Consider any $0<\zeta\leq \alpha\eps$ satisfying $\gamma \geq \frac{\opt+\zeta}{\alpha}$, which exists since $\gamma > \nfrac{\opt}{\alpha}$.
            Fix $H'=H'_\zeta$ for the remainder of the proof.
            We now show that, with high probability, $H'$ is feasible for \pERM{} with tolerance $\rho=O(\eps)$.
            \begin{lemma}[Feasibility of $H'$]
                \label{lem:sampleComplexity:Hard:feasibility}
                Let $\evE$ be the event that the uniform convergence bounds in \cref{thm:unifConvergence} hold for the positive samples $P$ and the hypothesis class $\hyH\cup\inbrace{G^c\colon G\in \hyH}$.
                Then, $\Pr[\evE]\geq 1-(\nfrac{\delta}{4})$.
                Moreover, conditioned on $\evE$, $H'$ is feasible for \pERM{} defined by samples $(P,U)$ with $\rho=5\eps$.
            \end{lemma}
            \begin{proof}[Proof of \cref{lem:sampleComplexity:Hard:feasibility}]
                The lower bound on the probability of event $\evE$ follows directly from \cref{thm:unifConvergence}.
                We now prove the feasibility of $H'$.
                Since $\cPtrue$ is a truncation of $\cDtrue$ to $\optset$, $ \cPtrue\sinparen{H'\triangle \optset} \leq \nfrac{\cDtrue\sinparen{H'\triangle \optset}}{\cDtrue(\optset)}\leq \nfrac{\inparen{\opt+ \zeta }}{{\cDtrue(\optset)}}$ and, further, since $\cDtrue(\optset)\geq  \alpha $, 
                \[
                    \cPtrue\!\inparen{H'\triangle \optset} \leq \frac{\opt +  \zeta }{\alpha}\,.
                \]
                Condition on event $\evE$.
                The uniform convergence bound from \cref{thm:unifConvergence} implies that 
                \[
                    \frac{\abs{P\setminus H'}}{\abs{P}}
                    \leq
                        \cPtrue\!\inparen{{H'}^c}
                        + \sqrt{ \cPtrue\!\inparen{{H'}^c} \cdot \eps}
                        + \eps\,.
                \]
                Combining the last two inequalities implies
                \[
                    \frac{\abs{P\setminus H'}}{\abs{P}}
                    \leq
                        \frac{\opt +  \zeta }{\alpha}
                        + \sqrt{ \frac{\opt +  \zeta }{\alpha} \cdot \eps}
                        + \eps\,.
                \]
                Further as $\zeta,\opt\leq \alpha\eps$ and $0<\alpha\leq 1$, \mbox{$\nfrac{\abs{P\setminus H'}}{\abs{P}}
                    \leq 5\eps$ which ensures the feasibility of $H'$.}
            \end{proof}
            Now, we are ready to upper bound the mass of $H\setminus \optset$ under $\cDgiven$.
                For this, we use the following measure identity:
                \begin{align*}
                    \cDgiven\inparen{H\setminus \optset}
                        &=
                        \cDgiven\inparen{H}
                        -
                        \cDgiven\inparen{\optset}
                        +
                        \cDgiven\inparen{\optset \setminus H}\\
                        &=
                        \inparen{
                            \cDgiven\inparen{H} - \cDgiven\inparen{H'}
                        }
                        +
                        \inparen{
                            \cDgiven\inparen{H'} - \cDgiven\inparen{\optset}
                        }
                        +
                        \cDgiven\inparen{\optset \setminus H}\,.
                        \yesnum\label{eq:sampleComplexity:decomposition}
                \end{align*}
            The bulk of the remaining proof will focus on bounding the first term; before that, we turn to upper bounding the second and third terms. 
            \begin{lemma}
                \label{lem:sampleComplexity:Hard:UBterms23}
                Conditioned on event $\evE$ from \cref{lem:sampleComplexity:Hard:feasibility}, the following bounds hold
                \[
                    \cD(H') - \cD(\optset) \leq \frac{2\eps}{\sigma}
                    \qquadand
                    \cD(\optset\setminus H)
                    \leq \frac{3\eps}{\sigma}\,.
                \]
            \end{lemma}
            \begin{proof}[Proof of \cref{lem:sampleComplexity:Hard:UBterms23}]
                As $\cDgiven\sinparen{H'} - \cDgiven\inparen{\optset}\leq \cDgiven\sinparen{H'\triangle \optset}$, \cref{asmp:densityLB} implies that 
                \[
                    \cDgiven\sinparen{H'} - \cDgiven\inparen{\optset}
                    \leq \cDgiven\sinparen{H'\triangle \optset}
                    \leq \frac{1}{\sigma}\cdot \cDtrue\sinparen{H'\triangle \optset}
                    \,.
                \]
                By the definition of $H'$ (\cref{eq:perm:closeToInf}), 
                \[
                    \cDgiven\inparen{H'} - \cDgiven\inparen{\optset}
                    \leq \frac{1}{\sigma} \inparen{\opt+{ \zeta }}
                    ~~\quad\Stackrel{\opt,~\zeta\leq \eps}{\leq}\quad~~ \frac{2\eps}{\sigma}\,.
                \]
                Next, since $H$ is feasible for \pERM{} with tolerance $\rho=\gamma$, \cref{lem:perm:easy} implies that, conditioned on $\evE$, $\cDtrue(\optset\setminus H)\leq \gamma +\sqrt{\gamma\eps}+\eps$.
                Since $\gamma\leq \eps$, we deduce $\cDtrue(\optset\setminus H)\leq 3\eps$.
                Applying \cref{asmp:densityLB} implies the result.
            \end{proof} 
            Next, we upper bound the first term in \cref{eq:sampleComplexity:decomposition}.
            \begin{lemma}
                \label{lem:sampleComplexity:Hard:UBterm1}
                Let $\evF$ be the event that the uniform convergence bounds in \cref{thm:unifConvergence} hold for both the positive samples $P$ and the unlabeled samples $U$, and the hypothesis class $\hyF_{H'}$ 
                \[
                    \hyF_{H'}\coloneqq \hyH\cup\inbrace{G^c\colon G\in \hyH}\cup \inbrace{G\setminus (\optset \cap H^\circ\cap H')\colon G,H^\circ\in \hyH}\,.
                \]
                Then, $\Pr[\evF]\geq 1-(\nfrac{\delta}{4})$.
                Moreover, conditioned on $\evF$, 
                \[
                    \cD(H) - \cD(H')
                    \leq \frac{36\eps}{\sigma}
                    \,.
                \]
            \end{lemma}
            Note that \cref{lem:sampleComplexity:Hard:UBterms23,lem:sampleComplexity:Hard:UBterm1} imply that 
            \[
                \cDgiven\inparen{H\setminus \optset}
                    \leq 
                    \frac{2\eps}{\sigma}
                    + \frac{3\eps}{\sigma}
                    + \frac{36\eps}{\sigma}
                    \leq \frac{41\eps}{\sigma}\,.
            \]
            Now, \cref{lem:perm:hard} follows by using \cref{asmp:smoothness} (with $q=1$).

            It remains to prove \cref{lem:sampleComplexity:Hard:UBterm1}.
            The key idea is that $H$ is $\eps$-optimal for \pERM{} with $\rho=5\eps$ while $H'$ is feasible (\cref{lem:sampleComplexity:Hard:feasibility}), and, hence,
            \[
                \frac{\abs{H\cap U}}{\abs{U}}
                \leq 
                \frac{\abs{H'\cap U}}{\abs{U}}\,.
                \yesnum\label{eq:sampleComplexity:optimalityOfH}
            \]
            Now, one can control $\cD(H)-\cD(H')$ by combining uniform convergence and the above result.
            However, this requires $O(\nfrac{1}{\eps^2})$ samples since the error of $H'$, $\cD(H')$, may not be small. 
            To get the optimal sample complexity of $O(\nfrac{1}{\eps})$, we use the fact that both $H$ and $H'$ are close to $\optset$, \ie{}, their errors are small for the labeling $\mathds{1}\sinbrace{x\in \optset},$ and then use the second-order uniform convergence bounds (\cref{thm:unifConvergence}).
            \begin{proof}[Proof of \cref{lem:sampleComplexity:Hard:UBterm1}]
                The lower bound on $\Pr[\evF]$ follows from \cref{thm:unifConvergence} and the choice of the sample size $n$,
    once we note that $\vc(\hyF_{H'})=O(\vc(\hyH))$.
    Indeed, fix any finite set of points $X$ of size $m$.
    For any $G,H^\circ\in\hyH$,
    \[
        X\cap\inparen{G\setminus(\optset\cap H^\circ\cap H')}
        =
        (X\cap G)\setminus\inparen{(X\cap \optset\cap H')\cap (X\cap H^\circ)}\,.
    \]
    Hence, for fixed $X$, the number of distinct labelings induced by $\hyF_{H'}$ on $X$ is at most
    \[
        \abs{ \{X\cap G : G\in\hyH\} }
        + \abs{ \{X\cap G : G\in\hyH\} }
        + 
        \abs{ \inbrace{X\cap G : G\in\hyH} }\cdot \abs{ \inbrace{X\cap H^\circ : H^\circ\in\hyH} }
        = 3\Pi_{\hyH}(m)^2\,,
    \]
    so $\Pi_{\hyF_{H'}}(m)\leq 3\Pi_{\hyH}(m)^2$.
    Using Sauer--Shelah--Perles lemma, this implies $\vc(\hyF_{H'})\leq O(1) + 2\,\vc(\hyH)$ (in particular, $O(\vc(\hyH))$). 
                Define 
                \[
                    T = \optset \cap H\cap H'\,.
                \]
                Note that for every $G\in\hyH$ we have $G\setminus T \in \hyF_{H'}$ by taking $H^\circ = H$.
                Therefore, under $\evF$, the uniform convergence bounds apply to the class $\inbrace{G\setminus T:G\in\hyH}$.
                Since
                \[
                    \cD(H) = \cD(T) + \cD(H\setminus T)
                    \qquadand
                    \cD(H') = \cD(T) + \cD(
                    H'\setminus T)\,,
                \]
                we have
                \[
                    \cD(H)-\cD(H')
                    = 
                    \cD(H\setminus T)-\cD(H'\setminus T)
                    \leq \cD(H\setminus T)\,.
                    \yesnum\label{eq:sampleComplexity:UBterm1:ub}
                \]
                We now bound $\cD(H\setminus T)$ in two steps.  

                \paragraph{Step 1 (Upper bound on $\nfrac{\abs{U\setminus T}}{\abs{U}}$).}
                    Notice that 
                    \begin{align*}
                        \cD(H'\setminus T)
                        = \cD(H'\setminus \optset)
                        + \cD\!\inparen{\sinparen{H'\cap \optset}\setminus T}\,.
                    \end{align*}
                    Now, \cref{asmp:densityLB} and \eqref{eq:perm:closeToInf} 
                    \[
                        \cDgiven(H'\setminus \optset)
                        \leq \cDgiven(H'\triangle \optset)
                        \leq \frac{1}{\sigma}\,\cDtrue(H'\triangle \optset)
                        ~~\Stackrel{\eqref{eq:perm:closeToInf}}{\leq}~~ 
                            \frac{\opt+\zeta}{\sigma}
                        \qquad \Stackrel{(\opt,~\zeta\leq \eps)}{\leq} \qquad 
                            \frac{2\eps}{\sigma}\,.
                    \]
                    Therefore, 
                    \[
                        \cD(H'\setminus T)
                        \leq \frac{2\eps}{\sigma}
                        + \cD\!\inparen{\sinparen{H'\cap \optset}\setminus T}\,.
                    \]
                    Further, since $\optset\setminus T \subseteq \inparen{\optset \setminus H'} \cup \inparen{\optset\setminus H}$,
                    \[
                        \cD\!\inparen{\sinparen{H'\cap \optset}\setminus T}\leq \cD(\optset \setminus T)\leq \cD(\optset\setminus H')+ \cD(\optset\setminus H)\,.
                    \]
                    Therefore, 
                    \begin{align*}
                        \cD(H'\setminus T)
                        \leq \frac{2\eps}{\sigma} 
                        + \cD(\optset\setminus H')
                        + \cD(\optset\setminus H)\,.
                        \yesnum\label{eq:sampleComplexity:UBterm3:step1}
                    \end{align*}
                    We have the following upper bounds on the last two terms:
                    \begin{enumerate}
                        \item Since $H$ is feasible for \pERM{} with $\rho=\gamma\leq \eps$, conditioned on event $\evF$, \cref{lem:sampleComplexity:Hard:UBterms23}, implies that $\cD(\optset\setminus H)\leq \nfrac{3\eps}{\sigma}$.
                        (Here, we use the fact that $\evF$ implies $\evE$ by construction.)
                        \item Conditioned on $\evF$ (which implies $\evE$), $H'$ is feasible for \pERM{} with $\rho=5\eps$.
                        Hence, \cref{lem:perm:easy} implies that $\cDtrue(\optset\setminus H')\leq {5\eps}+{\sqrt{5}\, \eps}+\eps\leq 9\eps$.
                        Finally, \cref{asmp:densityLB} implies that $\cD(\optset\setminus H')\leq \nfrac{9\eps}{\sigma}$.
                    \end{enumerate}
                    Condition only event $\evF$.
                    Substituting the above two bounds into \cref{eq:sampleComplexity:UBterm3:step1}, implies that 
                    \[
                        \cD(H'\setminus T) ~~~\leq~~~ \frac{2\eps}{\sigma} + \frac{3\eps}{\sigma} + \frac{9\eps}{\sigma}
                        ~~~\Stackrel{}{\leq}~~~ \frac{14\eps}{\sigma}\,.
                    \]
                    Now, uniform convergence over the unlabeled samples and hypotheses class $\inbrace{G\setminus T\colon G\in \hyH}$ implies that,
                    \[
                        \frac{\abs{(H'\setminus T)\cap U}}{\abs{U}}
                        \leq 
                        \frac{14\eps}{\sigma}
                        + 
                        \sqrt{\frac{14\eps}{\sigma}\cdot \eps}
                        +
                        \eps\,. 
                    \]
                    Using $\sqrt{ab}\leq \frac{a+b}{2}$ with $a=\frac{14\eps}{\sigma}$ and $b=\eps$, and using $\sigma\leq 1$, we obtain
                    \[
                        \frac{\abs{(H'\setminus T)\cap U}}{\abs{U}}
                        \leq
                        \frac{14\eps}{\sigma}
                        +
                        \frac{1}{2}\cdot \frac{14\eps}{\sigma}
                        +
                        \frac{1}{2}\eps
                        +
                        \eps
                        \leq
                        \frac{23\eps}{\sigma}\,.
                        \yesnum\label{eq:sampleComplexity:boundSamples}
                    \]
                \paragraph{Step 2 (Upper bound on $\cD(H\setminus T)$).}
                    \cref{eq:sampleComplexity:optimalityOfH} and the fact that for any set $S$, $\nfrac{\abs{(S\setminus T)\cap U}}{\abs{U}}=\nfrac{\abs{S\cap (U\setminus T)}}{\abs{U}}$ implies that 
                    \[
                        \frac{\abs{(H\setminus T)\cap U}}{\abs{U}}
                        ~~\leq ~~
                        \frac{\abs{(H'\setminus T)\cap U}}{\abs{U}}
                        ~~\Stackrel{\eqref{eq:sampleComplexity:boundSamples}}{\leq}~~  \frac{23\eps}{\sigma}\,.
                    \]
                    Now, uniform convergence over $U$ and the hypothesis $\inbrace{G\setminus T\colon G\in \hyH}$ implies that,
                    \begin{align*}
                        \cD(H\setminus T)
                        \leq 
                            \frac{23\eps}{\sigma}
                            + \sqrt{\frac{23\eps}{\sigma} \cdot \eps}
                            + \eps\,.
                    \end{align*}
                    Using $\sqrt{ab}\leq \frac{a+b}{2}$ with $a=\frac{23\eps}{\sigma}$ and $b=\eps$, and using $\sigma\leq 1$, we get
                    \[
                        \cD(H\setminus T)
                        \leq
                        \frac{23\eps}{\sigma}
                        +
                        \frac{1}{2}\cdot \frac{23\eps}{\sigma}
                        +
                        \frac{1}{2}\eps
                        +
                        \eps
                        =
                        \frac{36\eps}{\sigma}\,.
                    \]
                    Substituting this into \cref{eq:sampleComplexity:UBterm1:ub} completes the proof.
            \end{proof}  

\begin{remark}[Deduction details for \cref{thm:sampleComplexity:bothSided}]\label{rem:deductionDetails}
The statement of \cref{new:thm:perm:robust} assumes \cref{asmp:posFraction} (hence $\alpha>0$), whereas \cref{thm:sampleComplexity:bothSided} is stated without this assumption.
To prove \cref{thm:sampleComplexity:bothSided} with target accuracy parameter $\eps$, we apply the proof of \cref{new:thm:perm:robust} with the internal accuracy parameter $\eps' \coloneqq \nfrac{\sigma^2\eps}{44}$
 and 
    $\gamma \coloneqq \eps'$.
Then the sample size required by that proof is
\[
    \wt{O}\!\inparen{
        \inparen{\nfrac{1}{\eps'}}\cdot
        \inparen{\vc{}(\hyH)+\log{\nfrac{1}{\delta}}}
    }
    =
    \wt{O}\!\inparen{
        \inparen{\nfrac{1}{\eps\sigma^2}}\cdot
        \inparen{\vc{}(\hyH)+\log{\nfrac{1}{\delta}}}
    }.
\]
In the realizable setting of \cref{thm:sampleComplexity:bothSided}, we have $\opt=0$ and, moreover, the target set $\optset$ belongs to the hypothesis class $\hyH$.
In this case, we can run the same proof with the following local modification, which removes any dependence on $\alpha$:
in the proof of \cref{lem:perm:hard}, choose $H' \coloneqq \optset \in \hyH$.
Then $H'$ is automatically feasible for \pERM{} with tolerance $\rho=0$, since the positive sample $P$ is drawn from $\cPtrue$, whose support is contained in $\optset$, and hence $\frac{\abs{P\setminus H'}}{\abs{P}} = 0.$
Consequently, the auxiliary feasibility lemma \cref{lem:sampleComplexity:Hard:feasibility} (whose proof divides by $\alpha$) is not needed in the realizable deduction.
All subsequent steps bounding $\cD(H\setminus \optset)$ and then converting to $\cDtrue(H\setminus \optset)$ use only \cref{asmp:smoothness} and \cref{asmp:densityLB} (with $q=1$), together with the uniform convergence events, and therefore remain unchanged.
Running the same argument with internal accuracy parameter $\eps'$ therefore yields
$\cDtrue(H\triangle \optset)
    \leq \nfrac{44\eps'}{\sigma^2}
    = \eps,$
which proves \cref{thm:sampleComplexity:bothSided}.
\end{remark}

\section{Proofs Deferred from the Main Body} %
    \subsection{Proof of \cref{lem:constReg:unifConvergenceHoldsWHP} (Uniform Convergence in Each Repetition)}
                \label{sec:proofof:lem:constReg:unifConvergenceHoldsWHP}
                Recall that \cref{lem:constReg:unifConvergenceHoldsWHP} requires us to show that the following events hold with probability $1-\delta$:
                \begin{itemize}
                    \item[] For every repetition $1 \leq i \leq T$ of Subroutine B of \cref{alg:constReg}, the samples in each repetition satisfy $\zeta$-uniform convergence with respect to degree-$k$ PTFs, and
                \end{itemize}
                The proof follows from standard uniform convergence arguments and a union bound.
                \begin{proof}
                    Recall that the VC-dimensions of degree-$k$ polynomials is $\wt{O}(d^{k})$.
                    Therefore, for any $\eta\in (0,1)$, 
                    \[
                        \wt{O}\!\inparen{
                            \frac{1}{\zeta^2}
                            \cdot \inparen{d^k + \log{\frac{1}{\eta}}}
                        }\,.
                    \]
                    samples are sufficient to achieve $\zeta$-uniform convergence with respect to degree-$k$ polynomials with probability $1-\eta$ for any single repetition of Subroutine B.
                    We set $\eta=\nfrac{\delta}{T}$ and take a union bound over all $T$ events, which are all the $T$ independent repetitions of Subroutine B, to get that if $n$ is at least 
                    \[
                        \wt{O}\!\inparen{
                        \frac{1}{\zeta^2}
                        \cdot \inparen{d^k + \log{\frac{T}{\delta}}}
                        }\,,
                    \]
                    then the result follows.
                    Now, substituting the value of $T$, namely, $T=O\!\inparen{\!\inparen{\nfrac{1}{\sqrt{\zeta}}}\cdot \log{\nfrac{1}{\delta}}}$, and simplifying implies that the required number of samples is 
                    \[
                         \wt{O}\!\inparen{
                         \frac{1}{\zeta^2}
                            \cdot \inparen{d^k + \log{\frac{1}{\delta}}}
                        }\,.
                    \]
                    The choice of $n$ in \cref{alg:constReg} satisfies this.\qedhere{}
                \end{proof}

    \subsection{Proof of \cref{lem:constReg:boost} (Boosting via Repetition)}\label{sec:proofof:lem:constReg:boost}
                \begin{proof}[Proof of \cref{lem:constReg:boost}]
                Fix $\eta$ corresponding to the guarantee in \cref{lem:constReg:expectedValue}
                \[
                    \eta= O\!\inparen{\sqrt{\frac{C+1}{\alpha\sigma}}} \cdot \zeta^{1/(4q)}
                        + O\sinparen{\delta}\,.
                \]
                Assume $\eta\leq \nfrac{1}{2}$, which holds after choosing the hidden constants appropriately.
                Observe that for each $i$, $\nfrac{\abs{H_i\cap U_i}}{\abs{U_i}}\geq 0$ and, hence, Markov's inequality implies that 
                \[
                    \Pr\insquare{\frac{\abs{H_i\cap U_i}}{\abs{U_i}} 
                    \geq \cD(H_\opt)+2\eta} 
                    \leq {\frac{\cD(H_\opt)+\eta}{\cD(H_\opt)+2\eta}}\,.
                \]
                Hence, since each repetition of Subroutines A and B in \cref{alg:constReg} is independent, it follows that 
                \[
                    \Pr\insquare{
                        \min_{1\leq i\leq T}
                    \frac{\abs{H_i\cap U_i}}{\abs{U_i}} \geq \cD(H_\opt)+2\eta 
                    } 
                    \leq \inparen{{\frac{\cD(H_\opt)+\eta}{\cD(H_\opt)+2\eta}}}^T
                    ~~~~~~\Stackrel{\substack{\cD(H_\opt)\leq 1,\\0\leq\eta\leq \nfrac{1}{2}}}{\leq}~~~~~~ 
                        \inparen{1-\frac{\eta}{2}}^T
                    \leq e^{-\frac{1}{2}\eta T}\,.
                \]
                Since $T=O\!\inparen{\!\inparen{\nfrac{1}{\sqrt{\zeta}}}\log\inparen{\nfrac{1}{\delta}}}$ and $\eta=\Omega\sinparen{\zeta^{1/(4q)}}$, the hidden constant in the choice of $T$ can be chosen large enough so that $\eta T\geq 2\log{\nfrac{1}{\delta}}$.
                Hence, with probability $1-\delta$, there exists an index $1\leq j\leq T$ such that
                \[
                    \frac{\abs{H_j\cap U_j}}{\abs{U_j}} \leq \cD(H_\opt)+2\eta\,.
                \]
                Further, with probability $1-\delta$, $\zeta$-uniform convergence holds with respect to each $U_1,U_2,\dots,U_T$ and hypothesis of degree-$k$ PTFs (see \cref{lem:constReg:unifConvergenceHoldsWHP}).
                Therefore, since $\zeta\leq \eta$, with probability $1-2\delta$, there exists an index $1\leq j\leq T$ such that
                \[
                    \cD(H_j)\leq \cD(H_\opt)+2\eta+\zeta\leq \cD(H_\opt)+3\eta\,.
                \]
                Define $\wh{H}=H_{i^\star}$ for $i^\star\in \argmin_{1\leq i\leq T} \abs{H_i\cap U'}$.
                Now, a standard uniform-convergence bound over the finite class $\inbrace{H_1,\dots,H_T}$ implies that with probability $1-\delta$, $\abs{\cD(H_i) - \inparen{\nfrac{\abs{H_i\cap U'}}{\abs{U'}}}}\leq \eta$ (for all $1\leq i\leq T$).\footnote{This uses that $\eta^{-2}=O\!\inparen{\zeta^{-1/(2q)}}\leq O\sinparen{\zeta^{-1}}$ for $q\geq 1$ and $\zeta\in(0,1]$, so the sample size in Step~\ref{step:constReg:freshSamples} is sufficient.}
                Hence, with probability $1-3\delta$,
                \[
                    \cD(\wh{H})
                    \leq \frac{\sabs{\wh{H}\cap U'}}{\abs{U'}} + \eta
                    \leq \frac{\abs{H_j\cap U'}}{\abs{U'}} + \eta
                    \leq \cD(H_j) + 2\eta
                    \leq \cD(H_\opt) + 5\eta\,.
                \]
                By the constant-factor rescaling of $\delta$ at the beginning of the proof, this implies the claim.
            \end{proof}

    \subsection{Framing Program~\eqref{eq:constReg:prog} in \cref{alg:constReg} as a Linear Program}\label{sec:module:linearProgram}
            In this section, we show how to formulate the optimization program in \cref{alg:constReg} as a linear program. 
            Recall that Program~\eqref{eq:constReg:prog} is as follows.
            \[
                \min_{{\rm deg}(p)\leq k}~\frac{\sum_{x\in U} \abs{{p(x)}}}{\abs{U}}
                    \,,\quad 
                    \text{such that}\,,\quad 
                    {\frac{\sum_{x\in P} \min\!\inbrace{p(x), 1}}{\abs{P}} \geq {1-\rho}} \,.
            \]
            Let $\hypo{R}$ be the set of all multisets $R\subseteq [d]$ of size at most $k$.
            For each $z\in \R^d$ and $R\in \hypo{R}$, define the monomial $z(R)\coloneqq \Pi_{i\in R} z_i$.
            For each monomial $z(R)$, let $c(R)$ be the corresponding coefficient in a polynomial $p$.
            The above program is equivalent to the following linear program.
            \begin{align*} 
                \min_{c\in \R^{\hypo{R}},\ z\in \R^{\abs{U}},\ w\in \R^{\abs{P}}}~
                    & {\sum_{x\in U} z(x)} \,,\\
                \forall_{x\in U}\,,\quad& 
                    z(x) ~~\geq~~
                    \sum_{R\in \hypo{R}} c(R) x(R)
                \,,\\
                \forall_{x\in U}\,,\quad& 
                    z(x) ~~\geq~~
                    -\sum_{R\in \hypo{R}} c(R) x(R)
                \,,\\
                \forall_{x\in P}\,,\quad&   
                    w(x)~\leq~~ 1\,,\\
                \forall_{x\in P}\,,\quad& w(x) ~\leq~ ~ 
                    \sum_{R\in \hypo{R}} c(R) x(R)
                \,,\\
                & \sum_{x\in P} w(x) \geq \inparen{1-\rho} \cdot \abs{P}\,.
            \end{align*}

    \subsection{\cref{asmp:smoothness} Is Weaker than Bounding $\chi^2$--Square-Divergence}\label{sec:distCloseness}
        In this section, we show that if $\chidiv{\cDtrue}{\cD}<\infty$, then \cref{asmp:smoothness} holds. %
        \begin{lemma}
            \label{lem:closenessFromChiSquare}
            For $B\geq 0$, if $\chidiv{\cDtrue}{\cD}\leq B$, then \cref{asmp:smoothness} holds with $\sigma=\frac{1}{\sqrt{1+B}}$ and $q=2$.
        \end{lemma}
        \begin{proof}
            Fix measurable set $S$.
            Change of measure and the Cauchy--Schwarz inequality imply that
            \[
                \cDtrue(S) 
                = \int \mathds{1}_S(x) \frac{\d\cDtrue(x)}{\d\cD(x)}\d\cD(x)
                \leq 
                    \inparen{
                        \int \mathds{1}_S(x)^{2}~ \d\cD(x)
                    }^{1/2}
                    \inparen{
                        \int \inparen{\frac{\d\cDtrue(x)}{\d\cD(x)}}^2 \d \cD(x)
                    }^{1/2}\,.
            \]
        Since $\mathds{1}_S(x)^{2}=\mathds{1}_S(x)$ and $
                \chidiv{\cDtrue}{\cD}= 
            {\int \inparen{\nfrac{\d \cDtrue(x)}{\d \cD(x)}}^2 \d \cD(x)} - 1$,
        the above is the same as 
        \[
            \cDtrue(S) \leq \cD(S)^{1/2}\cdot 
                \inparen{\chidiv{\cDtrue}{\cD}+1}^{1/2}
                \leq 
                    \sqrt{1+B}\cdot 
                    \cD(S)^{1/2}
                \,.
        \]
        The result follows since this holds for all measurable sets $S$.
        \end{proof}
        More generally, generalizing the above proof (by using Hölder's inequality instead of the Cauchy--Schwarz inequality) shows that \cref{asmp:smoothness} also follows from bounds on Rényi divergences:
        \begin{lemma}
            \label{lem:closenessFromRenyi}
            \mbox{For $B\geq 0,r > 1$, if $\renyi{r}{\cDtrue}{\cD}\leq B$, then \cref{asmp:smoothness} holds with $\sigma=e^{-B/q}$ and $q=\frac{r}{r-1}$.}
        \end{lemma}
        Note that the first result follows from the second one when $r=2$.
        To see this, recall that $\renyi{2}{\cDtrue}{\cD} = \log\sinparen{1+\chidiv{\cDtrue}{\cD}}$, so $\chidiv{\cDtrue}{\cD}\leq B$ implies $\renyi{2}{\cDtrue}{\cD}\leq\log(1+B)$. 
        Applying \cref{lem:closenessFromRenyi} with $r=2$ (thus, $q=2$) gives $\cDtrue(S) \leq \cD(S)^{1/2}\,e^{\frac{1}{2}\renyi{2}{\cDtrue}{\cD}} \leq \cD(S)^{1/2}\,\sqrt{1+B}$, which is precisely the conclusion of \cref{lem:closenessFromChiSquare}.

    \subsection{Satisfying \cref{eq:mostGeneral:smoothness} via Bounds on KL-Divergence}
        \label{sec:klBoundsSatisfyGeneralSmoothness}
        Next, we show that if $\kl{\cDtrue}{\cD}<\infty$, then \cref{eq:mostGeneral:smoothness}, a weakening of \cref{asmp:smoothness}, holds.

        We restate \cref{eq:mostGeneral:smoothness} here  for convenience:
            there is a known rate function $r\colon \R_{\geq 0}\to \R_{\geq 0}$ satisfying $\lim_{m\to 0^+} r(m)=0$ such that the pair of distributions $\inparen{\cD, \cDtrue}$ satisfy
            \[
                \text{for any measurable set $S$}\,,\quad 
                r(\cDtrue(S))\leq \cD(S)\,.
                \tag{4}
            \] 
        \vspace{-5mm}
            
        \begin{theorem}\label{thm:capturingKL}
        If $\kl{\cDtrue}{\cD}=C$, then \eqref{eq:mostGeneral:smoothness} holds with the rate function
        \[
            r(m)\coloneqq
            \begin{cases}
                0, & m=0,\\[1mm]
                \exp\!\inparen{-\nfrac{3\max\inbrace{1,C}}{m}}, & m>0.
            \end{cases}
        \]
        Moreover, $\lim_{m\to 0^+} r(m)=0$.
    \end{theorem}
        In contrast, $\kl{\cDtrue}{\cD}<\infty$ is insufficient to imply \cref{asmp:smoothness} for any finite $\sigma$ and $q$ (see \cref{rem:KL:insufficient}).
        Hence, the weakening of \cref{asmp:smoothness} in \cref{eq:mostGeneral:smoothness} is necessary to capture $\kl{\cDtrue}{\cD}<\infty$.

        To prove \cref{thm:capturingKL}, we will use the following inequality, which we believe is well-known but since we could not find an explicit reference, we prove it at the end of this section.
        \begin{lemma}\label{lem:kl_transfer}
            Let $\cP$ and $\cQ$ be probability distributions on $\R^d$ with $\kl{\cQ}{\cP} < \infty$. 
            Then, for any measurable set $S$,
            \[
                \cQ(S) \leq 
                \inf_{t>1} \inbrace{t\,\cP(S) + \frac{\kl{\cQ}{\cP}}{\ln t}}\,.
            \]
        \end{lemma} 
        Now we are ready to prove \cref{thm:capturingKL}.
        \begin{proof}[Proof of \cref{thm:capturingKL}]
            Fix any measurable set $S$.
            \begin{itemize}
                \item If $\cD(S)=0$, then $\cDtrue(S)=0$ because $\kl{\cDtrue}{\cD}<\infty$ implies $\cDtrue\ll \cD$.
                Hence the desired inequality holds since $r(0)=0$.
                \item If $\cD(S)=1$, then the desired inequality is trivial because $r(\cDtrue(S))\leq 1=\cD(S)$.
                \item Further, if $\cDtrue(S)=0$, then the desired inequality is immediate since $r(0)=0$.
            \end{itemize}
            Thus, it remains to consider the case $0<\cD(S)<1$ and $\cDtrue(S)>0$.
            By \cref{lem:kl_transfer},
            $\cDtrue(S)\leq\inf_{t>1}\inbrace{t\,\cD(S)+\frac{C}{\ln t}}.$
            Choosing $t=\nfrac{1}{\sqrt{\cD(S)}}$ and using that $\sqrt{z}\leq \frac{1}{\ln(1/z)}$ for all $z\in (0,1)$, gives
            \[
            \cDtrue(S)\leq\sqrt{\cD(S)}+\frac{2C}{\ln(1/\cD(S))}\leq\frac{3\max\!\inbrace{1, C}}{\ln(1/\cD(S))}\,.
            \]
            Therefore, as $\cDtrue(S)>0$,
            \[
                \cD(S)\geq \exp\inparen{-\frac{3 \max\!\inbrace{1, C}}{\cDtrue(S)}}\,.
            \]
            This is exactly \eqref{eq:mostGeneral:smoothness} with the rate function described in \cref{thm:capturingKL}.
        \end{proof}
        \begin{proof}[Proof of \cref{lem:kl_transfer}]
        Since $\kl{\cQ}{\cP}<\infty$, we have $\cQ\ll \cP$.
        Hence the Radon--Nikodym derivative $\nfrac{\d\cQ}{\d\cP}$ exists.
        Fix any $t>1$, and define the set
        \[
        T_t \coloneqq \inbrace{\, x \in \R^d : \frac{\d\cQ}{\d\cP}(x) > t }\,.
        \]
        On $T_t$, we have $\ln\inparen{\sfrac{\d\cQ}{\d\cP}(x)} > \ln t$. Hence,
        \[
        \kl{\cQ}{\cP} = \int \cQ(x)\,\ln{\frac{\d\cQ}{\d\cP}(x)}\,\d x
        \geq \int_{T_t} \cQ(x)\,\ln{\frac{\d\cQ}{\d\cP}(x)}\,\d x
        > (\ln t)\,\cQ(T_t)\,,
        \]
        which gives $\cQ(T_t) \leq \kl{\cQ}{\cP}/\ln t$. Next, split
        $\cQ(S) = \cQ(S \cap T_t) + Q(S \cap T_t^c)\,.$
        On $S \cap T_t^c$, we have $\nfrac{\d\cQ}{\d\cP}(x) \leq t$, so
        \[
        Q(S \cap T_t^c)
        = \int_{S \cap T_t^c} \cQ(x)\,\d x
        \leq \int_{S \cap T_t^c} t\,\cP(x)\,\d x
        = t\,\cP(S \cap T_t^c)
        \leq t\,\cP(S)\,.
        \]
        Therefore,
        \[
            \cQ(S) \leq \cQ(S \cap T_t) + \cQ(S \cap T_t^c)
            \leq \cQ(T_t) + t\,\cP(S)
            \leq 
            \frac{\kl{\cQ}{\cP}}{\ln t} + t\,\cP(S)\,.
        \]
        Since $t>1$ was arbitrary, taking the infimum over $t>1$ completes the proof.
        \end{proof}
 
    \vspace{-7mm}
            \begin{remark}\label{rem:KL:insufficient}
Finite KL-divergence does not imply \cref{asmp:smoothness} for any finite $\sigma$ and $q$.
Indeed, let $\cX=\N$ and define $\cDtrue(n)\propto e^{-n^2}$ and $\cD(n)\propto e^{-n^3}$.
Then $\kl{\cDtrue}{\cD}<\infty$ because
$\sum_{n\geq 1} e^{-n^2}\inparen{n^3-n^2}<\infty.$
However, for any $n\in \N$, 
    \[
        \frac{\cDtrue(\inbrace{n})}{\cD(\inbrace{n})^{1/q}}
        =
        \Theta\!\inparen{e^{-n^2+n^3/q}}\,,
    \]
    which goes to $\infty$ as $n\to \infty$ for any fixed finite $q$.
    Hence, there do not exist finite $\sigma$ and $q$ such that
    $(\cDtrue,\cD)$ satisfy \cref{asmp:smoothness}.
Thus, the weakening in \cref{eq:mostGeneral:smoothness} is genuinely needed in order to capture finite KL-divergence.
\end{remark}
    \vspace{-7mm}

\subsection{Generating List of Distributions From Adversarially Corrupted Sources}\label{sec:advCorruption}
            In this section, we show how one can satisfy the List Decoding model's requirements (\cref{def:list}) in the following types of data sources -- where all but an $\gamma$-fraction of the samples can be \textit{adversarially} corrupted.

            \begin{mdframed}
                \textbf{Example (Adversarially Corrupted Gaussian Source):}\quad
                Suppose that $\cDtrue$ is a Gaussian distribution $\cN(\muStar,\SigmaStar)$.
                Consider a data source that is $(1-\gamma)$-corrupted (for some $\gamma\in (0,1]$):
                    with probability $\gamma$, the source outputs a sample from $\cDtrue$, and, otherwise, outputs an arbitrary sample that may depend on the past samples and the knowledge of the learning algorithm.
            \end{mdframed}
            We show that one can efficiently convert the above $(1-\gamma)$-corrupted source for any $\gamma>0$, into one satisfying the requirements of the List Decoding Model with $\ell= \poly(\nfrac{1}{\gamma})$.

            In the remainder of this section, we will prove a generalization of this result where $\cDtrue$ is any exponential family distribution satisfying mild assumptions.
            The result for Gaussian $\cDtrue$ follows as a special case of this.

            \paragraph{Notation.}
            We begin by recalling the definition of an exponential family. %
            \begin{definition}[\textbf{Exponential Family}]
                Given $m\geq 1$, a function $h\colon \R^d\to \R_{\geq 0}$ called the \emph{carrier measure}, a function $t\colon \R^d\to \R^m$ called the \emph{sufficient statistic}, the exponential family $\cE_{m,h,t}$ is a set of distributions parameterized by a vector $\theta\in \R^m$ where, for each $\theta$, $\cE_{m,h,t}(\theta)\propto h(x) \cdot \exp\!\inparen{\theta^\top t(x)}$.
            \end{definition}
            Henceforth, $m$, $h$, and $t$ will be clear from context and, hence, we drop them from the subscript. %
            The set of $\theta$ for which the density of $\cE(\theta)$ is well defined is called the \textit{natural parameter space} $\overline{\Theta}$.
            For each $\theta\in \overline{\Theta}$, define the \textit{log-partition function} $A(\theta)\coloneqq \log~{\int h(x) \cdot \exp\!\inparen{\theta^\top t(x)}\d x}$.
            For any canonical exponential family, $\overline{\Theta}$ is a convex set and $A$ is a convex function \cite{dasgupta2008asymptotic}.

            \paragraph{Assumptions.}
            Suppose $\cDtrue = \cE(\theta^\star)$ for some unknown parameter $\theta^\star$.
            We will be interested in a subset $\Theta\subseteq \overline{\Theta}$ where the Fisher Information Matrix $\cov_{\cE(\theta)}[t(x)]$ has bounded eigenvalues.
            Concretely, we assume the family $\cE$ satisfies \cref{unknowntrunc:int,unknowntrunc:proj,unknowntrunc:fisher} of \cref{asmp:exponentialFamily}. 
            \begin{remark}[Assumptions Hold Unconditionally for Gaussians]
                Before proceeding, we note when $\cDtrue$ is a Gaussian distribution, then \cref{unknowntrunc:int,unknowntrunc:proj,unknowntrunc:fisher} of \cref{asmp:exponentialFamily} always holds.
                This is due to pre-processing routines by \citet{lee2024unknown} which show how to convert an arbitrary Gaussian distribution with parameter $\theta^\star=(\muStar,\SigmaStar)$ into one satisfying \cref{unknowntrunc:int,unknowntrunc:proj,unknowntrunc:fisher} of \cref{asmp:exponentialFamily}.
            \end{remark}
 
            \noindent In this section, we prove the following result.
            \begin{theorem}\label{thm:listDecoding:expFamily}
                Suppose \cref{unknowntrunc:int,unknowntrunc:proj,unknowntrunc:fisher} of \cref{asmp:exponentialFamily} holds with constants $\eta,\lambda,\Lambda>0$.
                There is an algorithm that, given $\gamma>0$ and $n={\inparen{\nfrac{m}{\gamma}}\log\inparen{\nfrac{1}{\delta}}}$ samples from $\cE(\theta^\star)$ out of which $(1-\gamma)$-fraction are adversarially corrupted, %
                outputs a list of $\nfrac{2}{\gamma}$ parameters $\theta_1,\theta_2,\dots,\theta_k$ such that, with probability at least $1-\delta$,
                for at least one index $i$ ($1\leq i\leq k$), 
                    $\norm{\theta_i-\theta^\star}_2
                        \leq
                        O\!\inparen{
                            \sqrt{\nfrac{\Lambda\log{\nfrac{2}{\gamma}}}{\gamma\lambda^2}}
                        }.$
                    Consequently, if
                    \[
                        w
                        =
                        1+\Omega\!\inparen{
                            \min\!\inbrace{1, \sqrt{\frac{\eta^2\lambda^2\gamma}{\Lambda\log{\nfrac{2}{\gamma}}}}
                            }
                        }
                    \qquadtext{then}
                        \renyi{w}{\cE(\theta^\star)}{\cE(\theta_i)}
                        \leq
                        O\!\inparen{
                            \frac{\Lambda^2}{\lambda^2}
                            \cdot
                            \frac{\log{\nfrac{2}{\gamma}}}{\gamma}
                        }\,.
                    \]
            \end{theorem}
            Now \cref{lem:closenessFromRenyi} implies that the above closeness in Rényi divergence implies \cref{asmp:smoothness} with $q=\nfrac{w}{(w-1)}=\wt{O}\!\inparen{\max\!\inbrace{1,\sqrt{\nfrac{\Lambda}{(\eta^2\lambda^2\gamma)}}}}$ and $\sigma=\exp(-\wt{O}\!\inparen{\nfrac{\Lambda^2}{(\lambda^2 q \gamma)}})$.
            \begin{proof}[Proof of \cref{thm:listDecoding:expFamily}]
                The proof of this result uses techniques from \cite{charikar2017learning}.

                Fix $\beta= \nfrac{1}{2}$. %
                Let the (corrupted) samples be $X=\inbrace{x_1,x_2,\dots,x_n}$ and let $x_1,\dots,x_{r}$ be the set of \textit{uncorrupted} samples from $X$.
                For each $i$, define $f_i\colon \R^m\to \R$ to be the negative log-likelihood function, \ie{}, 
                $f_i(\theta) \coloneqq -\log\inparen{\cE(x_i; \theta)}$.
                Let $\overline{f}\colon \R^m \to \R$ be the population log-likelihood, \ie{}, $\overline{f}(\theta)=\Ex_{z\sim \cE(\theta^\star)}\insquare{-\log\inparen{\cE(z;\theta)}}$.
                Standard properties of the exponential density (see, \eg{}, \citet{dasgupta2008asymptotic}) imply that:
                for all $i$ and $\theta$
                \begin{align*}
                    \nabla f_i(\theta) &= \nabla A(\theta) - t(x_i)
                    ~~\quad\quadand~~
                    \nabla^2 f_i(\theta) = \nabla^2 A(\theta)\,,\\
                    \nabla \overline{f}(\theta) &= \nabla A(\theta) - \nabla A(\theta^\star)
                    \quadand~~
                    \nabla^2 \overline{f}(\theta) = \nabla^2 A(\theta)\,.
                \end{align*}
                Where $A\colon \R^m \to \R$ is the log-partition function of the exponential family and $t\colon \R^d\to \R^m$ is the sufficient statistic of the exponential family.
                
                By \cref{unknowntrunc:fisher} of \cref{asmp:exponentialFamily}, $A$ is $\lambda$-strongly convex and $\Lambda$-smooth over $\Theta$.
                Hence each $f_i$ and $\overline{f}$ is $\lambda$-strongly convex over $\Theta$, so \cite[Theorem 6.1]{charikar2017learning} applies once we verify the required spectral bound.
                After reindexing, assume that $x_1,\dots,x_r$ are the uncorrupted samples, so $r\geq \gamma n$.
                For each $\theta\in \Theta$ and random $X\sim \cE(\theta^\star)$, define
                \[
                    g_\theta(X)
                    \coloneqq
                    \nabla f_X(\theta)-\nabla \overline{f}(\theta)
                    =
                    \nabla A(\theta^\star)-t(X)\,.
                \]
                Then
                \[
                    \Ex_{X\sim \cE(\theta^\star)}\insquare{g_\theta(X)}=0
                    \qquad\text{and}\qquad
                    \cov_{X\sim \cE(\theta^\star)}\insquare{g_\theta(X)}
                    =
                    \cov_{X\sim \cE(\theta^\star)}\insquare{t(X)}
                    \preceq
                    \Lambda I\,,
                \]
                where the last inequality follows from \cref{unknowntrunc:fisher} of \cref{asmp:exponentialFamily}.
                Now, we can establish a bound on $\min_{1\leq i\leq k}\norm{\theta_i-\theta^\star}_2$ using a straightforward adaptation of the proof of \cite[Corollary 9.1]{charikar2017learning}, and then show that this implies a bound on Rényi divergence.

                \itparagraph{Bound on the $\cR_w$ divergence.}
                    Applying \cite[Proposition B.1]{charikar2017learning} to the \iid{} uncorrupted samples $x_1,\dots,x_r$, together with the above bounds on the mean and covariance of $g_\theta(X)$, implies that with probability at least $1-e^{-\Omega(\beta^2\gamma n)}$, there exists a subset $G\subseteq [n]$ of $(1-(\nfrac{\beta}{2}))\gamma n$  good points with spectral bounds of $S=O(\sqrt{\nfrac{\Lambda}{\beta}})$ and (by \cite[Proposition 5.6]{charikar2017learning}) $S_{\sfrac{\beta}{2}}=O(\nfrac{\sqrt{\Lambda}}{\beta})$.
                    Applying \cite[Theorem 6.1]{charikar2017learning} now yields parameters $\theta_1,\theta_2,\dots,\theta_k$ with $k\leq \inparen{(1-\beta)\gamma}^{-1}$
                    such that
                    \[
                        \min_{1\leq i\leq k}\norm{\theta_i-\theta^\star}_2
                        \leq
                        O\!\inparen{
                            \frac{\sqrt{\Lambda}}{\beta\lambda}
                            \sqrt{\frac{\log{\nfrac{2}{\gamma}}}{\gamma}}
                        }.
                        \yesnum\label{eq:application:listDecoding:paramDistance}
                    \]
                    Since we fixed $\beta=\nfrac12$, this gives $k\leq \nfrac{2}{\gamma}$ and
                    $\min_{1\leq i\leq k}\norm{\theta_i-\theta^\star}_2
                        \leq
                        O\!\inparen{
                            \sqrt{\frac{\Lambda\log{\nfrac{2}{\gamma}}}{\lambda^2\gamma}}
                        }.$

                \medskip 
                
                \noindent Next, project each $\theta_i$ to $\Theta(\eta)$.
With some abuse of notation, denote the projected points again by $\theta_1,\dots,\theta_k$.
Since $\theta^\star\in \Theta(\eta)$, projection can only decrease the distance to $\theta^\star$.
                Define 
                \[
                    R \coloneqq O\!\inparen{
                        \frac{\sqrt{\Lambda}}{\lambda}
                        \sqrt{\frac{\log{\nfrac{2}{\gamma}}}{\gamma}}
                    }\,.
                \]
                We now convert this parameter-distance bound into a Rényi-divergence bound.
                    \begin{fact}\label{fact:app}
                        For any $\theta_1,\theta_2\in \Theta(\eta)$ that satisfy $\norm{\theta_1-\theta_2}_2^2\leq R$ and let $w>1$ satisfy $(w-1)R\leq \eta$, it holds that 
                        $\Ex_{x\sim \cE(\theta_2)}\insquare{
                            \inparen{
                                \nfrac{
                                    \cE(x; \theta_1)
                                }{
                                    \cE(x; \theta_2)
                                }
                            }^w
                            }
                            \leq \exp\!\inparen{
                                \nfrac{\Lambda w(w-1)R^2}{2} 
                            }.$
                        Consequently, $\renyi{w}{\cE(\theta_1)}{\cE(\theta_2)}
                            \leq
                            \nfrac{\Lambda w R^2}{2}.$
                    \end{fact}
                    \begin{proof}[Proof of \cref{fact:app}]
                        Observe that 
                        \begin{align*}
                            \Ex_{x\sim \cE(\theta_2)}\insquare{
                                \inparen{
                                    \frac{
                                        \cE(x; \theta_1)
                                    }{
                                        \cE(x; \theta_2)
                                    }
                                }^w
                            }
                            &=
                            \int  e^{w(\theta_1-\theta_2)^\top t(x) + (w-1)A(\theta_2)-wA(\theta_1)} \cE(x; \theta_2)\d x\\
                            &=
                             e^{
                                    (w-1)A(\theta_2)+ A\inparen{\theta_1+(w-1)(\theta_1 - \theta_2)} - wA(\theta_1) 
                                }
                             \,.
                        \end{align*}
                        Since $\theta_1$ is in the $\eta$-relative interior of $\Theta$ and $\theta_1-\theta_2$ lies in the affine subspace of $\Theta$ with $\norm{\theta_1-\theta_2}\leq R$, for the chosen $w$, $\theta_1+(w-1)(\theta_1 - \theta_2)\in \Theta$.
                        Since $A$ is a $\Lambda$-smooth function over $\Theta$, 
                        \begin{align*}
                            (w-1)A(\theta_2)+ A\inparen{\theta_1+(w-1)(\theta_1 - \theta_2)} - wA(\theta_1) 
                            &\leq 
                            \frac{\Lambda }{2}(w-1)\norm{\theta_1-\theta_2}_2^2
                            +
                            \frac{\Lambda }{2}(w-1)^2\norm{\theta_1-\theta_2}_2^2\\
                            &\leq \frac{\Lambda w(w-1)}{2} \norm{\theta_1-\theta_2}_2^2\,.
                        \end{align*}
                        Since $\norm{\theta_1-\theta_2}_2\leq R$, the first claim follows, and the bound on $\renyi{w}{\cE(\theta_1)}{\cE(\theta_2)}$ follows taking the logarithm and by dividing by $w-1$.
                    \end{proof}
Let $i^\star$ be an index attaining the bound in \cref{eq:application:listDecoding:paramDistance}.
Choose
$w
    \coloneqq
    1+c\cdot \min\inbrace{1,\nfrac{\eta}{R}}$
for a sufficiently small absolute constant $c>0$.
Then $(w-1)\norm{\theta_{i^\star}-\theta^\star}_2\leq \eta,$
so \cref{fact:app} applies with $\theta_1=\theta^\star$ and $\theta_2=\theta_{i^\star}$.
Therefore,
\[
    \renyi{w}{\cE(\theta^\star)}{\cE(\theta_{i^\star})}
    \leq
    \frac{\Lambda}{2}wR^2
    =
    O\!\inparen{
        \Lambda R^2
    }
    =
    O\!\inparen{
        \frac{\Lambda^2}{\lambda^2}
        \cdot
        \frac{\log{\nfrac{2}{\gamma}}}{\gamma}
    }.
\]
Finally, since $\beta=\nfrac12$, the preceding high-probability event holds with probability at least $1-e^{-\Omega(\gamma n)},$
which is at least $1-\delta$ for $n=O\!\inparen{(\nfrac{m}{\gamma})\,\log\nfrac{1}{\delta}}.$
This completes the proof.\qedhere{}
 
            \end{proof}

\section{Impossibility Result for Positive-Only Learning under \cref{asmp:posFraction}}
    \label{sec:impossibility}

The works of \cite{shvaytser1990positiveonly,kivinen1995one,graus1989lower} characterize  positive-only learning (in the usual worst-case sense)
Their proofs can also be straightforwardly extended to show that the same characterization for positive-only learning holds even when the mass of positive samples is lower bounded (as guaranteed by \cref{asmp:posFraction}).
For completeness, here we present an alternate proof of the impossibility that continues to hold under \cref{asmp:posFraction}.

Due to the results of \cite{shvaytser1990positiveonly,kivinen1995one} one can also show that this impossibility result is tight: if a class $\hyH$ has $\lim_{k\to \infty} \vc{}(\hyH_{\cap k})<\infty$, then it can be learned from positive-only samples.
A combinatorial characterization of classes with $\lim_{k\to \infty} \vc{}(\hyH_{\cap k})<\infty$ is provided by both the splitting dimension in \cite{kivinen1995one} and also the very closely related one-star dimension in \cite{hanneke2024star}.

\begin{restatable}[Impossibility of Learning from Positive Samples]{theorem}{ImpossibilityPositiveSampleLearning}
    \label{thm:impossibilityPositiveSamples}
    Let $\hyH$ be any hypothesis class satisfying $\lim_{k\to\infty} \vc{}(\hyH_{\cap k})=\infty$ where $\hyH_{\cap k}$ is the hypothesis class formed by taking intersections of $k$ hypothesis from $\hyH$, \ie{}, 
    $\hyH_{\cap k}\coloneqq \inbrace{H_1\cap H_2\cap \dots\cap H_k\colon H_1,H_2,\dots,H_k\in \hyH}.$
    
    For any (randomized) algorithm and any finite sample size $n$, there is $h^\star\in \hyH$ and distribution $\cDtrue$ satisfying \cref{asmp:posFraction} with $\alpha=\nfrac{1}{2}$ such that, given $n$ independent samples from $\cPtrue$ (the truncation of $\cDtrue$ to the set of positive samples), with probability at least $\nfrac{1}{2}$ over the samples and the internal randomness of the algorithm, the algorithm outputs a hypothesis $\wh{h}$ satisfying
    $\Pr\nolimits_{\cDtrue}\binparen{\wh{h}(x)\neq h^\star(x)}\geq \Omega(1).$
\end{restatable}
\vspace{-7mm}
\begin{proof}[Proof of \cref{thm:impossibilityPositiveSamples}]
        Consider any $k\geq 1$.
        We reduce the problem of realizable $(\eps,\delta)$-PAC learning {$\hyH_{\cap k}$} from (labeled) samples $(x,y)$ where $x\sim \cDtrue$ to $(\eps,\delta)$-PAC learning {$\hyH$} from \textit{only} positive examples under the assumptions in the theorem.
        
        This reduction will work when the optimal hypothesis $\optset$ has sufficient mass in given PAC learning instance, in particular, $\cDtrue(\optset)\geq \nfrac{1}{2}$.
        The desired result can be deduced from this reduction due to the following sample complexity lower bound for learning intersections of $k$ halfspaces by taking the limit $k\to\infty$.
        \begin{theorem}[\protect{\cite[Theorem 5.3]{neuralNetwork2009bartlett}}]
                \label{thm:PAClearningLowerbound}
                The sample complexity of $(\eps,\delta)$-PAC learning {$\hyH_{\cap k}$} under the assumption that $\cD(\optset)\geq \nfrac{1}{2}$ is $S_{\eps,\delta}(k)\coloneqq \wt{\Omega}\inparen{(\nfrac{1}{\eps})(\vc{}(\hyH_{\cap k})+\log{\nfrac{1}{\delta}})}$ for $\eps<\nfrac{1}{8}$ and $\delta < \nfrac{1}{100}$.
                {In particular, if $\lim_{k\to \infty} \vc{}(\hyH_{\cap k})=\infty$, then $\lim_{k\to \infty} S_{\eps,\delta}(k)=\infty$.}
        \end{theorem}
        \vspace{-5mm}
        \begin{proof}[Proof sketch of \cref{thm:PAClearningLowerbound}]
                The proof follows by verifying that the hard instance created in Theorem 5.3 of \cite{neuralNetwork2009bartlett} satisfies $\cD(\optset)\geq \nfrac{1}{2}$ and that {$\vc{}(\hyH_{\cap k})$ goes to $\infty$ with $k$ by definition}. 
            \end{proof}

        \noindent In the remainder of the proof, we describe the reduction and prove its correctness.

        \paragraph{Reduction.}
        Fix an instance of PAC-learning {$\hyH_{\cap k}$}, covariate distribution $\cDtrue$, and the optimal set $\optset$ satisfying $\cDtrue(\optset)\geq \nfrac{1}{2}$.
        The reduction is as follows.
        \begin{mdframed}
            \begin{enumerate}[itemsep=-1pt]
                \item[]  \hspace{-6mm} \textbf{Input:} 
                    An accuracy parameter $\eps>0$, a confidence parameter $\delta>0$, and
                    $n$ labeled samples $\inbrace{\inparen{x_i,y_i}\colon 1\leq i\leq n}$ where $x_1,x_2,\dots,x_n$ are \iid{} from $\cDtrue$ and $y_i=\mathds{1}\sinbrace{x_i\in \optset}$
                \item[]  \hspace{-6mm}  \textbf{Oracle Access:}
                    Query access to a PAC learner $\cL$ that learns from only positive samples 
                \vspace{0mm}
                \item Let the set of positive examples be $P=\inbrace{x_i \colon i\in \inbrace{1, 2, \dots,n},~~~ y_i=1}$.
                \item Return the hypothesis $\wh{H}$ obtained by querying the learner $\cL$ with $\inparen{\eps,\nfrac{\delta}{k},P}$.
            \end{enumerate}
        \end{mdframed} 
        Since $\cDtrue(\optset)\geq \nfrac{1}{2}$, we get that with high probability $\abs{P}\geq \Omega(n)$ and, hence, we only lose a constant factor of samples in the reduction.

        \paragraph{Soundness and Correctness.}
        Let $\wh{H}$ be the hypothesis output by the learner $\cL$ when queried with $\inparen{\eps,\delta/k,P}$.
        We will apply the PAC guarantee of $\cL$ to at most $k$ auxiliary realizable positive-only learning instances and then take a union bound.
        This will imply that, with probability at least $1-\delta$ over the randomness of $P$ and $\cL$,
        $\cDtrue\sinparen{\wh{H}\triangle \optset}\leq 2\eps\,,$
        proving the correctness of the reduction.
        Let the optimal hypothesis be $\optset=H^\star_1\cap H^\star_2\cap \dots \cap H^\star_k$ where $H^\star_1,H^\star_2,\dots,H^\star_k \in \hyH$.
        Without loss of generality, suppose that there is no set $S\subsetneq[k]$ with $\optset\neq \bigcap_{i\in S} H^\star_i$.
        Otherwise, we can select the smallest such $S$, and continue with the remainder of the proof with $k=\abs{S}$ and the halfspaces $\inbrace{H^\star_i \colon i\in S}$.
        Define the following sets, which are useful in the analysis:
        \[
            T_1 = \cX \setminus H^\star_1\,,\quad 
            T_2 = H^\star_1 \setminus H^\star_2\,,\quad 
            T_3 = \sinparen{H^\star_1\cap H^\star_2} \setminus H^\star_3\,,\quad \dots\quad 
            T_k = {\bigcap\nolimits_{i=1}^{k-1}H^\star_i} \setminus H^\star_k
            \,.
        \] 
        We will only need the following properties of $T_1,T_2,\dots,T_k$: for distinct $1\leq i,j\leq k$
        \[
            T_i \neq \emptyset\,,\quad 
            T_i \cap T_j = \emptyset\,,\quad 
            T_i \cap \optset = \emptyset\,,
            \quad T_i \subseteq \cX\setminus H^\star_i\,,
            \quadand 
            \cX=\optset \sqcup T_1\sqcup T_2\sqcup \dots \sqcup T_k\,.
            \yesnum\label{eq:onesided:propertiesofT}
        \]
        The first condition follows from the assumption that there is no set $S\subsetneq [k]$ with $\optset=\bigcap_{i\in S}H^\star_i$.
        The remaining conditions follow from construction.
        Let $I\coloneqq \inbrace{i\in [k]\colon \cDtrue(T_i)>0}.$
        For each $i\in I$, define the auxiliary distribution
        \[
            \cD'_i = \frac{1}{2}\cDtrue_{\optset} + \frac{1}{2}\cDtrue_{T_i}\,,
        \]
        where, for each set $S$ and distribution $\cQ$, $\cQ_S$ is the truncation of $\cQ$ to $S$.
        First, we verify that each $(\cD'_i,H_i^\star)$ is a realizable positive-only learning instance for \mbox{$\hyH$ with $\alpha=\nfrac{1}{2}$ and positive distribution $\cP$.}
        \begin{fact}
            For each $i\in I$, $\cD'_i(H_i^\star)=\nfrac{1}{2}$ and $(\cD'_i)_{H_i^\star}=\cP.$
        \end{fact}
        \begin{proof}
            Fix any $i\in I$.
            Since $\optset\subseteq H_i^\star$ and $T_i\subseteq \cX\setminus H_i^\star$ by \eqref{eq:onesided:propertiesofT},
            $\cD'_i(H_i^\star)
                =
                \frac{1}{2}\cDtrue_{\optset}(H_i^\star)
                +
                \frac{1}{2}\cDtrue_{T_i}(H_i^\star)
                =
                \nfrac{1}{2}.$
            Moreover, conditioning $\cD'_i$ on $H_i^\star$ removes the $T_i$ component and leaves exactly $\cDtrue_{\optset}=\cP$.
        \end{proof}
        By the PAC guarantee of $\cL$ and the choice of confidence parameter $\nfrac{\delta}{k}$, for each fixed $i\in I$,
        \[
            \Pr\inparen{
                \cD'_i(\wh{H}\triangle H_i^\star)\leq \eps
            }
            \geq 1-\nfrac{\delta}{k}.
        \]
        Since $\abs{I}\leq k$, with probability at least $1-\delta$, the following holds simultaneously for every $i\in I$:
        \[
            \cD'_i(\wh{H}\triangle H_i^\star)\leq \eps.
        \]
        Condition on this event for the remainder of the proof.
        For each $i\in I$, the $\supp(\cD'_i)\subseteq \optset\cup T_i$.
        Moreover, $\optset\subseteq H_i^\star$ and $T_i\subseteq \cX\setminus H_i^\star$, so $H_i^\star$ and $\optset$ agree on the support of $\cD'_i$.
        Hence
        \[
            \cD'_i(\wh{H}\triangle \optset)
            =
            \cD'_i(\wh{H}\triangle H_i^\star)
            \leq \eps.
        \]
        Expanding the definition of $\cD'_i$ gives, for each $i\in I$,
        $\frac{1}{2}\cDtrue_{\optset}(\wh{H}\triangle \optset)
            +
            \frac{1}{2}\cDtrue_{T_i}(\wh{H}\triangle \optset)
            \leq \eps.$
        Hence,
        \[
            \cDtrue_{\optset}(\wh{H}\triangle \optset)\leq 2\eps
            \qquadand
            \forall_{i\in I},\quad
            \cDtrue_{T_i}(\wh{H}\triangle \optset)\leq 2\eps.
            \yesnum\label{eq:onesided:errorGuarantee}
        \]
        Now, we are ready to bound $\cDtrue(\wh{H}\triangle \optset)\leq O(\eps)$:
        \begin{align*}
            \cDtrue(\wh{H}\triangle \optset)
            &= \cDtrue(\optset)\,\cDtrue_{\optset}(\wh{H}\triangle \optset)
            + \sum\nolimits_{i\in I} \cDtrue(T_i)\,\cDtrue_{T_i}(\wh{H}\triangle \optset)
                \tag{since $\cX=\optset\sqcup T_1\sqcup T_2\sqcup \dots \sqcup T_k$ by \eqref{eq:onesided:propertiesofT}, and $\cDtrue(T_i)=0$ for $i\notin I$}\\
            &\leq 2\eps\,\cDtrue(\optset)
            + 2\eps \sum\nolimits_{i\in I} \cDtrue(T_i)
                \tag{by \eqref{eq:onesided:errorGuarantee}}\\
            &= 2\eps\,.\qedhere{}
        \end{align*} 
    \end{proof}  
\end{document}